\documentclass[journal,11pt,onecolumn]{IEEEtran}
\usepackage{amsthm}
\usepackage{amsmath}
\usepackage{amssymb}
\usepackage{latexsym}
\usepackage{verbatim}
\usepackage{subfigure}
\usepackage{psfrag}
\usepackage{hyperref}

\usepackage{tcolorbox}
\definecolor{darkred}{RGB}{250,0,0}
\definecolor{darkgreen}{RGB}{0,150,0}
\definecolor{myblue}{RGB}{0,0,250}
\definecolor{darkblue}{RGB}{0,0,200}
\hypersetup{colorlinks=true, linkcolor=darkred, citecolor=myblue, urlcolor=darkblue}

\newtheorem{theorem}{Theorem}

\newtheorem{lemma}{Lemma}

\newtheorem{assumption}{Assumption}

{\begin{list}               
    {$\bullet$ \hfill}{
        \setlength{\leftmargin}{\parindent}
        \setlength{\parsep}{0.04\baselineskip}
        \setlength{\itemsep}{0.5\parsep}
        \setlength{\labelwidth}{\leftmargin}
        \setlength{\labelsep}{0em}}
    }
{\end{list}}

\providecommand{\cref}[1]{Chapter~\ref{chap:#1}}

\providecommand{\R}{\ensuremath{\mathbb{R}}}

\providecommand{\abs}[1]{\lvert#1\rvert}
\providecommand{\norm}[1]{\lVert#1\rVert}

\renewcommand{\vec}[1]{\ensuremath{\boldsymbol{#1}}}
\providecommand{\mat}[1]{\ensuremath{\boldsymbol{#1}}}

\providecommand{\mA}{\mat{A}} \providecommand{\mB}{\mat{B}}
\providecommand{\mC}{\mat{C}} 
\providecommand{\mD}{\mat{D}}

\providecommand{\mI}{\mat{I}}  
\providecommand{\mK}{\mat{K}}  
\providecommand{\mM}{\mat{M}} \providecommand{\mP}{\mat{P}} 
 \providecommand{\mR}{\mat{R}}
\providecommand{\mS}{\mat{S}} \providecommand{\mU}{\mat{U}} 
\providecommand{\mV}{\mat{V}}

\providecommand{\mT}{\mat{T}}
\providecommand{\mSigma}{\mat{\Sigma}}
\providecommand{\mGm}{\mat{\Gamma}} \providecommand{\mG}{\mat{G}}

\providecommand{\va}{\vec{a}} \providecommand{\vb}{\vec{b}}
 
\providecommand{\vh}{\vec{h}} 
  
 \providecommand{\vp}{\vec{p}}
\providecommand{\vq}{\vec{q}} \providecommand{\vs}{\vec{s}}
\providecommand{\vt}{\vec{t}} \providecommand{\vr}{\vec{r}}
\providecommand{\vg}{\vec{g}}
\providecommand{\vu}{\vec{u}} \providecommand{\vw}{\vec{w}}
\providecommand{\vx}{\vec{x}} \providecommand{\vy}{\vec{y}}
\providecommand{\vz}{\vec{z}} 
 
 \providecommand{\vv}{\vec{v}}

\usepackage{lipsum}    
\usepackage{ifthen}
\usepackage{tcolorbox}
\newboolean{showcomments}
\setboolean{showcomments}{true}
\newcommand{\oussama}[1]{\ifthenelse{\boolean{showcomments}}
{ \textcolor{red}{(Oussama says:  #1)}}{}}
\newcommand{\christos}[1]{\ifthenelse{\boolean{showcomments}}
{ \textcolor{blue}{(Christos says: #1)} } {} }
\newcommand{\yue}[1]{\ifthenelse{\boolean{showcomments}}
{ \textcolor{magenta}{(Yue says:  #1)}}{}}

\usepackage{hyperref}

\usepackage{enumitem}

\providecommand{\abs}[1]{\lvert#1\rvert}
\providecommand{\norm}[1]{\lVert#1\rVert}

\renewcommand{\vec}[1]{\ensuremath{\boldsymbol{#1}}}
\providecommand{\mat}[1]{\ensuremath{\boldsymbol{#1}}}

\providecommand{\mA}{\mat{A}} \providecommand{\mB}{\mat{B}}
\providecommand{\mC}{\mat{C}} 
\providecommand{\mD}{\mat{D}}
\providecommand{\mF}{\mat{F}}
\providecommand{\mI}{\mat{I}}  
\providecommand{\mB}{\mat{B}}  
\providecommand{\mM}{\mat{M}} \providecommand{\mP}{\mat{P}} 
 \providecommand{\mR}{\mat{R}}
\providecommand{\mS}{\mat{S}} \providecommand{\mU}{\mat{U}} 
\providecommand{\mV}{\mat{V}}

\providecommand{\mZ}{\mat{Z}}
\providecommand{\mB}{\mat{B}}
\providecommand{\mP}{\mat{P}}

\providecommand{\mSigma}{\mat{\Sigma}}
\providecommand{\mGm}{\mat{\Gamma}} \providecommand{\mG}{\mat{G}}

\usepackage{balance}
\providecommand{\va}{\vec{a}} \providecommand{\vb}{\vec{b}}
 
\providecommand{\vh}{\vec{h}} 
  
 \providecommand{\vp}{\vec{p}}
\providecommand{\vq}{\vec{q}} \providecommand{\vs}{\vec{s}}
\providecommand{\vt}{\vec{t}} \providecommand{\vr}{\vec{r}}
\providecommand{\vf}{\vec{f}} 
\providecommand{\vg}{\vec{g}}
\providecommand{\vu}{\vec{u}} \providecommand{\vw}{\vec{w}}
\providecommand{\vx}{\vec{x}} \providecommand{\vy}{\vec{y}}
\providecommand{\vz}{\vec{z}} 
 
 \providecommand{\vv}{\vec{v}}

\providecommand{\vtau}{\vec{\tau}}

\usepackage{algorithm}
\usepackage{algpseudocode}
\algnewcommand\algorithmicforeach{\textbf{Until :}}
\algnewcommand\algorithmicendif{\textbf{End}}
\algnewcommand\ForEach{\item[ \algorithmicforeach]}
\algnewcommand\EndiFF{\item[ \algorithmicendif]}

\usepackage{cite}

\usepackage{color}

\usepackage{enumitem}
\usepackage{commath}

\usepackage{hyperref}

\usepackage{esdiff}

\usepackage{pgf,tikz,psfrag}

\usepackage{dsfont}

\newcommand{\argmin}{\operatornamewithlimits{argmin}}

\providecommand{\vxi}{\vec{\xi}}

\begin{document}

\title{On the Inherent Regularization Effects of\\ Noise Injection During Training}

\author{Oussama Dhifallah and Yue M. Lu
        \thanks{O. Dhifallah and Y. M. Lu are with the John A. Paulson School of Engineering and Applied Sciences, Harvard University, Cambridge, MA 02138, USA (e-mails: oussama$\_$dhifallah@g.harvard.edu,yuelu@seas.harvard.edu).}
\thanks{This research was funded by the Harvard FAS Dean's Fund for Promising Scholarship, and by the US National Science Foundations under grants CCF-1718698 and CCF-1910410.}
}

\maketitle

\begin{abstract}
Randomly perturbing networks during the training process is a commonly used approach to improving generalization performance. In this paper, we present a theoretical study of one particular way of random perturbation, which corresponds to injecting artificial noise to the training data. We provide a precise asymptotic characterization of the training and generalization errors of such randomly perturbed learning problems on a random feature model. Our analysis shows that Gaussian noise injection in the training process is equivalent to introducing a weighted ridge regularization, when the number of noise injections tends to infinity. The explicit form of the regularization is also given. Numerical results corroborate our asymptotic predictions, showing that they are accurate even in moderate problem dimensions. Our theoretical predictions are based on a new correlated Gaussian equivalence conjecture that generalizes recent results in the study of random feature models.
\end{abstract}

\section{Introduction}
A popular approach to improving the generalization performance is to  randomly perturb the network during the training process \cite{dropout,bishop,atvn,ord_imp,ridgoptimal}. Such random perturbations are widely used as an implicit regularization to the learning problem. One way that random perturbation has been used as a regularization is by injecting it to the input data before starting the learning process \cite{maxup,nise_inj1,nise_inj2}. In this paper, we provide a theoretical analysis of such learning procedure on a random feature model \cite{RR08} under Gaussian input and perturbation vectors. Our analysis particularly shows that Gaussian noise injection introduces a weighted ridge regularization, asymptotically.

First, we describe the models for our theoretical analysis. We are given a collection of training data $\lbrace (y_i,\va_i) \rbrace_{i=1}^{n}$, where $\va_i\in\mathbb{R}^p$ is referred to as the input vector and $y_i\in\mathbb{R}$ is referred to as the label corresponding to  $\va_i$. In this paper, we shall assume that the labels are generated according to the standard \textit{teacher--student} model, i.e.
\begin{align}\label{cmodel}
y_i=\varphi(\va_i^\top \vxi),~\forall i\in\lbrace 1,\dots,n \rbrace,
\end{align}
where $\vxi\in\mathbb{R}^p$ is an unknown teacher weight vector, and $\varphi(\cdot)$ is a scalar deterministic or probabilistic function. Here, we use the random feature model \cite{RR08} to learn the model described in \eqref{cmodel}. The random feature model considers the following class of functions 
\begin{align}\label{RFM}
\mathcal{F}_{\text{RF}}(\va)=\Big\lbrace g_{\vw}(\va)=\vw^\top \sigma(\mF^\top \va),~\vw\in\mathbb{R}^k \Big\rbrace,
\end{align}
where $\va\in\mathbb{R}^p$ is an input vector, $\mF\in\mathbb{R}^{p\times k}$ is a random matrix referred to as the \textit{feature matrix}, and $\sigma(\cdot)$ is a scalar function referred to as the \textit{activation function}. This model assumes that $\mF$ is fixed during the training. Note that the family in \eqref{RFM} can be viewed as a two--layer neural network where the first layer weights are fixed, i.e. $\mF$ is fixed. 

\subsection{Learning Formulation}
Before starting the learning process, $\ell$ independent perturbation vectors are injected to each $\va_i$. This procedure forms the augmented family $\lbrace \va_i+\Delta \vz_{ij} \rbrace_{j=1}^{\ell}$ for each $\va_i$, where $\lbrace \vz_{ij} \rbrace_{j=1}^{\ell}$ are independent random perturbations and $\Delta \geq 0$ denotes the \textit{noise variance}. In this paper, we study the effects of such perturbation method on an average loss and a random feature model. Specifically, we analyze formulations of the following form
\begin{align}\label{mform}
\widehat{\vw}=\argmin_{\vw\in\mathbb{R}^k}&~ \frac{1}{2n \ell} \sum_{i=1}^{n} \sum_{j=1}^{\ell} \big(y_i-\vw^\top \sigma\big(\mF^\top [\va_i+\Delta \vz_{ij}] \big) \big)^2+\tfrac{\lambda}{2} \norm{\vw}^2,
\end{align}
where $\lambda>0$ denotes the regularization parameter. Note that the  problem in \eqref{mform} is a standard feature formulation when $\Delta=0$. Then, we refer to \eqref{mform} as the \textit{noisy formulation}, when $\Delta>0$ and the \textit{standard formulation}, otherwise.

\subsection{Performance Measure}
The main objective in this paper is to study the performance of the learning formulation in \eqref{mform} on unobserved test data. For every test vector $\va_\text{new}\in\mathbb{R}^{p}$, the corresponding label $\widehat{y}$ can be predicted using the following (probabilistic) role
\begin{align}
\widehat{y}=\widehat{\varphi} [\widehat{\vw}^\top \sigma(\mF^\top \va_\text{new})],
\end{align}
for some predefined function $\widehat{\varphi}(\cdot)$, where $\widehat{\vw}\in\mathbb{R}^k$ denotes the optimal solution of the  formulation given in \eqref{mform}. To measure the performance of the learning problem in \eqref{mform} on any unobserved test data $\lbrace ({y}_{\text{new}},\va_\text{new}) \rbrace$, we use the \textit{generalization error} defined as follows
\begin{align}\label{testerr_def}
\mathcal{E}_{\text{test}}=\frac{1}{4^\upsilon} \mathbb{E}\Big[ \big( {y}_{\text{new}}-\widehat{\varphi}(\widehat{\vw}^\top \sigma(\mF^\top \va_{\text{new}})) \big)^2 \Big].
\end{align}
Here, the expectation is taken over the distribution of the unobserved test  vector $\va_{\text{new}}$ and the (random) functions $\varphi(\cdot)$ and $\widehat{\varphi}(\cdot)$. We take $\upsilon=0$ for regression problems (e.g. $\varphi(\cdot)$ is the identity function) and $\upsilon=1$ for binary classification problems (e.g. $\varphi(\cdot)$ is the sign function).
 In this paper, we assume that the test data is generated according to the same training model introduced in \eqref{cmodel}.
Furthermore, we measure the performance of the formulation in \eqref{mform} on the training data via the \textit{training error} defined as follows
\begin{align}
\mathcal{E}_{\text{train}}= \frac{1}{2n \ell} \sum_{i=1}^{n} \sum_{j=1}^{\ell} \big(y_i-\widehat{\vw}^\top \sigma\big(\mF^\top [\va_i+\Delta \vz_{ij}] \big) \big)^2.\nonumber
\end{align}
Note that the training error is the optimal cost value of our learning formulation in \eqref{mform} without regularization.

\subsection{Contributions}\label{ctrbs}
The contribution of this paper can be summarized as follows:
\begin{itemize}[wide = 0pt]
\item[(C.1)] Our first contribution is a \textit{correlated Gaussian equivalence conjecture} (cGEC). Our conjecture considers Gaussian input and perturbation vectors. It states that the learning formulation in \eqref{mform} is asymptotically equivalent to a simpler optimization problem that can be formulated by replacing the non--linear vectors 
\begin{align}
\vv_{ij}=\sigma\big(\mF^\top [\va_i+\Delta \vz_{ij}] \big),\nonumber
\end{align}
with linear vectors with the following form
\begin{align}
\vq_{ij}=\mu_0\vec{1}_k+\widetilde{\mu}_1 \mF^\top \va_i+\widehat{\mu}_1 \mF^\top \vz_{ij} + \mu_2 \vb_i+ \mu_3  \vp_{ij}.\nonumber
\end{align}
Here, $\lbrace \vb_i \rbrace_{i=1}^n$ and $\lbrace \vp_{ij} \rbrace_{i,j=1}^{n,\ell}$ are independent standard Gaussian random vectors and independent of $\lbrace \va_i \rbrace_{i=1}^n$ and $\lbrace \vz_{ij} \rbrace_{i,j=1}^{n,\ell}$. Moreover, the weights $\mu_0$, $\widetilde{\mu}_1$, $\widehat{\mu}_1$, $\mu_2$ and $\mu_3$ depend on $\sigma(\cdot)$ and $\Delta$ as follows
\begin{align}
\begin{cases}\nonumber
\mu_0=\mathbb{E}[\sigma(x_1)], \widetilde{\mu}_1=\mathbb{E}[z \sigma(x_1)], \widehat{\mu}_1=\mathbb{E}[v_1 \sigma(x_1)]\\
\mu_2^2=\mathbb{E}[\sigma(x_1) \sigma(x_2)]-\mu_0^2-\widetilde{\mu}_1^2\\
\mu_3^2=\mathbb{E}[\sigma(x_1)^2 ] - \mathbb{E}[\sigma(x_1) \sigma(x_2)]-\widehat{\mu}_1^2,
\end{cases}
\end{align} 
where $x_1=z+\Delta v_1$, $x_2=z+\Delta v_2$, and $z$, $v_1$ and $v_2$ are independent standard Gaussian random variables. Specifically, the cGEC states that the performance of the formulation:
\begin{align}\label{gform}
&\min_{\vw\in\mathbb{R}^k} \frac{1}{2n\ell} \sum_{i=1}^{n} \sum_{j=1}^{\ell}  \big(y_i-  \widetilde{\mu}_1 \vw^\top  \mF^\top \va_i-\widehat{\mu}_1 \vw^\top  \mF^\top \vz_{ij} \nonumber\\
&-\mu_0 \vw^\top \vec{1}_k  -\mu_2 \vw^\top \vb_i- \mu_3 \vw^\top \vp_{ij} \big)^2 + \tfrac{\lambda}{2} \norm{\vw}^2,
\end{align}
is asymptotically equivalent to the performance of the noisy formulation. This conjecture is valid in the asymptotic limit (i.e. $n$, $p$ and $k$ grow to infinity at finite ratios). More details about this equivalence is provided in Section \ref{CGEC}. We refer to \eqref{gform} as the \textit{Gaussian formulation}. The cGEC is verified by presenting multiple simulations in different scenarios.

\item[(C.2)] The second contribution is a precise characterization of the training and generalization errors of the noise injection procedure formulated in \eqref{mform} for Gaussian input and perturbation vectors. Our analysis is based on the cGEC and valid in the high--dimensional setting (i.e. $n$, $p$ and $k$ grow to infinity at finite ratios). Our predictions show that the asymptotic limit of the training and generalization errors can be precisely predicted after solving a scalar deterministic formulation. The theoretical predictions are obtained using an extended version of the convex Gaussian min-max theorem (CGMT) \cite{chris:151,chris:152} which we refer to as the multivariate CGMT. The new version of the CGMT accounts for the correlation introduced by injecting Gaussian noise during the learning process. Our asymptotic results hold for a general family of feature matrices, activation functions and generative models satisfying \eqref{cmodel}.  

\begin{figure}[t]
    \centering
    \subfigure[]{\label{fintro1a}
        \includegraphics[width=0.4\linewidth]{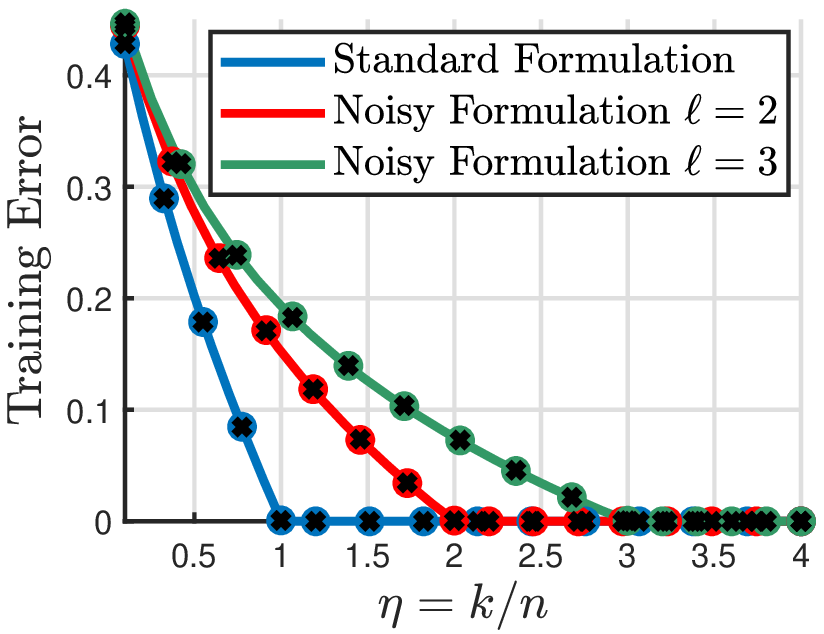}}
        \subfigure[]{\label{fintro1b}
        \includegraphics[width=0.4\linewidth]{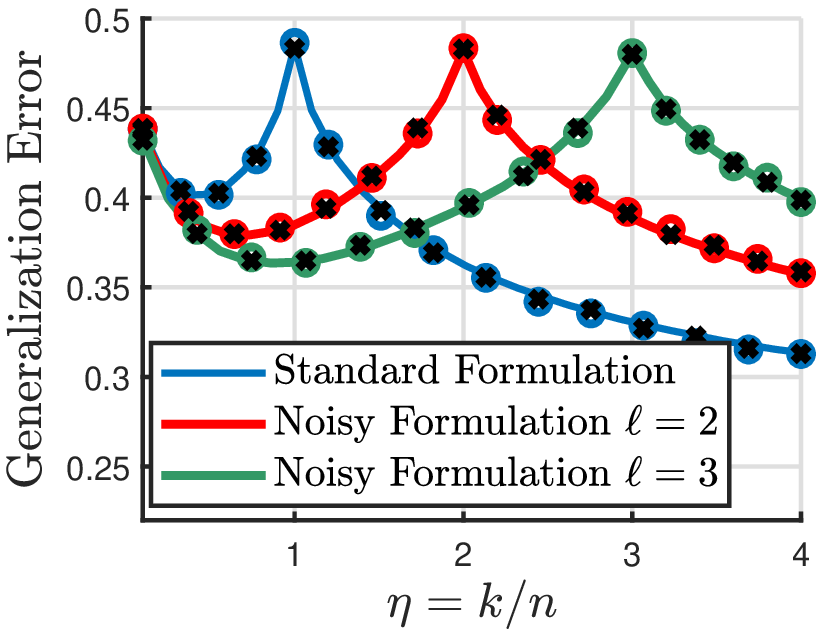}}

    \caption{Solid line: Theoretical predictions. Circle: Numerical simulations for \eqref{mform}. Black cross: Numerical simulations for \eqref{gform}. $\varphi(\cdot)$ is the sign function with probability $\theta$ of flipping the sign. $\widehat{\varphi}(\cdot)$ and $\sigma(\cdot)$ are the sign function. We set $p=500$, $\Delta=0.5$, $\alpha=2$, $\theta=0.1$, $\lambda=10^{-5}$. $\mF$ has independent Gaussian components with zero mean and variance $1/p$. The results are averaged over $200$ independent Monte Carlo trials.}
        \label{fintro1}
\end{figure}

In Figure \ref{fintro1}, we compare our theoretical predictions with the actual performance of the learning problem given in \eqref{mform}. First, note that our asymptotic predictions are in excellent agreement with the actual performance of \eqref{mform} and its Gaussian formulation given in \eqref{gform}, even for moderate values of $p$, $n$ and $k$. This provides a first empirical validation of our results. Figure \ref{fintro1} also study the effects of $\ell$ on the training and generalization performance. Note that the generalization error follows a double descent curve \cite{BMM18,BHMM18}. Specifically, the generalization error decreases monotonically as a function of the complexity $\eta=k/n$ after reaching a peak known as the interpolation threshold \cite{BMM18,BHMM18}. Figure \ref{fintro1b} particularly demonstrates that the location of the interpolation threshold depends on the number of noise samples. Specifically, the interpolation threshold peak occurs at $\ell$ for fixed noise variance $\Delta=0.5$. Additionally, Figure \ref{fintro1a} shows that the interpolation threshold occurs when the training error converges to zero. Then, we can see that perturbing the input data with $\ell$ random noise vectors moves the interpolation threshold from $1$ to $\ell$ and improves the generalization error for complexity $\eta$ lower than $\ell$.

\item[(C.3)] The third contribution is a precise analysis of the regularization effects of the considered noise injection procedure. Specifically, we use the asymptotic predictions of the noisy formulation to show that the noise injection model in \eqref{mform} is equivalent to solving a standard feature formulation with an additional weighted ridge regularization. This theoretical result is valid when the number of noise samples $\ell$ tends to infinity. In particular, we show that the formulation in \eqref{mform} is equivalent to solving the problem
\begin{align}\label{form_asy}
\min_{\vw\in\mathbb{R}^k}&~\frac{1}{2n} \sum_{i=1}^{n}  \big(y_i-  \vw^\top \widehat{\sigma}(\mF^\top \va_i)  \big)^2+ \frac{1}{2} \norm{\mR^{\frac{1}{2}} \vw}^2 + \tfrac{\lambda}{2} \norm{\vw}^2,
\end{align}
when $\ell$ grows to infinity slower than the dimensions $n$, $p$ and $k$.
Here, $\widehat{\sigma}(\cdot)$ is a new activation function and $\mR$ is defined as follows
\begin{align}
\mR=\widehat{\mu}_1^2 \mF^\top \mF + \mu_3^2 \mI_k.
\end{align}
Finally, we provide a precise asymptotic characterization of the training and generalization errors corresponding to \eqref{form_asy}. We refer to this formulation as the \textit{limiting formulation}.

\begin{figure}[t]
    \centering
    \subfigure[]{\label{fintro2a}
    \includegraphics[width=0.4\linewidth]{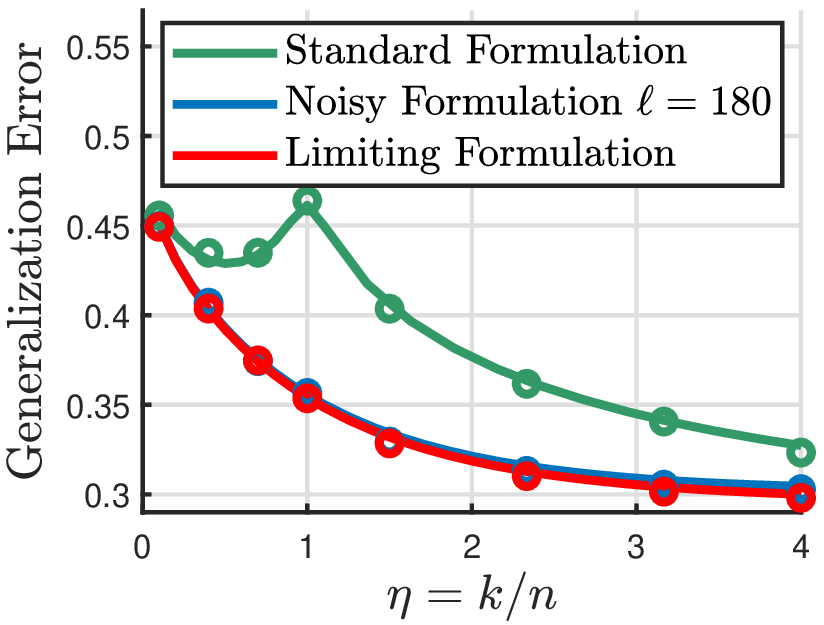}
    }
    \subfigure[]{\label{fintro2b}
    \includegraphics[width=0.4\linewidth]{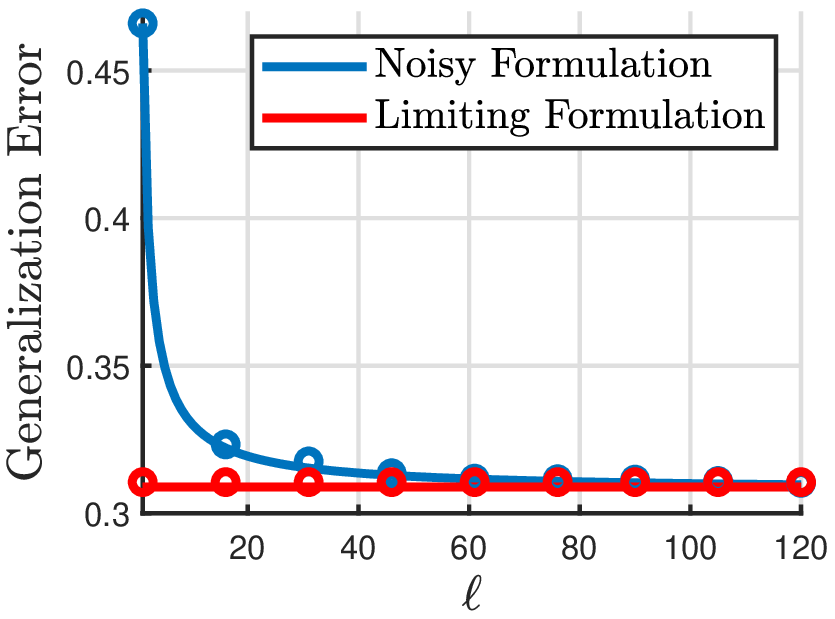}
    }
    \caption{Solid line: Theoretical predictions. Circle: Numerical simulations for \eqref{mform} and \eqref{form_asy}. $\varphi(\cdot)$, $\widehat{\varphi}(\cdot)$ and $\sigma(\cdot)$ are the sign function. {\bf (a)} $p=700$, $\ell=180$, $\alpha=1$, $\Delta=1$ and $\lambda=10^{-3}$. {\bf (b)} $p=600$, $\alpha=1.5$, $\Delta=1$ and $\lambda=10^{-3}$. $\mF$ has independent Gaussian components with zero mean and variance $1/p$. The results are averaged over $100$ independent Monte Carlo trials.}
        \label{fintro2}
\end{figure}

Figure \ref{fintro2} provides another empirical verification of our theoretical predictions since it shows that they are in excellent agreement with the actual performance of \eqref{mform} and \eqref{form_asy}. Figure \ref{fintro2a} shows that the noisy formulation in \eqref{mform} has approximately the same performance as the formulation in \eqref{form_asy} for $\ell=180$. This is aligned with our theoretical prediction which states that the formulations in \eqref{mform} and \eqref{form_asy} are equivalent when $\ell$ grows to infinity slower than the dimensions $n$, $p$ and $k$. Figure \ref{fintro2b} illustrates the convergence behavior of the generalization error corresponding to \eqref{mform} for a fixed value of $\eta$. It particularly shows that the noisy formulation has a good \textit{convergence rate}, i.e. the limit is already attained with a moderate value of $\ell$. Moreover, we can see from Figure \ref{fintro2a} that the convergence rate depends on the complexity parameter $\eta$.
\end{itemize}

\subsection{Related Work}
There has been significant interest in precisely characterizing the performance of the random feature model in recent literature \cite{RF_monta_1,RF_rep_1,dhifallah2020,hu2020}. The ridge regression formulation, (i.e. $\varphi(\cdot)$ is the identity function and $\Delta=0$ in \eqref{mform}) is precisely analyzed in \cite{RF_monta_1} where the feature matrix is Gaussian. In a subsequent work, \cite{RF_monta_2} uses the CGMT  to accurately analyze the maximum-margin linear classifier in the overparametrized regime. The work in \cite{RF_rep_1} precisely characterizes the performance of the standard formulation, i.e. $\Delta=0$, for general families of feature matrices and convex loss functions. The results presented in \cite{RF_rep_1} are derived using the non--rigorous replica method \cite{Mezard:1986}. The predictions in \cite{RF_rep_1} are rigiourously verified in \cite{dhifallah2020} using the CGMT. All the previous work consider an unperturbed formulation of the random feature model. In this paper, we study the effects of adding random noise during  training. Our analysis is based on an extended version of the CGMT referred to as the multivariate CGMT. The CGMT is first used in \cite{stojnic2013} and further developed in \cite{chris:151}. It extends a Gaussian theorem first introduced in \cite{gordon}. It relies on (strong) convexity properties to prove an equivalence between two Gaussian processes. It has been successfully applied in the analysis of convex regression \cite{chris:151,ouss19,dhifallah2020} and convex classification \cite{ea19bc,kam20,mig2020role,dhifallah2021phase} formulations. 

There has been significant interest in studying the effects of random noise injection during training (see e.g. \cite{bishop,Gan,atvn}). In particular, prior literature \cite{zantedeschi,kannan} shows that Gaussian noise injection during training improves the robustness of the network. Moreover, several recent papers \cite{bishop,maxup} show that such perturbation technique introduces some sort of regularization to the loss function. In particular, the work in \cite{maxup} shows that minimizing the worst--case loss introduces a gradient norm regularization.

Another popular perturbation approach used in regularizing learning models is the \textit{dropout} method \cite{dropout,implicitdp}. It consists of perturbing the learning problem by randomly dropping units from the network during the training procedure. In this paper, we precisely analyze the Gaussian noise injection method and we leave the analysis of the dropout technique for future work. Our empirical studies suggest that the dropout method has a better convergence rate as compared to the noisy formulation. Moreover, they suggest that both methods have comparable generalization performance.

\subsection{Organization}
The rest of this paper is organized as follows. Section \ref{CGEC} provides more details about the cGEC. Section \ref{assmp} lay out the technical assumptions under which our results are derived. Section \ref{pr_analysis} provides an asymptotic characterization of the noisy formulation. Our theoretical predictions hold for a general family of feature matrices, activation functions and generative models as in \eqref{cmodel}. We provide additional simulation examples for special cases of our results in Section \ref{sim_res}. The detailed proof of our theoretical predictions is provided in Section \ref{dproof}. Section \ref{concd} concludes the paper. The appendix in Section \ref{app_add} provides additional technical details.

\section{Gaussian Equivalence Conjecture with an Intuitive Explanation}\label{CGEC}  
Consider three independent standard Gaussian random vectors $\va\in\mathbb{R}^p$, $\vz_1\in\mathbb{R}^p$ and $\vz_2\in\mathbb{R}^p$. Moreover, consider the random variables $\nu_1=\vxi^\top \va$, $\nu_2$ and $\nu_3$ defined as follows
\begin{align}
&\nu_2=\vw^\top \sigma(\mF^\top [\va+\Delta \vz_1]),\nu_3=\vw^\top \sigma(\mF^\top [\va+\Delta \vz_2]),\nonumber
\end{align}
where $\sigma(\cdot)$, $\vxi\in\mathbb{R}^p$ and $\mF\in\mathbb{R}^{p\times k}$ satisfy some regularity assumptions, and where $\vw\in\mathbb{R}^k$. Moreover, define the joint probability distribution of $\nu_1$, $\nu_2$ and $\nu_3$ as $\mathbb{P}(\nu_1,\nu_2,\nu_3)$. The cGEC states that the joint distribution $\mathbb{P}(\nu_1,\nu_2,\nu_3)$ is asymptotically Gaussian, i.e. $d(\mathbb{P}(\nu_1,\nu_2,\nu_3),\mathbb{P}(\nu_{g,1},\nu_{g,2},\nu_{g,3}))$ converges in probability to zero where $\nu_{g,1}$, $\nu_{g,2}$ and $\nu_{g,3}$ are jointly Gaussian with the same first and second moments of $\nu_1$ , $\nu_2$ and $\nu_3$ and $d(\cdot,\cdot)$ is some probability distance that metrizes the convergence in distribution (e.g maximum-sliced (MS) distance \cite{ko19gener,goldt20g}). To have the same first two moments, the random variables ${\nu}_{g,1}$, ${\nu}_{g,2}$ and ${\nu}_{g,3}$ are selected as follows
 $\nu_{g,1}=\nu_1$ and
\begin{align}
&{\nu}_{g,2}=\vw^\top \big( \mu_0  \vec{1}_k + \mF^\top [\widetilde{\mu}_1 \va+\widehat{\mu}_1 \vz_{1}]+\mu_2 \vb+\mu_3 \vp_{1} \big),\nonumber\\
&{\nu}_{g,3}=\vw^\top \big( \mu_0  \vec{1}_k + \mF^\top [ \widetilde{\mu}_1 \va+\widehat{\mu}_1 \vz_{2}]+\mu_2 \vb+\mu_3 \vp_{2} \big),\nonumber
\end{align}
where $\vec{1}_k$ represents the all $1$ vector with size $k$. Here, $\vb \in\mathbb{R}^k$, $\vp_1 \in\mathbb{R}^k$ and $\vp_2 \in\mathbb{R}^k$ are three independent standard Gaussian random vectors and they are independent of $\va$, $\vz_{1}$ and $\vz_{2}$. 
The weights $\mu_0$, $\widetilde{\mu}_1$, $\widehat{\mu}_1$, $\mu_2$ and $\mu_3$ are as defined in Section \ref{ctrbs}.

In the standard setting, i.e. $\Delta=0$, the cGEC is equivalent to the uniform Gaussian equivalence theorem (uGET), observed and used in many earlier papers \cite{RF_monta_2,RF_rep_1,modelling,dhifallah2020}. Recently, the work in \cite{hu2020} provided a rigorous proof of the uGET. Specifically, the work in \cite{hu2020} proves a special case of cGEC when $\Delta=0$, the feature matrix is Gaussian and the activation functions have bounded first three derivatives. However, similar to previous literature \cite{modelling}, we conjecture that the cGEC is valid under more general settings. We believe that the analysis in \cite{hu2020} can be extended to prove the cGEC and we leave the full technical details for future work. 

Our theoretical results are based on this conjecture. It is thus useful to also provide an intuitive explanation for the plausibility of the cGEC. Assume that $\vf_i$ is the $i$th column of $\mF$. The nonlinear term $I_1=\sigma( \vf_i^\top(\va+\Delta \vz_1) )$ can be decomposed by projecting on the basis $(1,\vf_i^\top\va,\vf_i^\top\vz_1)$, i.e. $I_1=\mu_0 +\widetilde{\mu}_1 \vf_i^\top \va+\widehat{\mu}_1 \vf_i^\top \vz_1 + \sigma^\perp_i$. The term $\sigma^\perp_i$ is selected so that we match the variance of $I_1$ and the correlation with $I_2=\sigma( \vf_i^\top(\va+\Delta \vz_2) )$. We note that the cGEC makes sense when the columns of $\mF$ are independent and have the same norm. These are the regularity assumptions for the feature matrix in \cite{modelling}. The same intuition also appears in the analysis of the unperturbed random kernel models, in particular, the random feature model \cite{RF_monta_2}. In this paper, we suppose that the feature matrix and the activation function satisfy the regularity assumptions in \cite{modelling} and conjecture that the Gaussian equivalence is valid for $(\nu_1,\nu_2,\dots,\nu_\ell)$ for $\ell\geq 1$ and uniformly in $\vw\in\mathbb{R}^k$. Using the cGEC, the performance of the formulation in \eqref{mform} can be fully characterized by asymptotically analyzing the Gaussian formulation given in \eqref{gform}. We verify this conjecture by performing multiple simulation examples in different settings (regression and classification). 

\section{Technical Assumptions}\label{assmp}
In this paper, we precisely characterize the noisy formulation under the following technical assumptions. 
\begin{assumption}[Gaussian Vectors]\label{itm:data_Gauss} 
The input vectors $\lbrace \va_i \rbrace_{i=1}^{n}$ and the perturbation vectors $\lbrace \vz_{ij} \rbrace_{i=1,j=1}^{n,\ell}$ are known and drawn  independently from a standard Gaussian distribution. Without loss of generality, we assume that the hidden vector $\vxi\in\mathbb{R}^p$ has unit norm. Also, it is independent of the input vectors, the noise vectors and $\mF$.
\end{assumption}
Our theoretical predictions are valid in the high-dimensional setting where  $n$, $p$ and $k$ grow to infinity at finite ratios.
\begin{assumption}[Asymptotic Limit]\label{itm:asy_lim} 
The number of samples and the number of hidden neurons satisfy $n=n(p)$ and $k=k(p)$, respectively. We assume that $\alpha_p=n(p)/p \to \alpha>0$ and $\eta_p=k(p)/n(p) \to \eta>0$ as $p\to \infty$. Also, the number of noise injections $\ell$ is independent of $p$.
\end{assumption}
Moreover, we consider the following assumption to ensure that the generalization error defined in \eqref{testerr_def} concentrates in the high--dimensional limit.
\begin{assumption}[Generative Model]\label{itm:fun_fwf} 
The data generating function $\varphi(\cdot)$ introduced in \eqref{cmodel} is independent of the input vectors, the noise vectors and the feature matrix. Moreover, the following conditions are satisfied.
\begin{itemize}[wide = 0pt]
\item[\bf (a)] $\varphi(\cdot)$ and $\widehat{\varphi}(\cdot)$ are continuous almost everywhere in $\mathbb{R}$. For every $h > 0$ and $z\sim\mathcal{N}(0,h)$, we have $\mathbb{E}[\varphi^2(z)]<+\infty$, $\mathbb{E}[z\varphi(z)]\neq 0$ and $0<\mathbb{E}[\widehat{\varphi}^2(z)]<+\infty$.
\item[\bf (b)] For any compact interval $[c, C]$, there exists a function $g(\cdot)$ such that 
\[
\sup_{h,\chi \in [c, C]} \abs{ \widehat{\varphi}(\chi+h x) }^2 \leq g(x) \quad \text{for all } x \in \R.
\]
Additionally, the function $g(\cdot)$ satisfies $\mathbb{E}[{{g}^2(z)}]<+\infty$, where $z\sim\mathcal{N}(0,1)$.
\end{itemize}
\end{assumption}
In addition to the assumptions in Section \ref{CGEC}, we consider the following regularity conditions for the activation function.
\begin{assumption}[Activation Function]\label{itm:act_fun} 
The activation function $\sigma(\cdot)$ is independent of the input vectors, the noise vectors and the feature matrix. It also satisfies $\mathbb{E}[\sigma(z)^2]<+\infty$ and $\mathbb{E}[z\sigma(z)]\neq 0$, where $z\sim\mathcal{N}(0,1)$.
\end{assumption}
In addition to the assumptions discussed in Section \ref{CGEC}, we consider a family of feature matrices that satisfy the following assumption to guarantee that the Gaussian formulation converges to a deterministic problem.
\begin{assumption}[Feature Matrix] \label{itm:ass_F} 
The SVD decomposition of the feature matrix can be expressed as $\mF=\mU\mS\mV$, where $\mU\in\mathbb{R}^{p\times p}$ and $\mV\in\mathbb{R}^{k\times k}$ are random orthogonal matrices and $\mS\in\mathbb{R}^{p\times k}$ is a diagonal matrix formed by the singular values of $\mF$. Define the matrix $\mM$ as $\mM=\mF^\top\mF$. 
\begin{itemize}
\item[\bf (a)] We assume that $\mU$ is a Haar-distributed random unitary matrix.
\item[\bf (b)] We also assume that the empirical distribution of the eigenvalues of the matrix $\mM$ converges weakly to a probability distribution $\mathbb{P}_\kappa(\cdot)$ supported in $[0~\zeta_{\text{max}}]$, where $\zeta_{\text{max}}>0$ is a constant independent of $p$ and $\ell$.
\item[\bf (c)] We finally assume that $\mathbb{E}_\kappa[\kappa] > 0$, where the expectation is taken over the distribution $\mathbb{P}_\kappa(\cdot)$.
\end{itemize}
\end{assumption}
Based on Assumption \ref{itm:asy_lim}, we also have the following property $\delta_p=k(p)/p \to \delta>0$ as $p$ grows to infinity. Our theoretical predictions use the weak convergence in Assumption \ref{itm:ass_F} to fully characterize the noisy formulation. 

\section{Precise Analysis of the Noisy Formulation}\label{pr_analysis}
In this section, we asymptotically analyze the noise injection procedure introduced in \eqref{mform}. Specifically, we provide a precise asymptotic characterization of the training and generalization errors corresponding to  \eqref{mform}. 
\subsection{Precise Asymptotic Analysis}\label{st_asylim}
Before stating our technical results, we start with few definitions. Define the following two deterministic functions 
\begin{align}
&T_{2,\lambda}(\vt,\vtau)= \frac{\delta}{T_1^2} \mathbb{E}_\kappa\Big[ \frac{\kappa}{g_{\kappa,\lambda}(\vt,\vtau)} \Big] /\Big(1-\frac{\widetilde{\mu}_1^2 t_1 \delta}{ \tau_1} \mathbb{E}_\kappa\Big[ \frac{\kappa}{g_{\kappa,\lambda}(\vt,\vtau)} \Big] \Big)\nonumber\\
&T_{3,\lambda}(\vt,\vtau)=\frac{ t_1^2}{\ell} \mathbb{E}_\kappa\Big[\frac{ \widetilde{\mu}_1^2 \kappa + \mu_2^2}{g_{\kappa,\lambda}(\vt,\vtau)}\Big]+\frac{t_1^2+t_2^2}{\ell^2} \mathbb{E}_\kappa\Big[\frac{ \widehat{\mu}_1^2 \kappa + \mu_3^2}{g_{\kappa,\lambda}(\vt,\vtau)}\Big],\nonumber
\end{align}
where the expectations are taken over the probability distribution $\mathbb{P}_\kappa(\cdot)$ defined in Assumption \ref{itm:ass_F} and where $\vt=[ t_1,t_2 ]^\top$ and $\vtau=[ \tau_1,\tau_2 ]^\top$. Here, the function $g_{\kappa,\lambda}(\cdot,\cdot)$ is defined as follows
\begin{align}
g_{\kappa,\lambda}(\vt,\vtau)&=\frac{t_1}{\tau_1} \big( \widetilde{\mu}_1^2 \kappa + \mu_2^2 \big)+\Big( \frac{t_1}{\tau_1 \ell}+\frac{t_2(\ell-1)}{\tau_2 \ell} \Big) \big( \widehat{\mu}_1^2 \kappa + \mu_3^2 \big) + \lambda.
\end{align}
Furthermore, define the following four-dimensional deterministic optimization problem
\begin{align}
\label{scprob1}
\hspace{-2mm}\inf_{\substack{\tau_1> 0 \\ \tau_2> 0} } \max_{\substack{ t_1\geq 0\\ t_2 \geq 0}}&~\frac{t_1}{2\tau_1} \big( \gamma_1 -  2 \widetilde{\mu}_1 T_1 q_{t,\tau}^\star \gamma_2 + \widetilde{\mu}_1^2 T_1^2 (q_{t,\tau}^\star)^2 + \mu_0^2 (\vartheta^\star)^2-2\mu_0 \vartheta^\star\gamma_3 \big)  \nonumber\\
&~+\frac{\tau_1 t_1+\tau_2 t_2}{2 \ell}  - \frac{t_1^2+t_2^2}{2 \ell} +\frac{(q_{t,\tau}^\star)^2}{2 T_{2,\lambda}(\vt,\vtau)} -\frac{\eta T_{3,\lambda}(\vt,\vtau)}{2},
\end{align}
where the constant $\vartheta^\star$ satisfies $\vartheta^\star=0$ if $\mu_0=0$ and $\vartheta^\star=\gamma_3/\mu_0$ otherwise, and $T_1=\sqrt{\delta \mathbb{E}_\kappa[\kappa]}$.
Here, $\gamma_1$, $\gamma_2$ and $\gamma_3$ depend on the data distribution and are defined as $\gamma_1=\mathbb{E}[y^2]$, $\gamma_2=\mathbb{E}[y s]$, $\gamma_3=\mathbb{E}[y]$, where $y=\varphi( s )$, and $s$ is a standard Gaussian random variable. Note that the problem defined in \eqref{scprob1} depends on $q_{t,\tau}^\star$ which is given by
\begin{align}\label{qstar_m}
q_{t,\tau}^\star=\frac{ \gamma_2 t_1 \widetilde{\mu}_1 T_1 T_{2,\lambda}(\vt,\vtau) }{\tau_1 + t_1 \widetilde{\mu}_1^2 T_1^2 T_{2,\lambda}(\vt,\vtau) }.
\end{align}
Now, we summarize our main theoretical results in the following theorem.
\begin{theorem}[Noisy Formulation Characterization]\label{ther1}
Suppose that the assumptions in Section \ref{assmp} are all satisfied and the cGEC introduced in Section \ref{CGEC} is valid. Then, the training error converges in probability as follows
\begin{align}
\mathcal{E}_{\text{train}} \xrightarrow{~p\to+\infty~} C^\star(\Delta,\lambda)-\frac{\lambda}{2} \left((q^\star)^2+h^\prime(\lambda) \right), \nonumber
\end{align}
where $C^\star(\Delta,\lambda)$ is the optimal cost of the deterministic problem in \eqref{scprob1}.
Here, the function $h(\cdot)$ is defined as follows
\begin{align}
h(\lambda)=-(q^\star)^2 \Big( \lambda - \frac{1}{T_{2,\lambda}(\vt^\star,\vtau^\star)} \Big)-\eta T_{3,\lambda}(\vt^\star,\vtau^\star).\nonumber
\end{align}
Moreover, the generalization error defined in \eqref{testerr_def} converges in probability to a deterministic function as follows
\begin{align}\label{gen_conv}
{\mathcal{E}}_{\text{test}} \xrightarrow{~p\to+\infty~} \frac{1}{4^\upsilon} \mathbb{E}\left[ \left( \varphi(g_1) -\widehat{\varphi}(g_2) \right)^2 \right],
\end{align}
where $g_1$ and $g_2$ have a bivariate Gaussian distribution with  mean vector $[0,\mu_{0s} \vartheta^\star]$ and covariance matrix $\mC$, defined as follows
\begin{align}
\mC=\begin{bmatrix}
1 & \mu_{1s} T_1 q^{\star} \\
 \mu_{1s} T_1 q^\star &  \mu_{1s}^2 \beta^\star+\mu_{2s}^2 \left((q^\star)^2+h^\prime(\lambda) \right)
\end{bmatrix}.\nonumber
\end{align}
The constant $\vartheta^\star$ satisfies $\vartheta^\star=0$ if $\mu_0=0$ and $\vartheta^\star=\gamma_3/\mu_0$ otherwise. Here, the constants $\mu_{0s}$, $\mu_{1s}$ and $\mu_{2s}$ are defined as $\mu_{0s}=\mathbb{E}[\sigma(z)]$, $\mu_{1s}=\mathbb{E}[z \sigma(z)]$ and $\mu_{2s}^2=\mathbb{E}[\sigma(z)^2]-\mu_{0s}^2-\mu_{1s}^2$, where $z$ is a standard Gaussian random variable. Additionally, the constant $\beta^\star$ can be computed via the following expression
\begin{align}\label{bstar}
\beta^\star&= \frac{1}{V_1+V_3} \Big( V_1 T_1^2 - V_2-V_4 -\lambda + \frac{1}{T_{2,\lambda}(\vt^\star,\vtau^\star)} \Big) (q^\star)^2  \nonumber\\
&+ \frac{\eta T_{3,\lambda}(\vt^\star,\vtau^\star)}{V_1+V_3} - \frac{V_2+V_4+\lambda}{V_1+V_3} h^\prime(\lambda),
\end{align}
where the constants $V_1$, $V_2$, $V_3$ and $V_4$ are defined as follows
\begin{align}
&V_1=\frac{t_1^\star \widetilde{\mu}_1^2}{\tau_1^\star},~V_3=\widehat{\mu}_1^2 \Big( \frac{t_1^\star}{\tau_1^\star \ell}+\frac{t_2^\star (\ell-1)}{\tau_2^\star \ell} \Big)\nonumber\\
&V_2=\frac{t_1^\star \mu_2^2}{\tau_1^\star},V_4\textit{•}=\mu_3^2 \Big( \frac{t_1^\star}{\tau_1^\star \ell}+\frac{t_2^\star(\ell-1)}{\tau_2^\star \ell} \Big).\nonumber
\end{align}
Here, $q^\star=q^\star_{t^\star,\tau^\star}$ is given in \eqref{qstar_m}, $\vt^\star=[t_1^\star,t_2^\star]^\top$ and $\vtau^\star=[ \tau_1^\star,\tau_2^\star ]^\top$. Moreover, $\lbrace t_1^\star,t_2^\star,\tau_1^\star,\tau_2^\star \rbrace$ denotes the optimal solution of the problem defined in \eqref{scprob1}. Also, we treat $q^\star$, $\vt^\star$ and $\vtau^\star$ as constants independent of $\lambda$ when we compute the derivative of the function $h(\cdot)$.
\end{theorem}

To streamline our presentation, we postpone the proof of Theorem \ref{ther1} to Section \ref{dproof}. Note that Theorem \ref{ther1} provides a full asymptotic characterization of the training and generalization errors corresponding to the formulation given in \eqref{mform}. Specifically, it shows that the performance of \eqref{mform} can be fully characterized after solving a deterministic scalar formulation where the cost function depends on the parameters $\ell$ and $\Delta$.
The theoretical predictions stated in Theorem \ref{ther1} are valid for any fixed noise variance $\Delta \geq 0$ and number of noise samples $\ell \geq 1$. Additionally, it is valid for a general family of feature matrices, activation functions and generative models satisfying \eqref{cmodel}. The analysis presented in Section \ref{dproof} shows that the deterministic problem in \eqref{scprob1} is strictly convex-concave. This implies the uniqueness of the optimal solutions of the optimization in \eqref{scprob1}. Next, we study the properties of the noise injection method in \eqref{mform} when $\ell$ grows to infinity slower than  $n$, $p$ and $k$.

\subsection{Noise Regularization Effects}
Now, we consider the setting where $\ell$ grows to infinity slower than the dimensions $n$, $p$ and $k$. We use the theoretical predictions stated in Theorem \ref{ther1} to study the regularization effects of the noise injection method in \eqref{mform}. Our first theoretical result is introduced in the following theorem.
\begin{theorem}[Regularization Effects]\label{thm2}
Suppose that the assumptions in Theorem \ref{ther1} are all satisfied.
Moreover, define the following formulation
\begin{align}\label{form_asy_th}
\min_{\vw\in\mathbb{R}^k}&~\frac{1}{2n} \sum_{i=1}^{n}  \big(y_i-  \vw^\top \widehat{\sigma}(\mF^\top \va_i)  \big)^2+ \frac{1}{2} \norm{\mR^{\frac{1}{2}} \vw}^2 + \tfrac{\lambda}{2} \norm{\vw}^2.
\end{align}
Here, the regularization matrix $\mR$ is defined as follows
\begin{align}
\mR=\widehat{\mu}_1^2 \mF^\top \mF + \mu_3^2 \mI_k,
\end{align}
and the new activation function $\widehat{\sigma}(\cdot)$ satisfies the  properties
\begin{align}
&\mathbb{E}[\widehat{\sigma}(z)]=\mathbb{E}[\sigma(x_1)],~\mathbb{E}[z\widehat{\sigma}(z)]=\mathbb{E}[z\sigma(x_1)] \nonumber\\
&~~~~~~~~~~~\mathbb{E}[\widehat{\sigma}(z)^2]=\mathbb{E}[\sigma(x_1)\sigma(x_2)],
\end{align}
where $x_1=z+\Delta v_1$, $x_2=z+\Delta v_2$ and $z$, $v_1$ and $v_2$ are independent standard Gaussian random variables. Also, define $\widehat{\mathcal{E}}_{\text{train}}$ and $\widehat{\mathcal{E}}_{\text{test}}$ as the training and generalization errors corresponding to the problem in \eqref{form_asy_th}. Then, for any $\zeta>0$, we have the following convergence results
\begin{align}
\begin{cases}
\lim\limits_{\ell \to +\infty} \lim\limits_{p \to +\infty} \mathbb{P}\Big(\abs{{\mathcal{E}}_{\text{train}}-\widehat{\mathcal{E}}_{\text{train}}} < \zeta \Big)=1\\
\lim\limits_{\ell \to +\infty} \lim\limits_{p \to +\infty} \mathbb{P}\Big(\abs{{\mathcal{E}}_{\text{test}}-\widehat{\mathcal{E}}_{\text{test}}} < \zeta \Big)=1,
\end{cases}
\end{align}
where ${\mathcal{E}}_{\text{test}}$ and ${\mathcal{E}}_{\text{train}}$ are the training and generalization errors corresponding to the noisy formulation.
\end{theorem}

To streamline our presentation, we postpone the proof of Theorem \ref{thm2} to Section \ref{dproof}. The above theorem shows that the noisy formulation given in \eqref{mform} is equivalent to a standard formulation with a new activation function and an additional weighted ridge regularization, when $\ell \to +\infty$. It also provides the explicit form of the regularization. This shows that inserting Gaussian noise during the training procedure introduces a regularization that depend on the activation function, the feature matrix and the noise variance. Now, we provide a precise asymptotic characterization of the formulation in \eqref{form_asy_th}. Before stating our asymptotic result, we define the following deterministic problem
\begin{equation}\label{det_pr_lasy}
\begin{aligned}
&\inf_{\substack{\tau_1> 0} } \sup_{\substack{ t_1\geq 0}}~\frac{t_1}{2\tau_1} \big( \gamma_1 -  2 \widetilde{\mu}_1 T_1 \widehat{q}_{t,\tau}^\star \gamma_2 + \widetilde{\mu}_1^2 T_1^2 (\widehat{q}_{t,\tau}^\star)^2 + \mu_0^2 (\vartheta^\star)^2 \\
&-2\mu_0 \vartheta^\star\gamma_3 \big)+\frac{\tau_1 t_1}{2}  - \frac{t_1^2}{2} +\frac{(\widehat{q}_{t,\tau}^\star)^2}{2 \widehat{T}_{2,\lambda}(t_1,\tau_1)}   -\frac{\eta \widehat{T}_{3,\lambda}(t_1,\tau_1)}{2}.
\end{aligned}
\end{equation}
Here, the constant $\vartheta^\star$ satisfies $\vartheta^\star=0$ if $\mu_0=0$ and $\vartheta^\star=\gamma_3/\mu_0$ otherwise, and $T_1$ is defined in Section \ref{st_asylim}. Moreover, the functions $\widehat{q}_{t,\tau}^\star$ and $\widehat{T}_{2,\lambda}(\cdot,\cdot)$ are defined as follows
\begin{equation}\label{qh_t2h}
\begin{aligned}
\widehat{q}_{t,\tau}^\star=\frac{ \gamma_2 t_1 \widetilde{\mu}_1 T_1 \widehat{T}_{2,\lambda}(t_1,\tau_1) }{\tau_1 + t_1 \widetilde{\mu}_1^2 T_1^2 \widehat{T}_{2,\lambda}(t_1,\tau_1) },~\text{and}~\widehat{T}_{2,\lambda}(t_1,\tau_1)= \frac{\delta}{T_1^2} \mathbb{E}_\kappa\Big[ \frac{\kappa}{\widehat{g}_{\kappa,\lambda}(t_1,\tau_1)} \Big] /\Big(1-\frac{\widetilde{\mu}_1^2 t_1 \delta}{ \tau_1} \mathbb{E}_\kappa\Big[ \frac{\kappa}{\widehat{g}_{\kappa,\lambda}(t_1,\tau_1)} \Big] \Big).
\end{aligned}
\end{equation}
Here, the functions $\widehat{T}_{3,\lambda}(\cdot,\cdot)$ and $\widehat{g}_{\kappa,\lambda}(\cdot,\cdot)$ are defined as follows
\begin{align}
&\widehat{T}_{3,\lambda}(t_1,\tau_1)=t_1^2 \mathbb{E}_\kappa\Big[{ (\widetilde{\mu}_1^2 \kappa + \mu_2^2})/{\widehat{g}_{\kappa,\lambda}(t_1,\tau_1)}\Big],\nonumber\\
&\widehat{g}_{\kappa,\lambda}(t_1,\tau_1)=\frac{t_1}{\tau_1} \Big( \widetilde{\mu}_1^2 \kappa + \mu_2^2 \Big)+ \Big( \widehat{\mu}_1^2 \kappa + \mu_3^2 \Big) + \lambda, \nonumber
\end{align}
where the expectations are taken over the probability distribution $\mathbb{P}_\kappa(\cdot)$ defined in Assumption \ref{itm:ass_F}.
Now, we summarize the asymptotic properties of the limiting formulation in \eqref{form_asy_th} in the following theorem.
\begin{lemma}[Limiting Formulation Characterization]\label{lem1}
Suppose that the assumptions in Theorem \ref{ther1} are all satisfied. Then, the training error corresponding to the limiting formulation in \eqref{form_asy_th} converges in probability as follows
\begin{align}
\widehat{\mathcal{E}}_{\text{train}} \xrightarrow{~p\to+\infty~} \widehat{C}^\star(\Delta,\lambda)-\frac{\lambda}{2} \left((\widehat{q}^\star)^2+\widehat{h}^\prime(\lambda) \right), \nonumber
\end{align}
where $\widehat{C}^\star(\Delta,\lambda)$ is the optimal cost of the deterministic problem in \eqref{det_pr_lasy}.
Here, the function $\widehat{h}(\cdot)$ is defined as follows
\begin{align}
\widehat{h}(\lambda)=-(\widehat{q}^\star)^2 \Big( \lambda - \frac{1}{\widehat{T}_{2,\lambda}(t_1^\star,\tau_1^\star)} \Big)-\eta \widehat{T}_{3,\lambda}(t_1^\star,\tau_1^\star).\nonumber
\end{align}
Moreover, the generalization error corresponding to the limiting formulation in \eqref{form_asy_th} converges in probability to a deterministic function as follows
\begin{align}\label{gen_conv}
\widehat{\mathcal{E}}_{\text{test}} \xrightarrow{~p\to+\infty~} \frac{1}{4^\upsilon} \mathbb{E}\left[ \left( \varphi(g_1) -\widehat{\varphi}(g_2) \right)^2 \right],
\end{align}
where $g_1$ and $g_2$ have a bivariate Gaussian distribution with  mean vector $[0,\mu_{0s} \vartheta^\star]$ and covariance matrix $\mC$, defined as follows
\begin{align}
\mC=\begin{bmatrix}
1 & \mu_{1s} T_1 \widehat{q}^{\star} \\
 \mu_{1s} T_1 \widehat{q}^\star &  \mu_{1s}^2 \widehat{\beta}^\star+\mu_{2s}^2 \left((\widehat{q}^\star)^2+\widehat{h}^\prime(\lambda) \right)
\end{bmatrix}.\nonumber
\end{align}
The constant $\vartheta^\star$ satisfies $\vartheta^\star=0$ if $\mu_0=0$ and $\vartheta^\star=\gamma_3/\mu_0$ otherwise. Here, the constants $\mu_{0s}$, $\mu_{1s}$ and $\mu_{2s}$ are defined as $\mu_{0s}=\mathbb{E}[\sigma(z)]$, $\mu_{1s}=\mathbb{E}[z \sigma(z)]$ and $\mu_{2s}^2=\mathbb{E}[\sigma(z)^2]-\mu_{0s}^2-\mu_{1s}^2$, where $z$ is a standard Gaussian random variable. Additionally, the constant $\widehat{\beta}^\star$ can be computed via the following expression
\begin{align}\label{bstar}
\widehat{\beta}^\star&= \frac{1}{V_1+V_3} \Big( V_1 T_1^2 - V_2-V_4 -\lambda + \frac{1}{\widehat{T}_{2,\lambda}(t_1^\star,\tau_1^\star)} \Big) (\widehat{q}^\star)^2  \nonumber\\
&+ \frac{\eta \widehat{T}_{3,\lambda}(t_1^\star,\tau_1^\star)}{V_1+V_3} - \frac{V_2+V_4+\lambda}{V_1+V_3} \widehat{h}^\prime(\lambda),
\end{align}
where the constants $V_1$, $V_2$, $V_3$ and $V_4$ are defined as follows
\begin{align}
&V_1=\frac{t_1^\star \widetilde{\mu}_1^2}{\tau_1^\star},~V_3=\widehat{\mu}_1^2,~V_2=\frac{t_1^\star \mu_2^2}{\tau_1^\star},V_4=\mu_3^2 .\nonumber
\end{align}
Here, $\widehat{q}^\star=\widehat{q}^\star_{t^\star,\tau^\star}$ is given in \eqref{qh_t2h}. Moreover, $\lbrace t_1^\star,\tau_1^\star \rbrace$ denotes the optimal solution of the problem defined in \eqref{det_pr_lasy}. Also, we treat $\widehat{q}^\star$, $t_1^\star$ and $\tau_1^\star$ as constants independent of $\lambda$ when we compute the derivative of the function $\widehat{h}(\cdot)$. 
\end{lemma}

The proof of Lemma \ref{lem1} is provided in Section \ref{dproof}. The results in Theorem \ref{thm2} and Lemma \ref{lem1} are based on the asymptotic predictions stated in Theorem \ref{ther1}. Specifically, we show in Section \ref{dproof} that the asymptotic problem corresponding to the noisy formulation in \eqref{scprob1} converges to the deterministic problem in \eqref{det_pr_lasy}, when $\ell$ grows to infinity. Then, we show that the deterministic problem in \eqref{det_pr_lasy} is the asymptotic limit of the formulation in \eqref{form_asy_th} using the CGMT framework. The analysis presented in Section \ref{dproof} shows that the deterministic problem in \eqref{det_pr_lasy} is strictly convex-concave. This implies the uniqueness of its optimal solutions. 

\section{Simulation Results}\label{sim_res}
In this part, we provide additional simulation examples to verify our asymptotic results stated in Theorem \ref{ther1}, Theorem \ref{thm2} and Lemma \ref{lem1}. Our predictions stated in Section \ref{pr_analysis} are valid for a general family of feature matrices, activation functions and generative models satisfying \eqref{cmodel}. We specialize our general results to popular learning models. 

In particular, we consider two families of feature matrices. We consider feature matrices that can be expressed as $\mF= d \mV$, where:
{\bf (a)} The scalar $d$ satisfies $d=1/\sqrt{p}$ and the matrix $\mV$ has independent standard Gaussian components. We refer to this matrix as the \textit{Gaussian feature matrix}. {\bf (b)} The scalar $d$ satisfies $d=\sqrt{3/p}$ and the matrix $\mV$ has independent uniformly distributed components in $[-1~1]$. We refer to this matrix as the \textit{uniform feature matrix}.

Also, we consider two popular regression and classification models. For the regression model, we assume that $\varphi(\cdot)$ is the ReLu function and $\widehat{\varphi}(\cdot)$ is the identity function. For the classification model, we assume that $\varphi(\cdot)$ is the sign function with possible sign flip with probability $\theta$ and $\widehat{\varphi}(\cdot)$ is the sign function.

\subsection{Correlated Gaussian Equivalence}
In this simulation example, we consider a particular case of the cGEC to illustrate the Gaussian equivalence. Here, we consider a projected version of the probability density function $\mathbb{P}(\nu_1,\nu_2,\nu_3)$, defined in Section \ref{CGEC}. Specifically, we consider the probability density function of the random variable $\nu=\nu_1+\nu_2+\nu_3$. 
\begin{figure}[ht]
    \centering
    \subfigure[]{\label{fcgeca}
        \includegraphics[width=0.4\linewidth]{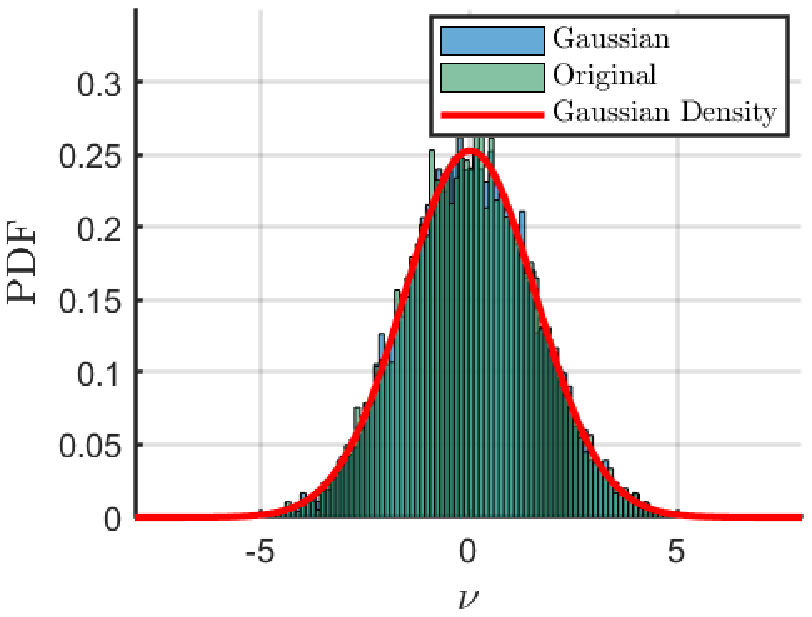}}
        \subfigure[]{\label{fcgecb}
        \includegraphics[width=0.4\linewidth]{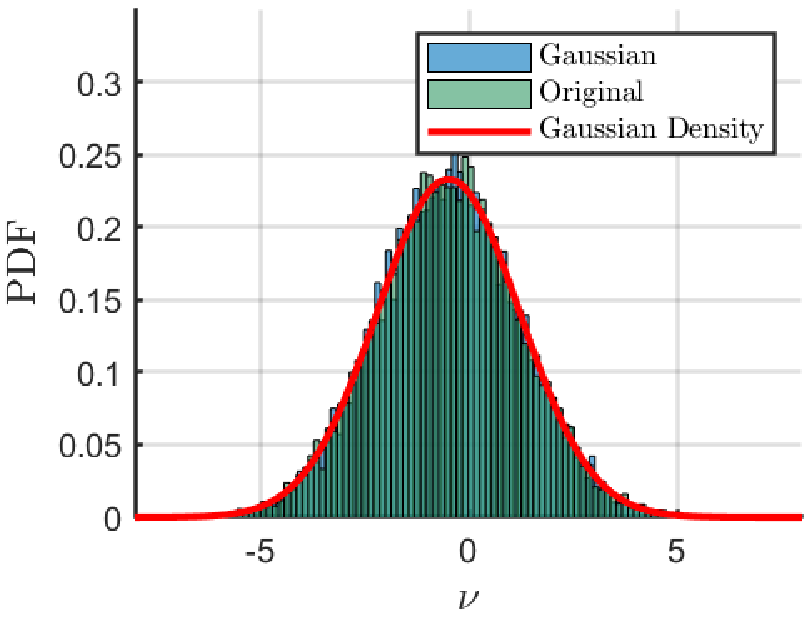}}

    \caption{\textit{Original}: measures the distribution of $\nu=\nu_1+\nu_2+\nu_3$. \textit{Gaussian}: measures the distribution of $\nu=\nu_{g,1}+\nu_{g,2}+\nu_{g,3}$. \textit{Gaussian density}: denotes the Gaussian pdf with the corresponding first two moments of $\nu$. {\bf (a)} $\mF$ is the Gaussian feature matrix and $\sigma(\cdot)$ is the tanh function. {\bf (b)} $\mF$ is the uniform feature matrix in $[-\sqrt{3}~ \sqrt{3}]$ and $\sigma(\cdot)$ is the ReLu function. We take $p=10^4$, $\alpha=1.1$, $\Delta=1$ and $\eta=1.3$.}
        \label{fcgec}
\end{figure}
Figure \ref{fcgec} considers two different feature matrices and activation functions. It also compares the probability density function of the random variable $\nu=\nu_1+\nu_2+\nu_3$ with the probability density function of the random variable $\nu_g=\nu_{g,1}+\nu_{g,2}+\nu_{g,3}$. Figure \ref{fcgec} shows that the random variable $\nu$ is Gaussian with the same first two moments of $\nu_g$ for particular choices of $\mF$, $\sigma(\cdot)$, $\vxi$ and $\vw$. This provides a particular illustration of the cGEC.

\subsection{Limiting Performance}
Our fourth simulation considers the non--linear regression model. Figure \ref{figasym} compares the numerical predictions and our predictions stated in Theorem \ref{thm2} and Lemma \ref{lem1}. 
\begin{figure}[ht]
    \centering
    \subfigure[]{\label{figasyma}
    \includegraphics[width=0.4\linewidth]{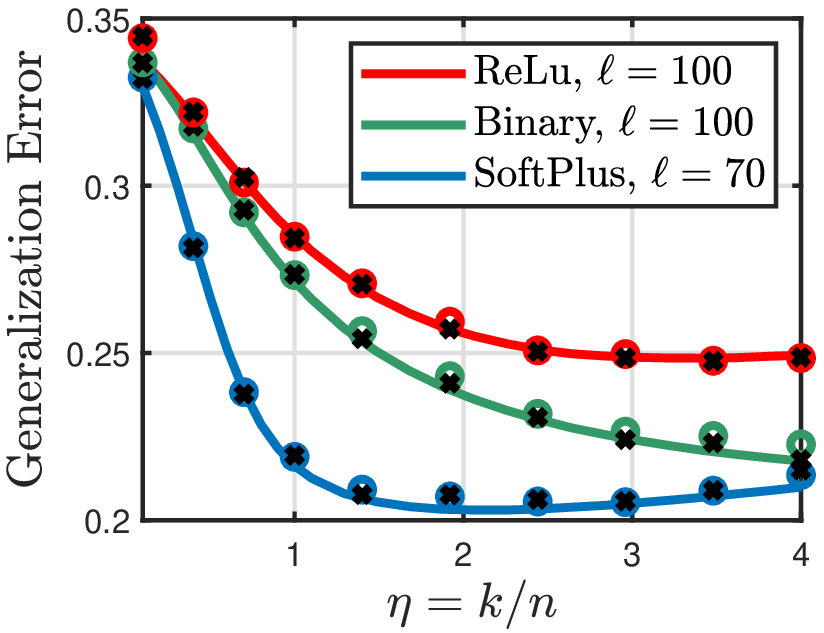}
    }
    \subfigure[]{\label{figasymb}
        \includegraphics[width=0.4\linewidth]{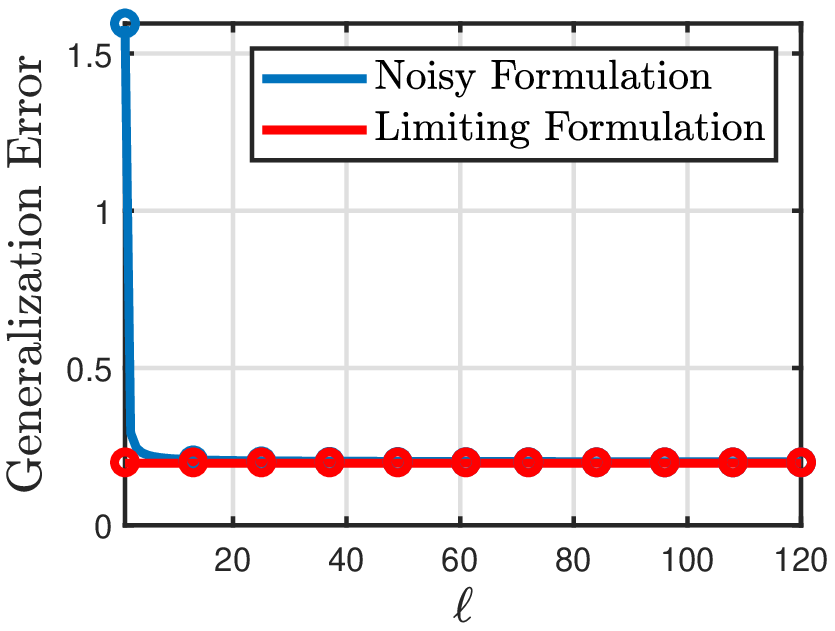}}

    \caption{Solid line: Theoretical predictions. Circle: Numerical simulations for \eqref{mform} in \ref{figasyma} and for both \eqref{mform} and \eqref{form_asy} in \ref{figasymb}. Black cross: Numerical simulations for \eqref{form_asy}.  {\bf (a)} $p=500$, $\Delta=0.4$, $\alpha=1.5$ and $\lambda=10^{-2}$. {\bf (b)} $p=500$, $\Delta=0.6$, $\alpha=2$, $\lambda=10^{-4}$, $\eta=1$ and $\sigma(\cdot)$ is the SoftPlus. Binary denotes the binary step activation. $\mF$ is the Gaussian feature matrix. The number of Monte Carlo trials is $100$.}
        \label{figasym}
\end{figure}
This simulation example first provides an empirical verification of the theoretical predictions in Theorem \ref{thm2} and Lemma \ref{lem1}. It particularly shows that our predictions are in excellent agreement with the empirical results for \eqref{mform} and \eqref{form_asy}. Furthermore, note that the performance of the deterministic formulation given in  \eqref{det_pr_lasy} is achieved with a moderate number of noise samples, i.e. $\ell=70$ and $\ell=100$. This verifies the results stated in Theorem \ref{thm2} and Lemma \ref{lem1} and  provides an empirical verification of the cGEC introduced in Section \ref{CGEC}. Figure \ref{figasyma} further shows that the considered noisy formulation can asymptotically mitigate the double descent in the generalization error for an appropriately selected activation function and fixed noise variance. Specifically, note that the ReLu and binary activation functions lead to a decreasing generalization performance which is not the case for the SoftPlus activation. Figure \ref{figasymb} illustrates the convergence behavior of the generalization error corresponding to \eqref{mform} for the SoftPlus activation and fixed $\eta$. It particularly shows that the generalization error of \eqref{mform} converges to the generalization error of \eqref{form_asy} when $\ell$ grows to infinity. Moreover, note that the limit is already achieved with a small value of $\ell$. This verifies the predictions in Theorem \ref{thm2}.

\subsection{Impact of the Noise Variance}
In this simulation example, we study the effects of the noise variance $\Delta$ on the generalization error corresponding to the noisy formulation and the limiting formulation. Here, we consider the binary classification model. Figure \ref{figvar} compares the numerical predictions and our theoretical predictions stated in Theorem \ref{ther1}, Theorem \ref{thm2} and Lemma \ref{lem1}. 
\begin{figure}[ht]
    \centering
    \subfigure[]{\label{figvara}
    \includegraphics[width=0.31\linewidth]{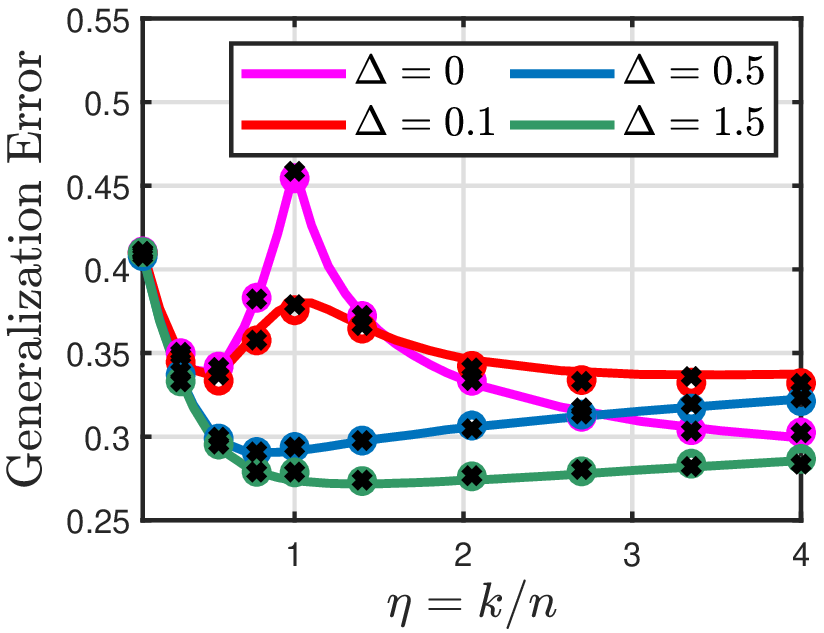}
    }
    \subfigure[]{\label{figvarb}
        \includegraphics[width=0.31\linewidth]{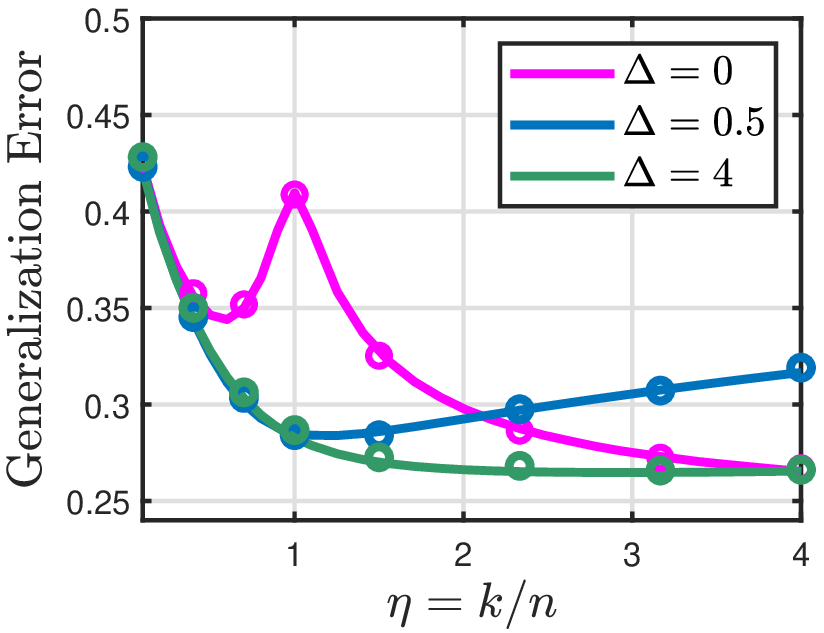}}
    \subfigure[]{\label{figvarc}
            \includegraphics[width=0.31\linewidth]{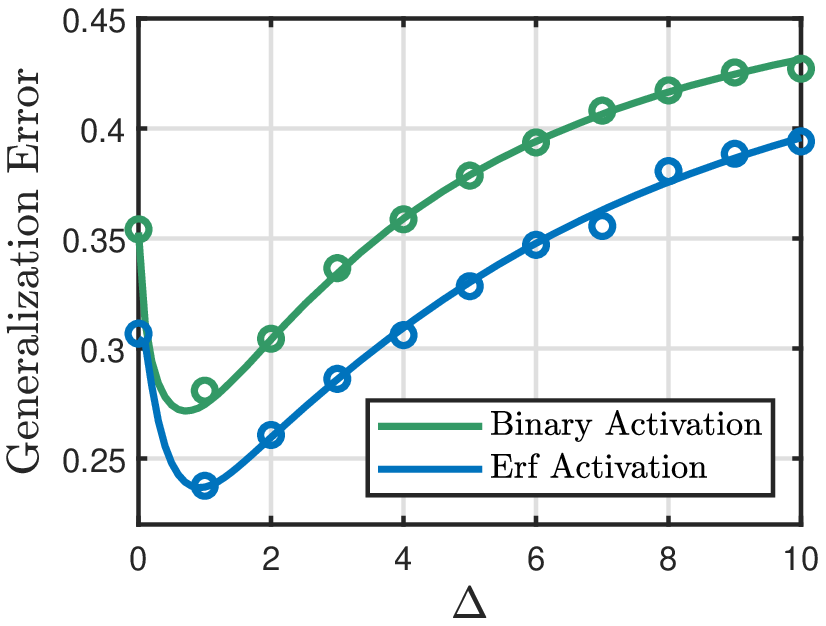}}
            
    \caption{Solid line: Theoretical predictions. Circle: Numerical simulations for \eqref{mform} in \ref{figvara} and \ref{figvarc} and for \eqref{form_asy} in \ref{figvarb}. Black cross: Numerical simulations for \eqref{gform}. {\bf (a)} $\mF$ is the Gaussian feature matrix and $\sigma(\cdot)$ is the tanh activation function. We set $p=400$, $\ell=50$, $\alpha=2$, $\theta=0.1$ and $\lambda=10^{-5}$. {\bf (b)} $\mF$ is the uniform feature matrix and $\sigma(\cdot)$ is the SoftPlus activation. We set $p=1500$, $\alpha=1.5$, $\theta=0$ and $\lambda=10^{-4}$. {\bf (c)} $\mF$ is the uniform feature matrix. We set $p=600$, $\alpha=2$, $\theta=0$, $\ell=20$, $\eta=1.5$ and $\lambda=10^{-4}$. The number of Monte Carlo trials is $100$.}
        \label{figvar}
\end{figure}
It provides another empirical verification of our  theoretical predictions since our results are in excellent agreement with the actual performance of the considered formulations. It also provides an empirical verification of the cGEC discussed in Section \ref{CGEC}. Figure \ref{figvara} studies the effects of the noise variance $\Delta$ on the generalization error corresponding to the noisy formulation for fixed $\ell$. Note that increasing the noise variance improves the generalization error especially at low $\eta$. Figure \ref{figvara} also suggests that an optimized noise variance can reduce the effects of the double descent phenomenon. Now, Figure \ref{figvarb} studies the effects of the noise variance $\Delta$ on the generalization error corresponding to the limiting formulation. We can see that the generalization error increases after reaching a minimum for $\Delta=0.5$. For $\Delta=4$, observe that the generalization error is  deceasing. This suggests that the double descent phenomenon can be mitigated for an appropriately selected noise variance. Figure \ref{figvarc} also shows that the generalization error corresponding to the noisy formulation has a unique minimum as a function of $\Delta$ for the considered activation functions. This suggests that the optimization problem over the noise variance $\Delta$ has a unique solution. This can simplify the design of an efficient optimization scheme of the generalization error in terms of $\Delta$.  

\subsection{Alternative Formulations}
Now, we consider the binary classification model, where $\theta=0$. We compare the performance of the noisy formulation given in \eqref{mform} and the dropout technique. In this paper, we consider the following version of the dropout method 
\begin{align}
\min_{\vw\in\mathbb{R}^k}& \frac{1}{2n \ell} \sum_{i=1}^{n} \sum_{j=1}^{\ell} \big(y_i-\vw^\top \sigma\big(\mD_{ij}\mF^\top \va_i \big) \big)^2+\tfrac{\lambda}{2} \norm{\vw}^2,\nonumber
\end{align}
where $\lbrace \mD_{ij} \rbrace_{i,j}^{n,\ell}$ are diagonal matrices with independent and identically distributed diagonal entries drawn from the distribution, $\mathbb{P}(d=1)=1-\epsilon$ and $\mathbb{P}(d=0)=\epsilon$, where $\epsilon$ denotes the probability of dropping a unit. The above formulation is similar to the one considered in \cite{dropout,implicitdp}.

\subsubsection{Performance Comparison}
In Figure \ref{figdrop}, we compare the performance of the noisy formulation and the dropout formulation for four different activation functions.
\begin{figure}[ht]
    \centering
    \subfigure[]{\label{figdropa}
        \includegraphics[width=0.37\linewidth]{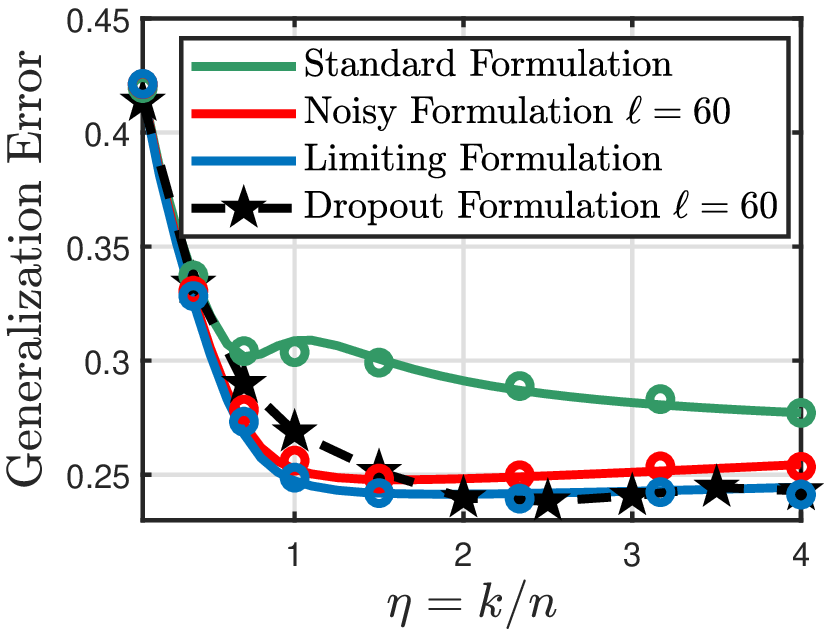}}
    \subfigure[]{\label{figdropb}
    \includegraphics[width=0.37\linewidth]{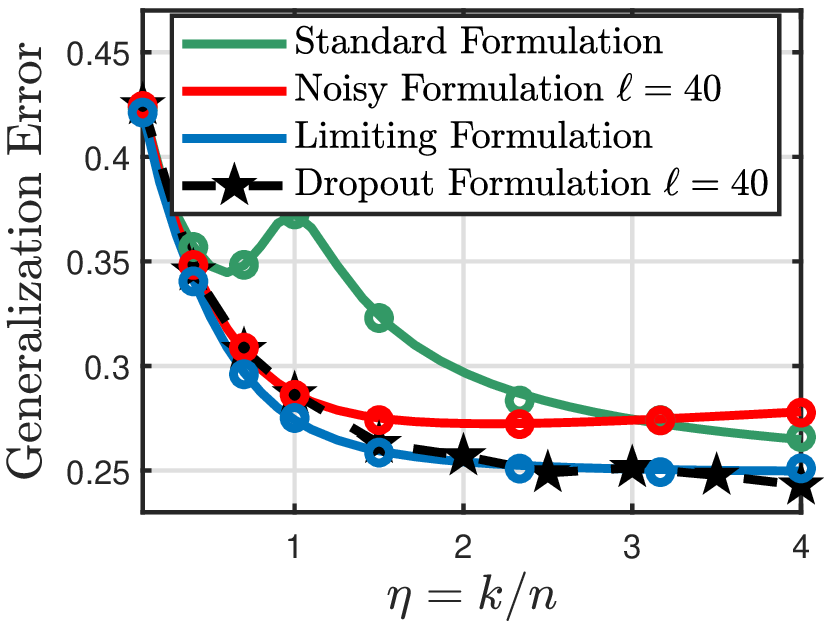}
    } \\   
    \subfigure[]{\label{figdropc}
    \includegraphics[width=0.37\linewidth]{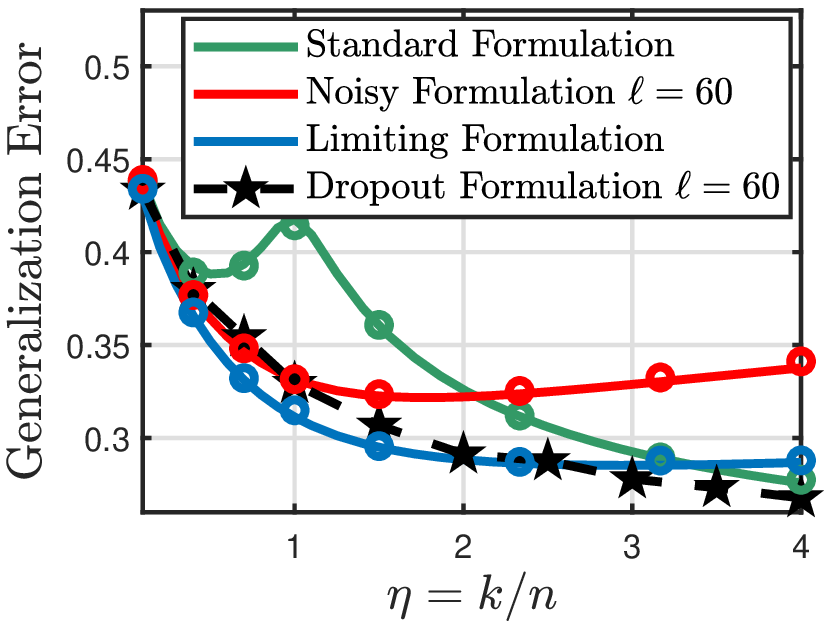}
    }
    \subfigure[]{\label{figdropd}
        \includegraphics[width=0.37\linewidth]{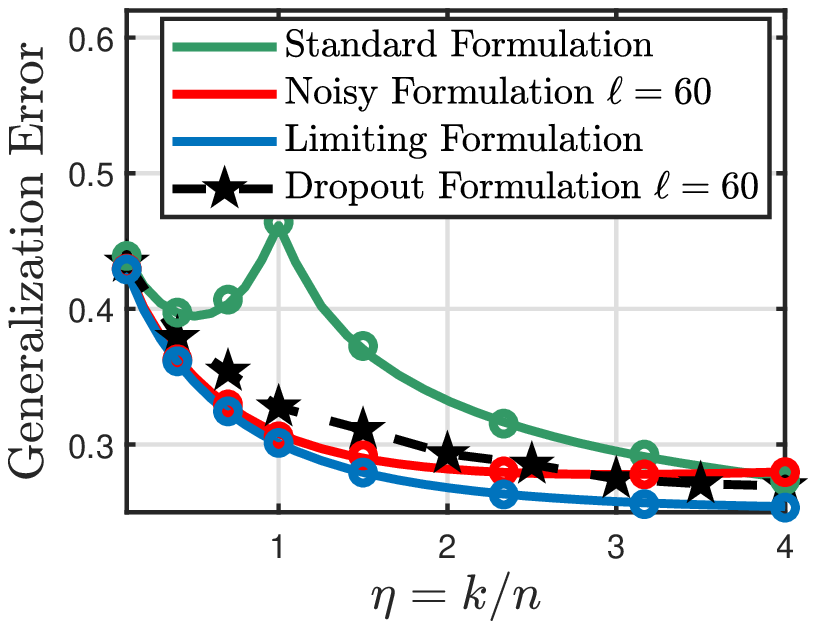}}
            
    \caption{Solid line: Theoretical predictions. Circle: Numerical simulations for \eqref{mform} and \eqref{form_asy}. Hexagram: Numerical simulations for the dropout formulation. {\bf (a)} Sigmoid activation and  $p=600$, $\alpha=1.4$, $\Delta=1.5$, $\lambda=10^{-4}$, $\ell=60$ and $\epsilon=0.3$. {\bf (b)} Erf activation and $p=600$, $\alpha=1.4$, $\Delta=2$, $\lambda=10^{-3}$, $\ell=40$ and $\epsilon=0.3$. {\bf (c)} ReLu activation and $p=600$, $\alpha=1.4$, $\Delta=3.5$, $\lambda=10^{-3}$, $\ell=60$ and $\epsilon=0.3$. {\bf (d)} Sign activation and $p=600$, $\alpha=1.6$, $\Delta=1.5$, $\lambda=10^{-4}$, $\ell=60$ and $\epsilon=0.4$. $\mF$ is the uniform feature matrix. The results are averaged over $35$ independent Monte Carlo trials.}
        \label{figdrop}
\end{figure}
First, we can notice that our asymptotic results provided in Theorem \ref{ther1}, Theorem \ref{thm2} and Lemma \ref{lem1} match with the actual performance of the noisy formulation and its limiting formulation. This gives another empirical verification of our theoretical predictions. Figure \ref{figdrop} suggests that the noisy and dropout formulations significantly improve the generalization performance of the standard formulation for an appropriately selected activation function, $\Delta$, $\ell$ and $\epsilon$. Figure \ref{figdropa} shows that the noisy and dropout formulations have a similar generalization performance for fixed $\ell=60$ and for the sigmoid activation function. Moreover, it shows that the dropout formulation approaches the generalization performance of the limiting formulation at high $\eta$. Moreover, Figures \ref{figdropb} and \ref{figdropc} shows that the dropout method provides a largely better performance as compared to the noisy formulation for fixed $\ell=40$ and $\ell=60$ and for the Erf and the ReLu activation functions, respectively. It also suggests that the limiting and dropout formulations have a similar generalization performance where the dropout method is better at high $\eta$ for the ReLu activation function. Now, Figure \ref{figdropd} considers the sign activation function and shows that the noisy formulation provides a better generalization performance as compared to the dropout formulation at low $\eta$ and for a fixed number of noise injections, $\ell=60$. We can also see that the limiting formulation generalizes better than the dropout method for the considered parameters. This simulation example particularly suggests that the performance of the noisy and dropout formulations depends on the activation function. Moreover, the dropout formulation have a similar generalization performance as compared to the limiting performance for the considered parameters.

\subsubsection{Convergence Behavior}
In the last simulation example, we study the convergence behavior of the noisy and dropout formulations for different activation functions. Figure \ref{figdrop_cb} first shows that our theoretical predictions stated in Theorem \ref{ther1}, Theorem \ref{thm2} and Lemma \ref{lem1} match with the actual performance of the noisy formulation and its limiting formulation. This gives another empirical verification of our predictions.
\begin{figure}[ht]
    \centering
    \subfigure[]{\label{figdrop_cba}
    \includegraphics[width=0.37\linewidth]{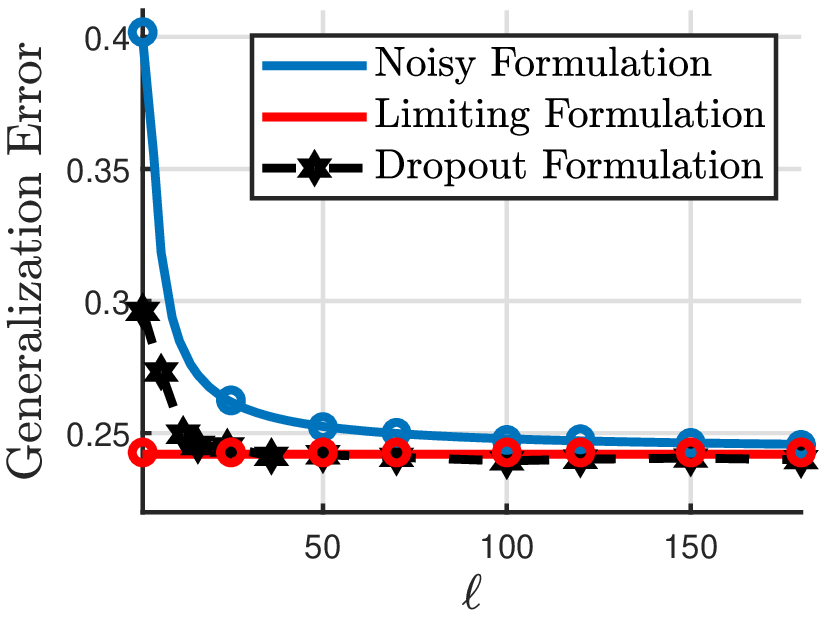}
    }
    \subfigure[]{\label{figdrop_cbb}
    \includegraphics[width=0.37\linewidth]{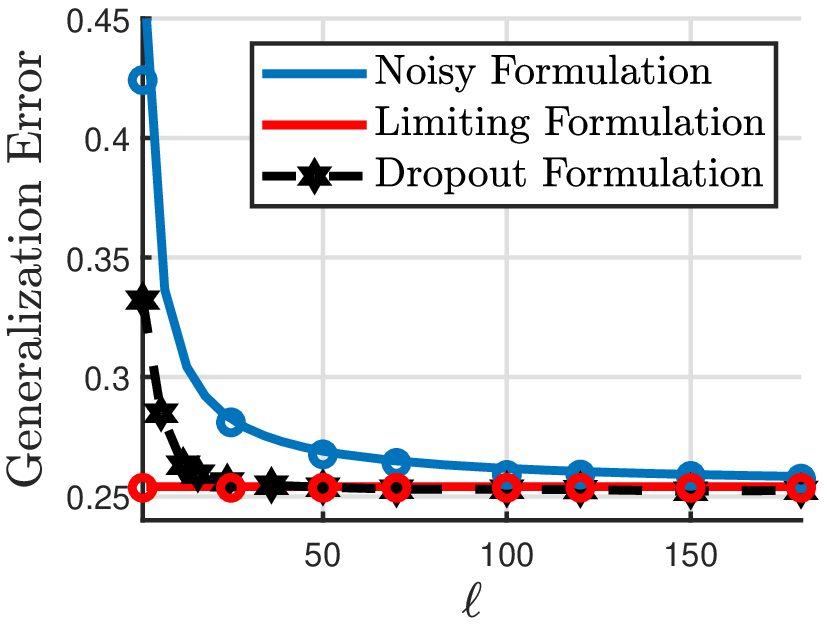}}
    \\
    \subfigure[]{\label{figdrop_cbc}
        \includegraphics[width=0.37\linewidth]{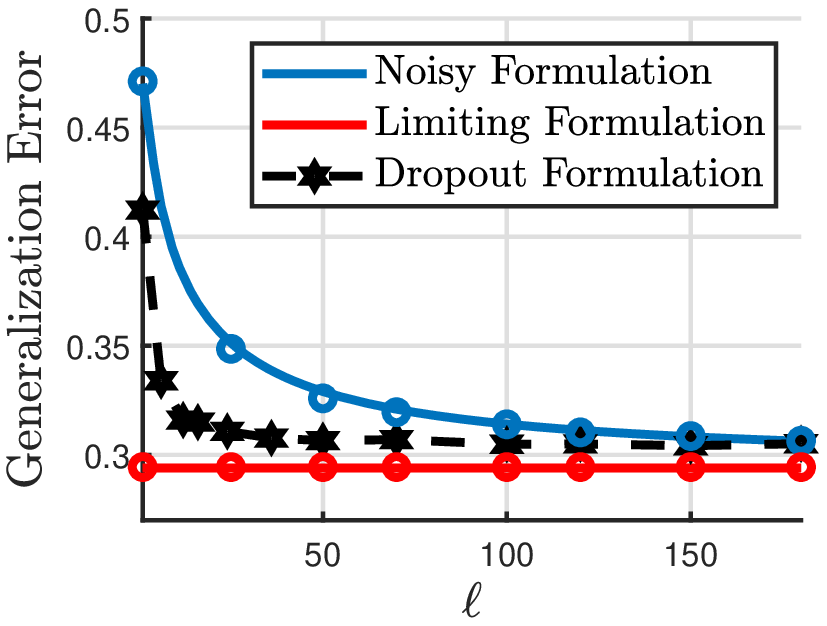}}
    \subfigure[]{\label{figdrop_cbd}
        \includegraphics[width=0.37\linewidth]{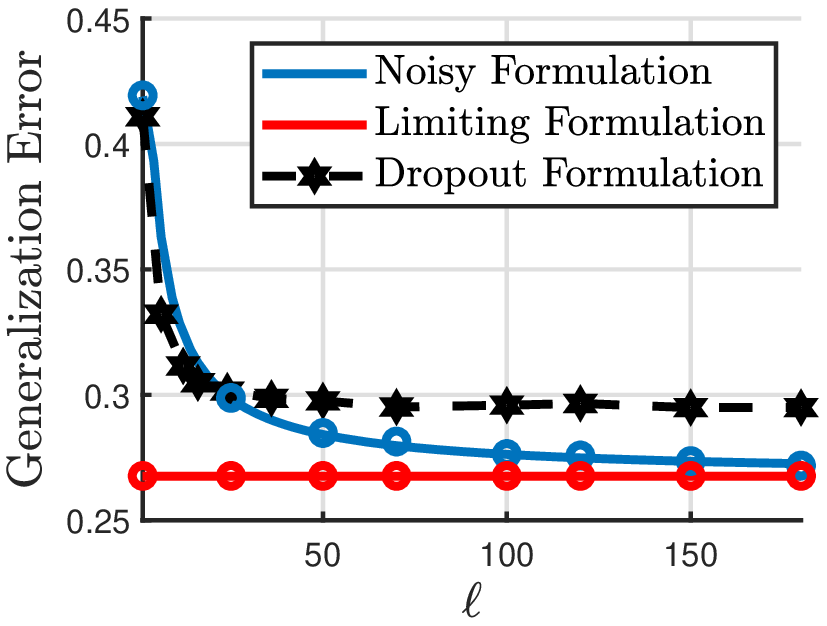}}
        
    \caption{Solid line: Theoretical predictions. Circle: Numerical simulations for \eqref{mform} and \eqref{form_asy}. Hexagram: Numerical simulations for the dropout formulation. {\bf (a)} Sigmoid activation function and we set the parameters as $p=500$, $\alpha=1.4$, $\Delta=1.5$, $\lambda=10^{-4}$, $\eta=2.5$ and $\epsilon=0.3$. {\bf (b)} Erf activation function and we set $p=600$, $\alpha=1.4$, $\Delta=2$, $\lambda=10^{-3}$, $\eta=2$ and $\epsilon=0.3$. {\bf (c)} ReLu activation function and we set the parameters as $p=500$, $\alpha=1.4$, $\Delta=3.5$, $\lambda=10^{-3}$, $\eta=1.5$ and $\epsilon=0.3$. {\bf (d)} Sign activation function and we set the parameters as $p=500$, $\alpha=1.6$, $\Delta=1.5$, $\lambda=10^{-4}$, $\eta=2$ and $\epsilon=0.4$. The feature matrix $\mF$ is the uniform feature matrix. The results are averaged over $35$ independent Monte Carlo trials.}
        \label{figdrop_cb}
\end{figure}
Figure \ref{figdrop_cb} studies the convergence properties of both approaches as a function of $\ell$ and for different activation functions. It particularly suggests that the dropout method has a better convergence rate as compared to the noisy formulation. Moreover, Figure \ref{figdrop_cb}  suggests that the noisy and dropout formulations have comparable generalization performance when the number of noise injections grows to infinity.

\section{Technical Details: Analysis of the Noisy Formulation}\label{dproof}
In this part, we provide a rigorous proof of the predictions stated in Theorem \ref{ther1}, Theorem \ref{thm2} and Lemma \ref{lem1}. To this end, we suppose that the assumptions considered in Sections \ref{CGEC} and \ref{assmp} are all satisfied. We derive our theoretical results using an extended version of the CGMT framework which we refer to as the multivariate CGMT.

\subsection{Multivariate Convex Gaussian Min--Max Theorem}\label{MCGMT_fram}
To derive the asymptotic results stated in Theorem \ref{ther1}, Theorem \ref{thm2} and Lemma \ref{lem1}, we use an extended version of the CGMT framework introduced in \cite{chris:151}. 
The CGMT is used to accurately analyze a generally hard primary formulation by introducing an asymptotically equivalent auxiliary optimization problem. In this paper, we consider primary optimization problems of the following form 
\begin{equation}\label{MPO}
\Phi_k=\min\limits_{\vw \in \mathcal{S}_{\vw}} \max\limits_{\vu\in\mathcal{S}_{\vu}} \sum_{i=1}^{\ell} \vu_i \mG_i \vw_i + \psi(\vw,\vu),
\end{equation}
where $\vu_i \in \mathbb{R}^{n_i}$ and $\vw_i \in \mathbb{R}^{k_i}$ are optimization variables and $\mG_i \in \mathbb{R}^{n_i\times k_i}$ has independent standard Gaussian random components, for any $i\in\lbrace 1\dots,\ell \rbrace$. Additionally, the vectors $\vw$ and $\vu$ are formed by the concatenation of the vectors $\lbrace \vw_i \rbrace_{i=1}^{\ell}$ and $\lbrace \vu_i \rbrace_{i=1}^{\ell}$, respectively. We refer to the formulation in \eqref{MPO} as the multivariate primary optimization (multivariate PO). We show that the corresponding multivariate auxiliary optimization (multivariate AO) is given by
\begin{align}\label{MAO}
\phi_k=\min\limits_{\vw \in \mathcal{S}_{\vw}} \max\limits_{\vu\in\mathcal{S}_{\vu}} &\sum_{i=1}^{\ell} \norm{\vu_i} \vg_i^\top \vw_i + \sum_{i=1}^{\ell} \norm{\vw_i} \vh_i^\top \vu_i + \psi(\vw,\vu),
\end{align}
where $\vg_i\in\mathbb{R}^{k_i}$ and $\vh_i\in\mathbb{R}^{n_i}$ are independent standard Gaussian random vectors, for any $i\in\lbrace 1\dots,\ell \rbrace$. Here, we assume that $\mG_i \in\mathbb{R}^{ {n_i}\times  {k_i}}$, $\vg_i \in \mathbb{R}^{ {k_i}}$ and $\vh_i\in\mathbb{R}^{ {n_i}}$, are all independent, the feasibility sets $\mathcal{S}_{\vw}\subset\R^{ {k}}$ and $\mathcal{S}_{\vu}\subset\R^{ {n}}$ are convex and compact, and the function $\psi: \mathbb{R}^{ {k}} \times \mathbb{R}^{ {n}} \to \mathbb{R}$ is continuous \emph{convex-concave} on $\mathcal{S}_{\vw}\times \mathcal{S}_{\vu}$, where $k=\sum_{i=1}^{\ell} k_i$ and $n=\sum_{i=1}^{\ell} n_i$. Now, we summarize our theoretical result in the following theorem.
\begin{theorem}[Multivariate CGMT]\label{mcgmt}
Assume that the above assumptions are all satisfied. For any fixed $\ell \geq 1$ and $\epsilon>0$, consider an open set $\mathcal{S}_{p,\epsilon}$.
Moreover, define the set $\mathcal{S}^c_{p,\epsilon}={\mathcal{S}}_{\vw} \setminus \mathcal{S}_{p,\epsilon}$. Let $\phi_k$ and $\phi^c_{k}$ be the optimal cost values of the multivariate AO formulation in \eqref{MAO} with feasibility sets ${\mathcal{S}}_{\vw}$ and $\mathcal{S}^c_{p,\epsilon}$, respectively. Assume that the following properties are all satisfied
\begin{itemize}
\item[\bf (1)] There exists a constant $\phi$ such that the optimal cost $\phi_k$ converges in probability to $\phi$ as $k$ goes to $+\infty$.
\item[\bf (2)] There exists a constant $\phi^c$ such that the optimal cost $\phi^c_{k}$ converges in probability to $\phi^c$ as $k$ goes to $+\infty$, for any fixed $\epsilon>0$.
\item[\bf (3)] There exists a positive constant $\zeta>0$ such that $\phi^c \geq \phi+\zeta$, for any fixed $\epsilon>0$.
\end{itemize}
Then, the following convergence in probability holds
\begin{equation}
\abs{ \Phi_k -\phi_k } \overset{k \to +\infty}{\longrightarrow} 0,~\text{and}~\mathbb{P}( \widehat{\vw}_{k} \in \mathcal{S}_{p,\epsilon} )  \overset{k\to\infty}{\longrightarrow} 1,\nonumber
\end{equation}
for any fixed $\epsilon>0$, where $\Phi_k$ and $\widehat{\vw}_{k}$ are the optimal cost and the optimal solution of the multivariate PO formulation in \eqref{MPO}.
\end{theorem}
The above theorem allows us to analyze the generally easy multivariate AO formulation given in \eqref{MAO} to infer asymptotic properties of the generally hard multivariate PO formulation in \eqref{MPO}. The proof of Theorem \ref{mcgmt} follows by showing that the formulation in \eqref{MAO} and the following formulation 
\begin{align}\label{pMPO}
\widehat{\Phi}_k&=\min\limits_{\vw \in \mathcal{S}_{\vw}} \max\limits_{\vu\in\mathcal{S}_{\vu}} \sum_{i=1}^{\ell} \vu_i \mG_i \vw_i + \psi(\vw,\vu) +\sum_{i=1}^{\ell} \norm{\vu_i} \norm{\vw_i} g_i,\nonumber
\end{align}
satisfy all the assumptions in \cite{gordon}, where $\lbrace g_i \rbrace_{i=1}^{\ell}$ are independent standard Gaussian random variables. Then, following the same analysis in \cite{chris:152} and \cite{chris:151}, we can show that for any $\chi \in \mathbb{R}$ and $c>0$, it holds
\begin{equation}\label{eq:cgmt}
\mathbb{P}\left( \abs{\Phi_k-\chi} > c\right) \leq 2^\ell~ \mathbb{P}\left(  \abs{\phi_k-\chi} > c \right).
\end{equation}
Combining this result with the assumptions of Theorem \ref{mcgmt} completes the proof. We omit the detailed proof since it is similar to the analysis in \cite{chris:152} and \cite{chris:151}. We refer to Theorem \ref{mcgmt} as the multivariate convex Gaussian min-max theorem (multivariate CGMT). 

Next, we use the multivariate CGMT to rigorously prove the technical results provided in Theorem \ref{ther1}, Theorem \ref{thm2} and Lemma \ref{lem1}. Our approach is to reformulate the Gaussian formulation in \eqref{gform} in the form of the multivariate PO problem given in \eqref{MPO}. Then, use the multivariate CGMT framework to show that the formulation in \eqref{mform} is asymptotically equivalent to an easier formulation that can be written in the form of the multivariate AO  problem given in \eqref{MAO}. The next step is to show that the multivariate AO formulation converges in probability to a deterministic problem that can be expressed in the form of the formulation given in \eqref{scprob1}. 

\subsection{Asymptotic Analysis of the Noisy Formulation}
In this part, we provide the technical steps to obtain the theoretical results stated in Theorem \ref{ther1}. Specifically, we use the multivariate CGMT framework to precisely analyze the noisy formulation introduced in \eqref{mform}. Next, we suppose that the assumptions introduced in Section \ref{assmp} are all satisfied.
\subsubsection{Formulating the Multivariate Primary Formulation}\label{for_mao}
Based on the cGEC introduced in Section \ref{CGEC}, it suffices to precisely analyze the Gaussian formulation in the large system limit. Then, it suffices to analyze the following formulation
\begin{align}\label{gform_pf}
\min_{\vw\in\mathbb{R}^k}& \frac{1}{2n\ell} \sum_{i=1}^{n} \sum_{j=1}^{\ell}  \big(y_i-  \widetilde{\mu}_1 \vw^\top  \mF^\top \va_i-\widehat{\mu}_1 \vw^\top  \mF^\top \vz_{ij} \nonumber\\
&-\mu_0 \vw^\top \vec{1}_k  -\mu_2 \vw^\top \vb_i- \mu_3 \vw^\top \vp_{ij} \big)^2 + \tfrac{\lambda}{2} \norm{\vw}^2.
\end{align}
Note that the formulation in \eqref{gform_pf} is strongly convex with a strong convexity parameter equals to $\lambda$. This means that it has a unique optimal solution. Note that the multivariate CGMT framework assumes that the feasibility sets of the multivariate PO formulation in \eqref{MPO} are compact. The following lemma shows that this assumption is satisfied by our formulation.
\begin{lemma}[Primal Compactness]\label{lem_comw}
Assume that $\widehat{\vw}_p$ is the unique optimal solution of the formulation given in \eqref{gform_pf}. Then, there exist two positive constants $C_w>0$ and $C_\vartheta>0$ such that
\begin{align}
\mathbb{P}\big( \norm{\widehat{\vw}_p} \leq C_w \big) \overset{p\to\infty}{\longrightarrow} 1,~\mathbb{P}\big( \abs{ \vec{1}_k^\top \widehat{\vw}_p} \leq C_\vartheta \big) \overset{p\to\infty}{\longrightarrow} 1,
\end{align}
where the second asymptotic result is valid only when $\mu_0 \neq 0$.
\end{lemma}
Given that the loss function in \eqref{gform_pf} is proper and strongly convex, one can use the results in \cite[Lemma 1]{dhifallah2020} to prove Lemma \ref{lem_comw}. This asymptotic result follows using Assumptions \ref{itm:fun_fwf}, \ref{itm:act_fun}, and \ref{itm:ass_F} and  \cite[Theorem 2.1]{eigen_cons}. Combining this result with the theoretical result stated in \cite[Proposition 1]{dhifallah2020}, the Gaussian formulation is asymptotically equivalent to the following formulation
\begin{equation}
\begin{aligned}\label{ana_fm1}
\min_{\substack{c_w \leq \norm{\vw}\leq C_w \\ \abs{\vartheta} \leq C_\vartheta }}& \frac{1}{2n\ell}  \sum_{j=1}^{\ell}  \norm{ \vy- [\widetilde{\mu}_1 \mA+\widehat{\mu}_1 \mZ_{j}] \mF \vw -\mu_0 \vartheta \vec{1}_n   -\mu_2  \mB \vw - \mu_3  \mP_{j} \vw  }^2 + \tfrac{\lambda}{2} \norm{\vw}^2,
\end{aligned}
\end{equation}
where the matrices $\mA \in\mathbb{R}^{n\times p}$, $\mZ_j \in\mathbb{R}^{n\times p}$, $\mB \in\mathbb{R}^{n\times k}$ and $\mP_j \in\mathbb{R}^{n\times k}$ are formed by the concatenation of the vectors $\va_i$, $\vz_{ij}$, $\vb_i$ and $\vp_{ij}$, respectively. Here, the label vector $\vy \in\mathbb{R}^n$ is formed by the labels $\lbrace y_i \rbrace_{i=1}^{n}$ and can be expressed as follows $\vy=\varphi(\mA \vxi)$. Based on the analysis in \cite[Proposition 2]{dhifallah2020}, it suffices to precisely analyze the problem in \eqref{ana_fm1} for fixed feasible $\vartheta$. Then, minimize its asymptotic limit over $\vartheta$ to infer the asymptotic properties of \eqref{ana_fm1}. Next, we start by analyzing the formulation in \eqref{ana_fm1} for fixed feasible $\vartheta$.
To express \eqref{ana_fm1} in the form of the multivariate PO introduced in \eqref{MPO}, we introduce additional dual optimization variables. Specifically, the formulation in \eqref{ana_fm1} can be equivalently formulated as follows
\begin{equation}
\begin{aligned}\label{ana_fm2}
\min_{\substack{\norm{\vw}\leq C_w }} \max_{\vu\in\mathbb{R}^{\ell n}}& -\frac{\norm{\vu}^2}{2n\ell}+\frac{1}{n\ell}  \sum_{j=1}^{\ell}  \vu_j^\top \Big( \mu_0 \vartheta \vec{1}_n + \mu_2  \mB \vw -\vy \\
&+ [\widetilde{\mu}_1 \mA+\widehat{\mu}_1 \mZ_{j}] \mF \vw + \mu_3  \mP_{j} \vw  \Big) + \tfrac{\lambda}{2} \norm{\vw}^2,
\end{aligned}
\end{equation}
where the dual optimization vector $\vu\in\mathbb{R}^{\ell n}$ can be decomposed as follows $\vu^\top=[\vu_1^\top,\dots,\vu_\ell^\top]^\top$, where $\vu_i\in\mathbb{R}^n$, for any $1 \leq i \leq \ell$. Note that the optimization problem given in \eqref{ana_fm2} has a unique optimal solution. The multivariate CGMT also assumes that the feasibility set of the maximization problem in \eqref{ana_fm2} is compact. The following lemma shows that this assumption is also satisfied by our formulation.
\begin{lemma}[Dual Compactness]\label{lem_comu}
Assume that $\widehat{\vu}_p$ is the unique optimal solution of the formulation given in \eqref{ana_fm2}. Then, there exists a positive constants $C_u>0$ such that
\begin{align}
\mathbb{P}\big( \norm{\widehat{\vu}_p}/\sqrt{n} \leq C_u \big) \overset{p\to\infty}{\longrightarrow} 1.
\end{align}
\end{lemma}
This result can also be proved using similar steps as in \cite[Lemma 2]{dhifallah2020}. Specifically, we can use the result in \cite[Proposition 11.3]{var_ana} to show the compactness of the optimal dual vector $\widehat{\vu}_p$. The results in Lemma \ref{lem_comw} and Lemma \ref{lem_comu} show that the Gaussian formulation is asymptotically equivalent to the following formulation
\begin{equation}
\begin{aligned}\label{ana_fm2p}
\min_{\substack{\norm{\vw}\leq C_w }} \max_{ \frac{\norm{\vu}}{\sqrt{n}} \leq C_u }& -\frac{\norm{\vu}^2}{2n\ell}+\frac{1}{n\ell}  \sum_{j=1}^{\ell}  \vu_j^\top \big( \mu_0 \vartheta \vec{1}_n + \mu_2  \mB \vw -\vy \\
&+ [\widetilde{\mu}_1 \mA+\widehat{\mu}_1 \mZ_{j}] \mF \vw + \mu_3  \mP_{j} \vw  \big) + \tfrac{\lambda}{2} \norm{\vw}^2.
\end{aligned}
\end{equation}
Next, we focus on precisely analyzing the formulation in \eqref{ana_fm2p}. Now, note that the label vector $\vy$ depend on the Gaussian matrix $\mA$. Then, we decompose $\mA$ as follows
\begin{align}
\mA=\mA \mP_{\vxi}+\mA \mP^{\perp}_{\vxi}=\mA {\vxi} {\vxi}^\top +\mA \mP^{\perp}_{\vxi},
\end{align}
where $\mP_{\vxi}\in\mathbb{R}^{p\times p}$ denotes the projection matrix onto the space spanned by the vector $\vxi\in\mathbb{R}^p$ and $\mP^{\perp}_{\vxi}=\mI_p-{\vxi} {\vxi}^\top$ denotes the projection matrix onto the orthogonal complement of the space spanned by the vector $\vxi$. Note that the random matrix $\mA {\vxi} {\vxi}^\top$ is independent of the random matrix $\mA \mP^{\perp}_{\vxi}$. Then, we can express $\mA$ as follows without changing its statistics
\begin{align}
\mA=\vs {\vxi}^\top + \mA \mP^{\perp}_{\vxi},
\end{align}
where $\vs \in\mathbb{R}^n$ has independent standard Gaussian components and the two random quantities $\vs$ and $\mA$ are independent. This shows that the optimization problem formulated in \eqref{ana_fm2p} is statistically equivalent to the following formulation
\begin{equation}
\begin{aligned}\label{ana_fm3}
\min_{\substack{\norm{\vw}\leq C_w }} \max_{ \frac{\norm{\vu}}{\sqrt{n}}\leq C_u }& -\frac{\norm{\vu}^2}{2n\ell}+\frac{1}{n\ell}  \sum_{j=1}^{\ell}  \vu_j^\top \big( -\vy + \widetilde{\mu}_1 \vs {\vxi}^\top \mF \vw  \\
&+ \mu_0 \vartheta \vec{1}_n + \mG \mSigma^{\frac{1}{2}} \vw + \mT_j \mGm^{\frac{1}{2}} \vw \big) + \tfrac{\lambda}{2} \norm{\vw}^2,
\end{aligned}
\end{equation}
where $\mG$ and $\lbrace \mT_j \rbrace_{j=1}^\ell$ are independent matrices with independent and identically distributed standard Gaussian components. Here, $\mG\in\mathbb{R}^{n\times k}$, $\mT_j\in\mathbb{R}^{n\times k}$. Moreover, $\mSigma\in\mathbb{R}^{k\times k}$ and $\mGm\in\mathbb{R}^{k\times k}$ are positive definite matrices and defined as follows
\begin{align}\label{msG}
\begin{cases}
\mSigma=\widetilde{\mu}_1^2 \mF^\top \mP^{\perp}_{\vxi} \mF + \mu_2^2 \mI_k \\
\mGm=\widehat{\mu}_1^2 \mF^\top \mF + \mu_3^2 \mI_k.
\end{cases}
\end{align}
The above results show that it suffices to precisely analyze the formulation given in \eqref{ana_fm3}. Moreover, note that \eqref{ana_fm3} can be equivalently formulated as follows
\begin{equation}
\begin{aligned}\label{ana_fm4}
\min_{\substack{\norm{\vw}\leq C_w }} \max_{ \frac{\norm{\vu}}{\sqrt{n}}\leq C_u }&~\frac{\mu_0 \vartheta}{n\ell}  \vu^\top \vec{1}_{\ell n}-\frac{\norm{\vu}^2}{2n\ell}-\frac{\vu^\top \widehat{\vy}}{n\ell} +\frac{\widetilde{\mu}_1 \bar{\vxi}^\top \mF \vw }{\ell n} \vu^\top \widehat{\vs}  \\
&+ \frac{1}{n\ell} \vu^\top \mM^\top \mG \mSigma^{\frac{1}{2}} \vw + \frac{1}{n\ell} \vu^\top \mT \mGm^{\frac{1}{2}} \vw + \tfrac{\lambda}{2} \norm{\vw}^2,
\end{aligned}
\end{equation}
where $\vec{1}_{\ell n}$ denotes the all one vector of size $\ell n$, $\widehat{\vy}\in\mathbb{R}^{\ell n}$ and $\widehat{\vs}\in\mathbb{R}^{\ell n}$ are formed after performing an $\ell$ times concatenation of the vectors $\vy$ and $\vs$, respectively. Here, the matrices $\mT\in\mathbb{R}^{n \ell \times k}$ and $\mM\in\mathbb{R}^{n \times \ell n}$ are given as follows
\begin{align}
\mT^\top=[\mT_1^\top,\mT_2^\top,\dots,\mT_\ell^\top]^\top,~\mM=[\mI_n,\mI_n,\dots,\mI_n]. \nonumber
\end{align}
We can notice that the optimization problem formulated in \eqref{ana_fm4} is in the form of the multivariate PO problem given in \eqref{MPO}. Therefore, applying the multivariate CGMT, the corresponding multivariate AO problem can be expressed as follows
\begin{equation}
\begin{aligned}\label{ana_fm5}
\min_{\substack{\norm{\vw}\leq C_w }} \max_{ \frac{\norm{\vu}}{\sqrt{n}}\leq C_u }& -\frac{\norm{\vu}^2}{2n\ell} +\frac{\widetilde{\mu}_1 \bar{\vxi}^\top \mF \vw }{\ell n} \vu^\top \widehat{\vs} +\frac{\mu_0 \vartheta}{n\ell}  \vu^\top \vec{1}_{\ell n} \\
& -\frac{\vu^\top \widehat{\vy}}{n\ell} + \frac{1}{n\ell} \norm{\mM \vu} \vg_1^\top \mSigma^{\frac{1}{2}} \vw + \frac{1}{\ell n} \norm{\vu} \vg_2^\top \mGm^{\frac{1}{2}} \vw \\
& + \frac{\norm{\mSigma^{\frac{1}{2}} \vw}}{\ell n}  \vh_1^\top \mM \vu+ \frac{\norm{\mGm^{\frac{1}{2}} \vw}}{\ell n}  \vh_2^\top \vu
+ \frac{\lambda}{2} \norm{\vw}^2,
\end{aligned}
\end{equation}
where the vectors $\vg_1\in\mathbb{R}^k$, $\vg_2\in\mathbb{R}^k$, $\vh_1\in\mathbb{R}^n$ and $\vh_2\in\mathbb{R}^{\ell n}$ are independent standard Gaussian random vectors. First, observe that the convexity assumption in Theorem \ref{mcgmt} is satisfied by our multivariate PO formulation in \eqref{ana_fm4}. Furthermore, note that the compactness assumptions in the multivariate CGMT framework are also satisfied by our primary problem in \eqref{ana_fm4}. Then, following the multivariate CGMT framework, we focus on analyzing the multivariate AO formulation introduced in \eqref{ana_fm5}. Specifically, the objective is to simplify the multivariate AO problem and study its asymptotic properties.
\subsubsection{Simplifying the Multivariate Auxiliary Formulation}\label{MAO_simp}
In this part, our objective is to simplify the multivariate AO problem given in \eqref{ana_fm5}. Specifically, the main objective is to express the formulation in \eqref{ana_fm5} in terms of scalar optimization variables. First, observe that the singular value decomposition (SVD) of the matrix $\mM$ can be expressed as $\mM=\mU \mS \mV^\top$, where $\mU\in\mathbb{R}^{n\times n}$ and $\mV\in\mathbb{R}^{\ell n\times \ell n}$ are two orthogonal matrices and $\mS \in \mathbb{R}^{n \times \ell n}$ is given by $\mS=[\sqrt{\ell} \mI_n~~\vec{0}_{n\times(\ell-1)n}]$. Therefore, the optimization problem expressed in \eqref{ana_fm5} can be formulated as follows
\begin{equation}
\begin{aligned}\label{ana_fm6}
\min_{\substack{\norm{\vw}\leq C_w }} \max_{ \frac{\norm{\vu}}{\sqrt{n}}\leq C_u }&~\frac{\widetilde{\mu}_1 \bar{\vxi}^\top \mF \vw }{\ell n} \vu^\top \mV^\top \widehat{\vs} -\frac{\norm{\vu}^2}{2n\ell} +\frac{\mu_0 \vartheta}{n\ell}  \vu^\top \mV^\top \vec{1}_{\ell n} \\
& -\frac{\vu^\top \mV^\top \widehat{\vy}}{n\ell} + \frac{\sqrt{\ell}}{n\ell} \norm{\vu_1} \vg_1^\top \mSigma^{\frac{1}{2}} \vw + \frac{1}{\ell n} \norm{\vu} \vg_2^\top \mGm^{\frac{1}{2}} \vw \\
& + \frac{\sqrt{\ell}}{\ell n} \norm{\mSigma^{\frac{1}{2}} \vw} \vh_1^\top \vu_1+ \frac{\norm{\mGm^{\frac{1}{2}} \vw}}{\ell n}  \vh_2^\top \vu+ \tfrac{\lambda}{2} \norm{\vw}^2,
\end{aligned}
\end{equation}
where we perform the change of variable $\vu_{\text{new}}=\mV^\top \vu$, we decompose the new vector as $\vu_{\text{new}}^\top=[\vu_1^\top,\dots,\vu_\ell^\top]$ and we replace $\vu_{\text{new}}$ by $\vu$. Now, we denote by $t_1$ and $t_2$ the norms of the independent vectors $\vu_1$ and $\vu_{-1}^\top=[\vu_2^\top,\dots,\vu_\ell^\top]$, i.e. $t_1=\norm{\vu_1}$ and $t_2=\norm{\vu_{-1}}$. Additionally, we decompose the orthogonal matrix $\mV$ as follows $\mV=[\mV_1~~\mV_2]$, where $\mV_1\in\mathbb{R}^{\ell n \times n}$ and $\mV_2\in\mathbb{R}^{\ell n \times (\ell-1) n}$. Define the vector $\vv\in\mathbb{R}^k$ as $\vv=\mF^\top {\vxi}$ and the scalar $q$ as $q=\bar{\vv}^\top \vw$, where $\bar{\vv}$ is defined as follows $\bar{\vv}=\vv/\norm{\vv}$. Also, define the scalar $T_{p,1}$ as $T_{p,1}=\norm{\vv}$. 

Now, we are ready to further simplify the multivariate AO formulation. The first step is to fix $t_1$ and $t_2$ and solve the formulation in \eqref{ana_fm6} over the direction of the independent vectors $\vu_1$ and $\vu_{-1}$. Specifically, based on the result in Lemma \ref{lem_comu}, the formulation given in \eqref{ana_fm6} can be simplified as follows
\begin{equation}\label{ana_fm6_nxt}
\begin{aligned}
&\min_{\substack{\norm{\vw}\leq C_w }} \max_{\substack{0 \leq {t_1}/{\sqrt{n}} \leq C_{t_1}\\ 0 \leq {t_2}/{\sqrt{n}} \leq C_{t_2}}}~ \frac{\sqrt{\ell} t_1}{n\ell} \vg_1^\top \mSigma^{\frac{1}{2}} \vw + \frac{\sqrt{t_1^2+t_2^2}}{\ell n}  \vg_2^\top \mGm^{\frac{1}{2}} \vw -\frac{t_1^2+t_2^2}{2n\ell}+ \tfrac{\lambda}{2} \norm{\vw}^2 \\
&+\frac{t_1}{\ell n} \norm{ \sqrt{\ell} \norm{\mSigma^{\frac{1}{2}} \vw} \vh_1 + \norm{\mGm^{\frac{1}{2}} \vw} \widehat{\vh}_2 - \mV_1^\top \widehat{\vy} + \mu_0 \vartheta \mV_1^\top \vec{1}_{\ell n} + \widetilde{\mu}_1 T_{p,1} q  \mV_1^\top \widehat{\vs}  } \\
& +\frac{t_2}{\ell n} \norm{ \norm{\mGm^{\frac{1}{2}} \vw} \widetilde{\vh}_2 - \mV_2^\top \widehat{\vy} + \mu_0 \vartheta \mV_2^\top \vec{1}_{\ell n} + \widetilde{\mu}_1 T_{p,1} q   \mV_2^\top \widehat{\vs}  },
\end{aligned}
\end{equation}
where we decompose the Gaussian vector $\vh_2$ as $\vh_2=[\widehat{\vh}_2^\top~~\widetilde{\vh}_2^\top]$, where $\widehat{\vh}_2\in\mathbb{R}^{n}$ and $\widetilde{\vh}_2\in\mathbb{R}^{(\ell-1) n}$. Here, $C_{t_1}$ and $C_{t_2}$ are sufficiently large positive constants that ensure the asymptotic result in Lemma \ref{lem_comu}. Note that it remains to solve over the primal vector $\vw$ to obtain a scalar formulation of the multivariate AO problem. We continue our analysis by defining the following optimization problem
\begin{equation}
\label{ana_fm6_nxt2}
\begin{aligned}
&\min_{\substack{\norm{\vw}\leq C_w }} \max_{\substack{0 \leq {t_1}/{\sqrt{n}} \leq C_{t_1}\\ 0 \leq {t_2}/{\sqrt{n}} \leq C_{t_2}}}~ \frac{\sqrt{\ell} t_1}{n\ell} \vg_1^\top \mSigma^{\frac{1}{2}} \vw + \frac{\sqrt{t_1^2+t_2^2}}{\ell n}  \vg_2^\top \mGm^{\frac{1}{2}} \vw -\frac{t_1^2+t_2^2}{2n\ell}+ \tfrac{\lambda}{2} \norm{\vw}^2 \\
&+\frac{t_1}{\ell n} \sqrt{ {\ell} \norm{\mSigma^{\frac{1}{2}} \vw}^2 \norm{ \vh_1}^2+\norm{\mGm^{\frac{1}{2}} \vw}^2 \norm{\widehat{\vh}_2}^2+\norm{ \mu_0 \vartheta \mV_1^\top \vec{1}_{\ell n} + \widetilde{\mu}_1 T_{p,1} q  \mV_1^\top \widehat{\vs} - \mV_1^\top \widehat{\vy} }^2  }  \\
& +\frac{t_2}{\ell n}  \sqrt{ \norm{\mGm^{\frac{1}{2}} \vw}^2 \norm{\widetilde{\vh}_2}^2+\norm{ \mu_0 \vartheta \mV_2^\top \vec{1}_{\ell n}  - \mV_2^\top \widehat{\vy} + \widetilde{\mu}_1 T_{p,1} q   \mV_2^\top \widehat{\vs}  }^2 }.
\end{aligned}
\end{equation}
Note that the difference between the cost functions of the formulations in \eqref{ana_fm6_nxt} and \eqref{ana_fm6_nxt2} are terms that converge in probability to zero. Before showing the asymptotic equivalence between the formulations in \eqref{ana_fm6_nxt} and \eqref{ana_fm6_nxt2}, we provide important convexity properties of the optimization problem in \eqref{ana_fm6_nxt2} as given in the following lemma. 
\begin{lemma}[Strong--convexity of \eqref{ana_fm6_nxt2}]\label{convty_vn2}
Define $\widehat{f}_{p,2}$ as the cost function of the problem in \eqref{ana_fm6_nxt2}. Then, $\widehat{f}_{p,2}$ is strongly convex in the vector $\vw$ where $\lambda$ is a strong convexity parameter. Moreover, it is strongly concave in the variables $t_1$ and $t_2$ in the feasibility sets where $-1$ is a strong concavity parameter.
\end{lemma}
\begin{proof}
The strong convexity can be proved by observing that the cost function of \eqref{ana_fm6_nxt2} is a positive sum of convex and strongly convex functions in terms of $\vw$ for fixed feasible $t_1$ and $t_2$. Moreover, note that the term ${\sqrt{t_1^2+t_2^2}}\vg_2^\top \mGm^{\frac{1}{2}} \vw$ can be replaced with $t_1\vg_{21}^\top \mGm^{\frac{1}{2}} \vw+t_2\vg_{22}^\top \mGm^{\frac{1}{2}} \vw$ without changing the statistics of our formulation, where $\vg_{21}$ and $\vg_{22}$ are two independent Gaussian vectors. Then, one can see that the cost function of \eqref{ana_fm6_nxt2} is strongly concave in the variables $t_1$ and $t_2$ where $-1$ is a strong concavity parameter.
\end{proof}
Lemma \ref{convty_vn2} provides important convexity properties of the optimization problem formulated in \eqref{ana_fm6_nxt2}. These properties are essential to prove the equivalence between \eqref{ana_fm6_nxt} and \eqref{ana_fm6_nxt2} as stated in the following lemma.
\begin{lemma}[High-dimensional Equivalence I]
\label{lem_eqv1}
Define $\widehat{\mathcal{S}}^{\star}_{p,1}$ and $\widehat{\mathcal{S}}^{\star}_{p,2}$ as the sets of optimal solutions of the minimization problems in \eqref{ana_fm6_nxt} and \eqref{ana_fm6_nxt2}, respectively. Moreover, let $\widehat{O}^{\star}_{p,1}$ and $\widehat{O}^{\star}_{p,2}$ be the optimal objective values of the optimization problems in \eqref{ana_fm6_nxt} and \eqref{ana_fm6_nxt2}, respectively. Then, the following convergence in probability holds
\begin{equation}
\abs{ \widehat{O}^{\star}_{p,1}-\widehat{O}^{\star}_{p,2} } \overset{p\to +\infty}{\longrightarrow} 0,~\text{and}~\mathbb{D}( \widehat{\mathcal{S}}^{\star}_{p,1},\widehat{\mathcal{S}}^{\star}_{p,2} )  \overset{p\to +\infty}{\longrightarrow} 0,
\end{equation}
where $\mathbb{D}( \mathcal{A},\mathcal{B} )$ denotes the deviation between the sets $\mathcal{A}$ and $\mathcal{B}$ and is defined as $\mathbb{D}( \mathcal{A},\mathcal{B} )=\sup_{\vx_1\in\mathcal{A}} \inf_{\vx_2\in\mathcal{B}} \norm{\vx_1-\vx_2}_2$.
\end{lemma}
The detailed proof of Lemma \ref{lem_eqv1} is deferred to Appendix \ref{lem_eqv1_pf}. Lemma \ref{lem_eqv1} particularly shows that the optimization problems in \eqref{ana_fm6_nxt} and \eqref{ana_fm6_nxt2} are asymptotically equivalent. Then, it suffices to precisely analyze the formulation in \eqref{ana_fm6_nxt2}. To solve over the primal vector $\vw$, we introduce two independent scalar optimization variables $\tau_1$ and $\tau_2$ where they both solve optimization problems of the following form
\begin{align}\label{inf_idty}
\sqrt{x}=\inf_{\tau>0}~\frac{\tau}{2}+\frac{x}{2\tau},~\text{for any}~x\geq 0.
\end{align}
Here, note that the optimal solution of the problem in \eqref{inf_idty} can be expressed as $\tau^\star=\sqrt{x}$. Next, we use the identity in \eqref{inf_idty} to transform the non-smooth square roots in the cost function of the formulation given in \eqref{ana_fm6_nxt2} to smooth terms. This is an essential step to solve over the primal vector $\vw$. Specifically, based on  the result in Lemma \ref{lem_comw}, our multivariate AO formulation given in \eqref{ana_fm6_nxt2} can be expressed as follows
\begin{equation}
\begin{aligned}\label{ana_fm7}
&\min_{\substack{\norm{\vw}\leq C_w }} \max_{\substack{0 \leq {t_1} \leq C_{t_1}\\ 0 \leq {t_2} \leq C_{t_2}}} \inf_{\substack{\tau_1 > 0\\ \tau_2 > 0}}~ \frac{\sqrt{\ell} t_1}{\sqrt{n}\ell} \vg_1^\top \mSigma^{\frac{1}{2}} \vw + \frac{\sqrt{t_1^2+t_2^2}}{\ell \sqrt{n}}  \vg_2^\top \mGm^{\frac{1}{2}} \vw +\frac{\tau_1 t_1}{2 \ell}+ \frac{\tau_2 t_2}{2 \ell}-\frac{t_1^2+t_2^2}{2\ell}+ \tfrac{\lambda}{2} \norm{\vw}^2 \\
& + \frac{t_1 \norm{\vh_1}^2}{2 \tau_1 n} \norm{\mSigma^{\frac{1}{2}} \vw}^2 + \frac{\zeta_{p,t,\tau}}{2 n}   \norm{\mGm^{\frac{1}{2}} \vw}^2 +\frac{t_1}{2 \tau_1 \ell n} \norm{ \mu_0 \vartheta \mV_1^\top \vec{1}_{\ell n} - \mV_1^\top \widehat{\vy} + \widetilde{\mu}_1 T_{p,1} q  \mV_1^\top \widehat{\vs}  }^2  \\
&+\frac{t_2}{2 \tau_2 \ell n} \norm{\mu_0 \vartheta \mV_2^\top \vec{1}_{\ell n}  - \mV_2^\top \widehat{\vy} + \widetilde{\mu}_1 T_{p,1} q   \mV_2^\top \widehat{\vs}  }^2,
\end{aligned} 
\end{equation}
where $\zeta_{p,t,\tau}={t_1  \norm{\widehat{\vh}_2}^2}/{(\ell \tau_1)} + {t_2 \norm{\widetilde{\vh}_2}^2}/{(\ell \tau_2)} $. Here, we also perform the change of variable $t_{1,\text{new}}=t_1/\sqrt{n}$ and $t_{2,\text{new}}=t_2/\sqrt{n}$, then, replace $t_{1,\text{new}}$ and $t_{2,\text{new}}$ by $t_{1}$ and $t_{2}$. Note that the feasibility sets of the optimization variables $\tau_1$ and $\tau_2$ are open unbounded sets. To simplify the analysis, we show that the feasibility sets of the variables $\tau_1$ and $\tau_2$ can be restricted to compact sets with probability going to $1$ as $p$ grows to $+\infty$ as stated in the following lemma.
\begin{lemma}[Additional Compactness]
\label{lembtau}
There exist positive constants independent of $p$, $c_{\tau_1}>0$, $C_{\tau_1}>0$, $c_{\tau_2}>0$ and $C_{\tau_2}>0$, such that the following convergence in probability holds
\begin{align}
\mathbb{P}(c_{\tau_1} \leq \widehat{\tau}_1 \leq C_{\tau_1}) \xrightarrow{p \to \infty} 1,~\mathbb{P}(c_{\tau_2} \leq \widehat{\tau}_2 \leq C_{\tau_2}) \xrightarrow{p \to \infty} 1,
\end{align}
where $\widehat{\tau}_1$ and $\widehat{\tau}_2$ are the optimal solutions of the formulation in \eqref{ana_fm7}. 
\end{lemma}
The detailed proof of Lemma \ref{lembtau} is provided in Appendix \ref{lembtau_pf}. Based on Lemmas \ref{convty_vn2} and \ref{lembtau}, the optimization problem given in \eqref{ana_fm7} is asymptotically equivalent to the following problem 
\begin{equation}
\begin{aligned}\label{ana_fm7p}
&\min_{\substack{\norm{\vw}\leq C_w }} \max_{\substack{0 \leq {t_1} \leq C_{t_1}\\ 0 \leq {t_2} \leq C_{t_2}}} \min_{\substack{c_{\tau_1} \leq {\tau_1} \leq C_{\tau_1}\\ c_{\tau_2} \leq {\tau_2} \leq C_{\tau_2}}}~ \frac{\sqrt{\ell} t_1}{\sqrt{n}\ell} \vg_1^\top \mSigma^{\frac{1}{2}} \vw + \frac{\sqrt{t_1^2+t_2^2}}{\ell \sqrt{n}}  \vg_2^\top \mGm^{\frac{1}{2}} \vw +\frac{\tau_1 t_1}{2 \ell}+ \frac{\tau_2 t_2}{2 \ell}-\frac{t_1^2+t_2^2}{2\ell}+ \tfrac{\lambda}{2} \norm{\vw}^2 \\
& + \frac{t_1 \norm{\vh_1}^2}{2 \tau_1 n} \norm{\mSigma^{\frac{1}{2}} \vw}^2 + \frac{\zeta_{p,t,\tau}}{2 n}   \norm{\mGm^{\frac{1}{2}} \vw}^2 +\frac{t_1}{2 \tau_1 \ell n} \norm{ \mu_0 \vartheta \mV_1^\top \vec{1}_{\ell n} - \mV_1^\top \widehat{\vy} + \widetilde{\mu}_1 T_{p,1} q  \mV_1^\top \widehat{\vs}  }^2  \\
&+\frac{t_2}{2 \tau_2 \ell n} \norm{\mu_0 \vartheta \mV_2^\top \vec{1}_{\ell n}  - \mV_2^\top \widehat{\vy} + \widetilde{\mu}_1 T_{p,1} q   \mV_2^\top \widehat{\vs}  }^2.
\end{aligned} 
\end{equation}
Now, we are ready to simplify the formulation in \eqref{ana_fm7p} over the optimization vector $\vw$. We start our analysis by decomposing the optimization vector $\vw\in\mathbb{R}^k$ as follows
\begin{align}\label{decp_ao}
\vw= q \bar{\vv} + \mB^{\perp}_{\vv} \vr,
\end{align}
where $\vr\in\mathbb{R}^{k-1}$ and $\mB^{\perp}_{\vv}\in\mathbb{R}^{k\times(k-1)}$ is formed by an orthonormal basis orthogonal to the vector $\vv\in\mathbb{R}^k$. Based on the result in Lemma \ref{lem_comw}, one can equivalently formulate the problem in \eqref{ana_fm7p} as follows
\begin{equation}
\begin{aligned}\label{ana_fm8}
&\min_{\substack{\abs{q} \leq C_q \\ \norm{\vr} \leq C_r }} \max_{\substack{0 \leq {t_1} \leq C_{t_1}\\ 0 \leq {t_2} \leq C_{t_2}}} \min_{\substack{c_{\tau_1} \leq {\tau_1} \leq C_{\tau_1}\\ c_{\tau_2} \leq {\tau_2} \leq C_{\tau_2}}} ~ \frac{\sqrt{\ell} t_1}{\sqrt{n}\ell} \vg_1^\top \mSigma^{\frac{1}{2}} \mB^{\perp}_{\vv} \vr + \frac{\sqrt{t_1^2+t_2^2}}{\ell \sqrt{n}}  \vg_2^\top \mGm^{\frac{1}{2}} \mB^{\perp}_{\vv} \vr +\frac{\tau_1 t_1}{2 \ell}+ \frac{\tau_2 t_2}{2 \ell}-\frac{t_1^2+t_2^2}{2\ell} \\
& + \frac{t_1 \norm{\vh_1}^2}{2 \tau_1 n} \norm{\mSigma^{\frac{1}{2}} \mB^{\perp}_{\vv} \vr + q \mSigma^{\frac{1}{2}} \bar{\vv}}^2 + \frac{\zeta_{p,t,\tau}}{2 n}   \norm{\mGm^{\frac{1}{2}} \mB^{\perp}_{\vv} \vr + q \mGm^{\frac{1}{2}} \bar{\vv}}^2 + \tfrac{\lambda}{2} (q^2+\norm{\vr}^2)  \\
&+\frac{t_1}{2 \tau_1 \ell n}\norm{ \mu_0 \vartheta \mV_1^\top \vec{1}_{\ell n} - \mV_1^\top \widehat{\vy} + \widetilde{\mu}_1 T_{p,1} q   \mV_1^\top \widehat{\vs}  }^2 +\frac{t_2}{2 \tau_2 \ell n}\norm{\mu_0 \vartheta \mV_2^\top \vec{1}_{\ell n}  - \mV_2^\top \widehat{\vy} + \widetilde{\mu}_1 T_{p,1} q   \mV_2^\top \widehat{\vs}  }^2,
\end{aligned} 
\end{equation}
where $C_q>0$ and $C_r>0$ are two positive constants selected to satisfy the asymptotic result in Lemma \ref{lem_comw}. Here, we also drop terms that converge in probability to zero. One way to justify this step is using similar analysis as in Lemma \ref{lem_eqv1}. Note that the convexity results in Lemma \ref{convty_vn2} are still satisfied by the formulation in \eqref{ana_fm8}. Specifically, the cost function in \eqref{ana_fm8} is jointly strongly convex in the minimization variables and jointly strongly concave in the maximization variables.

Now, it remains to solve over the optimization vector $\vr\in\mathbb{R}^{k-1}$. To solve over $\vr$, we interchange the minimization over $\vr$ and the maximization over $t_1$ and $t_2$. This step is justified using the result in \cite{sion1958}. The cost function of the optimization vector $\vr$ can then be expressed as follows
\begin{equation}
\begin{aligned}\label{fung}
g(\vr)&=\frac{\sqrt{\ell} t_1}{\sqrt{n}\ell} \vg_1^\top \mSigma^{\frac{1}{2}} \mB^{\perp}_{\vv} \vr + \frac{\sqrt{t_1^2+t_2^2}}{\ell \sqrt{n}}  \vg_2^\top \mGm^{\frac{1}{2}} \mB^{\perp}_{\vv} \vr + \frac{t_1 \norm{\vh_1}^2}{2 \tau_1 n} \norm{\mSigma^{\frac{1}{2}} \mB^{\perp}_{\vv} \vr + q \mSigma^{\frac{1}{2}} \bar{\vv}}^2 + \tfrac{\lambda}{2} \norm{\vr}^2 \\
&+ \frac{\zeta_{p,t,\tau}}{2 n}   \norm{\mGm^{\frac{1}{2}} \mB^{\perp}_{\vv} \vr + q \mGm^{\frac{1}{2}} \bar{\vv}}^2, 
\end{aligned}
\end{equation}
where we ignore the terms independent of $\vr$. Note that the function $g(\cdot)$ is convex and smooth. Before solving the minimization problem of the function $g(\cdot)$, we define the matrix $\mG\in\mathbb{R}^{k\times k}$ and the vectors $\vf\in\mathbb{R}^{k-1}$ and $\vz\in\mathbb{R}^{k-1}$  as follows
\begin{align}
\begin{cases}
\mG=\frac{t_1 \norm{\vh_1}^2}{\tau_1 n} \mSigma + \frac{\zeta_{p,t,\tau}}{ n} \mGm\\
\vf=\frac{t_1 \norm{\vh_1}^2}{\tau_1 n} \bar{\mB}^{\perp}_{\vv} \mSigma \bar{\vv} + \frac{\zeta_{p,t,\tau}}{ n} \bar{\mB}^{\perp}_{\vv} \mGm \bar{\vv} \\
\vz=\frac{\sqrt{\ell} t_1}{\sqrt{n}\ell}  \bar{\mB}^{\perp}_{\vv}\mSigma^{\frac{1}{2}} \vg_1 + \frac{\sqrt{t_1^2+t_2^2}}{\ell \sqrt{n}}  \bar{\mB}^{\perp}_{\vv} \mGm^{\frac{1}{2}}  \vg_2,
\end{cases}
\end{align}
where $\bar{\mB}^{\perp}_{\vv}=({\mB}^{\perp}_{\vv})^\top$.
After computing the derivative of the function $g(\cdot)$ and setting it to zero, the optimal solution of the unconstrained version of minimizing the function $g(\cdot)$ can be expressed as follows
\begin{align}\label{ropt}
\widetilde{\vr}^\star=- \left[ \bar{\mB}^{\perp}_{\vv} \mG \mB^{\perp}_{\vv} + \lambda \mI_{k-1} \right]^{-1} \left[ q \vf + \vz \right].
\end{align}
Similar to the analysis in Lemmas \ref{lem_comw}, \ref{lem_comu} and \ref{lembtau}, one can show that the norm of the optimal vector $\widetilde{\vr}^\star$ is bounded. This means that $\widetilde{\vr}^\star$ is an optimal solution of the formulation in \eqref{ana_fm8}. Then, the optimal loss function can be expressed as follows
\begin{equation}
\begin{aligned}
g^\star&=-\frac{q^2}{2} \vf^\top \left[ \bar{\mB}^{\perp}_{\vv} \mG \mB^{\perp}_{\vv} + \lambda \mI_{k-1} \right]^{-1} \vf + \frac{q^2}{2} \bar{\vv}^\top \mG \bar{\vv} \\
&- \frac{1}{2} \vz^\top \left[ \bar{\mB}^{\perp}_{\vv} \mG \mB^{\perp}_{\vv} + \lambda \mI_{k-1} \right]^{-1} \vz.
\end{aligned}
\end{equation}
Based on the SVD decomposition of the matrix $\mM$, it can be checked that the last term in the multivariate AO formulation given in \eqref{ana_fm8} is zero. Then, the formulation given in \eqref{ana_fm8} can be expressed as follows
\begin{equation}
\begin{aligned}\label{ana_fm9}
\min_{\substack{\abs{q} \leq C_q }} \max_{\substack{0 \leq {t_1} \leq C_{t_1}\\ 0 \leq {t_2} \leq C_{t_2}}} \min_{\substack{c_{\tau_1} \leq {\tau_1} \leq C_{\tau_1}\\ c_{\tau_2} \leq {\tau_2} \leq C_{\tau_2}}}& \frac{\tau_1 t_1}{2 \ell}-\frac{t_1^2+t_2^2}{2\ell}  + \frac{\tau_2 t_2}{2 \ell}  + \frac{\lambda+V_{p,2}(\vt,\vtau)-V_{p,3}(\vt,\vtau)}{2} q^2  \\
&+\frac{t_1}{2 \tau_1 n} \norm{ \mu_0 \vartheta  \vec{1}_{n} -  {\vy} + \widetilde{\mu}_1 T_{p,1} q   {\vs}  }^2  - \frac{1}{2} V_{p,4}(\vt,\vtau),
\end{aligned} 
\end{equation}
where $\vt=[t_1,t_2]^\top$ and $\vtau=[\tau_1,\tau_2]^\top$. Here, $\vec{1}_{n}$ denotes the vector of all one with size $n$ and the functions $V_{p,2}(\cdot,\cdot)$, $V_{p,3}(\cdot,\cdot)$ and $V_{p,4}(\cdot,\cdot)$ depend on the optimization variables and are given by
\begin{align}\label{Vs}
\begin{cases}
V_{p,2}(\vt,\vtau)=\bar{\vv}^\top \mG \bar{\vv}\\
V_{p,3}(\vt,\vtau)=\vf^\top \left[ \bar{\mB}^{\perp}_{\vv} \mG \mB^{\perp}_{\vv} + \lambda \mI_{k-1} \right]^{-1} \vf\\
V_{p,4}(\vt,\vtau)=\vz^\top \left[ \bar{\mB}^{\perp}_{\vv} \mG \mB^{\perp}_{\vv} + \lambda \mI_{k-1} \right]^{-1} \vz.
\end{cases}
\end{align}
Note that we simplified the multivariate AO formulation given in \eqref{ana_fm5} to a scalar optimization problem as given in \eqref{ana_fm9}. Then, it remains to study the asymptotic properties of the scalar formulation in \eqref{ana_fm9}. We refer to this problem as the \textit{scalar formulation}.
\subsubsection{Asymptotic Analysis of the Scalar Formulation} \label{asy_ana_ao}
In this part, we study the asymptotic properties of the scalar formulation in \eqref{ana_fm9} corresponding to the multivariate AO problem.
Based on Assumption \ref{itm:ass_F} and the result in \cite[Proposition 3]{Debbah}, the random variable ${T}_{p,1}$ converges pointwisely in probability to the scalar $T_1$ defined as follows
\begin{align}
T_{p,1}=\sqrt{ {\vxi}^\top \mF \mF^\top {\vxi} }  \xrightarrow{~p \to \infty~} T_1= \sqrt{ \delta \mathbb{E}_\kappa[\kappa] },
\end{align}
where the expectations are over the probability distribution $\mathbb{P}_\kappa(\cdot)$ defined in Assumption \ref{itm:ass_F}. Furthermore, the random function $V_{p,2}(\cdot,\cdot)$ can be expressed as follows
\begin{equation}
\begin{aligned}
V_{p,2}(\vt,\vtau)&=\frac{1}{T_{p,1}^2} {\vxi}^\top \mF \Big( \frac{t_1 \norm{\vh_1}^2}{\tau_1 n} \mSigma + \frac{\zeta_{p,t,\tau}}{ n} \mGm \Big) \mF^\top  {\vxi} \\
&=\frac{1}{T_{p,1}^2} {\vxi}^\top \mF \Big[ \Big( \frac{t_1 \norm{\vh_1}^2 \widetilde{\mu}_1^2}{\tau_1 n} + \frac{\zeta_{p,t,\tau} \widehat{\mu}_1^2 }{n} \Big) \mF^\top \mF - \frac{t_1 \norm{\vh_1}^2 \widetilde{\mu}_1^2}{\tau_1 n}   \\
&\times \mF^\top {\vxi} {\vxi}^\top \mF + \Big( \frac{t_1 \norm{\vh_1}^2 {\mu}_2^2}{\tau_1 n} + \frac{\zeta_{p,t,\tau} {\mu}_3^2 }{n}  \Big) \mI_k  \Big] \mF^\top {\vxi},
\end{aligned}
\end{equation}
where $\vt=[t_1,t_2]^\top$ and $\vtau=[\tau_1,\tau_2]^\top$.
Then, using the theoretical results in \cite{Debbah} and based on Assumption \ref{itm:ass_F}, the random function $V_{p,2}(\cdot,\cdot)$ converges pointwisely in probability as follows
\begin{equation}
\begin{aligned}
V_{p,2}(\vt,\vtau) \xrightarrow{~p \to +\infty~} V_2(\vt,\vtau)&=\frac{\delta}{T_{1}^2} \Big( \frac{t_1  \widetilde{\mu}_1^2}{\tau_1} + \zeta_{t,\tau} \widehat{\mu}_1^2  \Big) \mathbb{E}_\kappa[\kappa^2] - \delta^2 \frac{t_1 \widetilde{\mu}_1^2}{\tau_1} \mathbb{E}_\kappa[\kappa]^2  \\
& + \Big( \frac{t_1 {\mu}_2^2}{\tau_1} + \zeta_{t,\tau} {\mu}_3^2   \Big) \delta \mathbb{E}_\kappa[\kappa],
\end{aligned}
\end{equation}
where $\zeta_{t,\tau}={t_1}/{(\ell \tau_1)} + {t_2 (\ell-1)}/{(\ell \tau_2)}$ and the expectations are over the probability distribution $\mathbb{P}_\kappa(\cdot)$ defined in Assumption \ref{itm:ass_F}. The theoretical results in \cite[Proposition 3]{Debbah} also show that the random function $V_{p,4}(\cdot,\cdot)$ satisfies the following asymptotic result
\begin{align}
V_{p,4}(\vt,\vtau)-\widehat{V}_{p,4}(\vt,\vtau) \xrightarrow{~p \to + \infty~} 0,
\end{align}
where the random function $\widehat{V}_{p,4}(\cdot,\cdot)$ is defined as follows
\begin{equation}
\begin{aligned}
\widehat{V}_{p,4}(\vt,\vtau)&= \frac{t_1^2}{{n}\ell} \text{Tr}\Big[ \widehat{\mSigma}^{\frac{1}{2}} \left[ \widehat{\mG} + \lambda \mI_{k-1} \right]^{-1}  \widehat{\mSigma}^{\frac{1}{2}} \Big] \\
&+\frac{{t_1^2+t_2^2}}{\ell^2 {n}} \text{Tr}\Big[ \mGm^{\frac{1}{2}}  \left[  \widehat{\mG} + \lambda \mI_{k-1} \right]^{-1}  \mGm^{\frac{1}{2}} \Big].
\end{aligned}
\end{equation}
Here, $\text{Tr}[.]$ represents the trace. Additionally, the matrix $\widehat{\mSigma}$ is given as $\widehat{\mSigma}=\widetilde{\mu}_1^2 \mF^\top \mF + \mu_2^2 \mI_k$ and the matrix $\widehat{\mG}$ has the following expression  $\widehat{\mG}={t_1 \norm{\vh_1}^2}/{(\tau_1 n)} \widehat{\mSigma} + {\zeta_{p,t,\tau}}/{n} \mGm$. Using again \cite[Proposition 3]{Debbah} and Assumption \ref{itm:ass_F}, we can also see that the random function $V_{p,4}(\cdot,\cdot)$ converges in probability to the function $V_{4}(\cdot,\cdot)$ defined as follows
\begin{align}
V_{4}(\vt,\vtau)&=\frac{\eta t_1^2}{\ell} \mathbb{E}_\kappa \Big[  \frac{\widetilde{\mu}_1^2 \kappa + \mu_2^2}{\frac{t_1}{\tau_1} ( \widetilde{\mu}_1^2 \kappa + \mu_2^2 )+\zeta_{t,\tau} ( \widehat{\mu}_1^2 \kappa + \mu_3^2 ) + \lambda}  \Big] \nonumber\\
&+\frac{\eta{(t_1^2+t_2^2)}}{\ell^2} \mathbb{E}_\kappa \Big[  \frac{\widehat{\mu}_1^2 \kappa + \mu_3^2}{\frac{t_1}{\tau_1} ( \widetilde{\mu}_1^2 \kappa + \mu_2^2 )+\zeta_{t,\tau} ( \widehat{\mu}_1^2 \kappa + \mu_3^2 ) + \lambda}  \Big].\nonumber
\end{align}
Now, it remains to study the asymptotic properties of the random function $V_{p,3}(\cdot,\cdot)$. Based on the block matrix inversion lemma, it can be checked that the random function $V_{p,3}(\cdot,\cdot)$ satisfies the following
\begin{align}
V_{p,3}(\vt,\vtau)=V_{p,2}(\vt,\vtau)+\lambda-\frac{1}{T_{p,2}(\vt,\vtau)}.
\end{align}
Here, the random function $T_{p,2}(\cdot,\cdot)$ is defined as follows
\begin{align}
T_{p,2}(\vt,\vtau)=\bar{\vv}^\top \left[ \mG + \lambda \mI_{k} \right]^{-1} \bar{\vv}.
\end{align}
Using the matrix inversion lemma, it can be checked that the random function $T_{p,2}(\cdot,\cdot)$ converges in probability to the function $T_{2,\lambda}(\cdot,\cdot)$ defined as follows
\begin{align}
T_{2,\lambda}(\vt,\vtau)= \frac{\delta}{T_1^2} \mathbb{E}\Big[ \frac{\kappa}{g_{\kappa,\lambda}(t,\tau)} \Big] /\Big(1-\frac{\widetilde{\mu}_1^2 t_1 \delta}{ \tau_1} \mathbb{E}\Big[ \frac{\kappa}{g_{\kappa,\lambda}(t,\tau)} \Big] \Big),\nonumber
\end{align}
where the function $g_{\kappa,\lambda}(\cdot,\cdot)$ is defined in Section \ref{pr_analysis}.
Additionally, using the weak law of large numbers (WLLN), we have the following convergence property
\begin{align}
&\frac{1}{n} \norm{ \mu_0 \vartheta  \vec{1}_{n} -  {\vy} + \widetilde{\mu}_1 T_{p,1} q   {\vs}  }^2 \xrightarrow{~p\to +\infty~} \mu_0^2 \vartheta^2 + \gamma_1 +  \widetilde{\mu}_1^2 T_{1}^2 q^2-2\widetilde{\mu}_1 T_{1} q \gamma_2 -2\mu_0 \vartheta \gamma_3.
\end{align}
Here, $\gamma_1$, $\gamma_2$ and $\gamma_3$ depend on the data distribution and are defined as $\gamma_1=\mathbb{E}[y^2]$, $\gamma_2=\mathbb{E}[y s]$, $\gamma_3=\mathbb{E}[y]$, where $y=\varphi( s )$, and $s$ is a standard Gaussian random variable.
This proves that the cost function of the following deterministic problem
\begin{equation}
\begin{aligned}
\label{scprob1_pf}
\min_{\substack{\abs{q} \leq C_q }} \max_{\substack{0 \leq {t_1} \leq C_{t_1}\\ 0 \leq {t_2} \leq C_{t_2}}} \min_{\substack{c_{\tau_1} \leq {\tau_1} \leq C_{\tau_1}\\ c_{\tau_2} \leq {\tau_2} \leq C_{\tau_2}}}& \frac{\tau_1 t_1+\tau_2 t_2}{2 \ell}  - \frac{t_1^2+t_2^2}{2 \ell} +\frac{q^2}{2 T_{2,\lambda}(\vt,\vtau)}  -\frac{\eta T_{3,\lambda}(\vt,\vtau)}{2}  \\
&+\frac{t_1}{2\tau_1} \Big( \gamma_1 -  2 \widetilde{\mu}_1 T_1 q \gamma_2 + \widetilde{\mu}_1^2 T_1^2 q^2+ \mu_0^2 \vartheta^2 -2\mu_0 \vartheta \gamma_3 \Big),
\end{aligned}
\end{equation}
is the converging limit of the cost function the scalar formulation in \eqref{ana_fm9}, where the function  $T_{3,\lambda}(\cdot,\cdot)$ is given by $V_4(\vt,\vtau)/\eta$. Before continuing our analysis, we summarize convexity properties of the cost function of \eqref{scprob1_pf} in the following lemma.
\begin{lemma}[Strong--convexity of \eqref{scprob1_pf}]\label{convty_detpp}
Define $\widehat{f}$ as the cost function of the problem in \eqref{scprob1_pf} defined in the feasibility set. Then, $\widehat{f}$ is jointly strongly convex in the variables $(q,\vartheta,\tau_1,\tau_2)$ for fixed feasible $(t_1,t_2)$. Moreover, it is jointly strongly concave in the variables $(t_1,t_2)$ for fixed feasible $(q,\vartheta,\tau_1,\tau_2)$.
\end{lemma}
This result can be proved by observing that the strong convexity parameters in Lemma \ref{convty_vn2} are independent of $p$ and that the operations performed after Lemma \ref{convty_vn2} preserve the convexity properties. Another property of the scalar formulation is that its set of optimal solutions concentrates around the set of optimal solutions of the formulation in \eqref{scprob1_pf} as summarized in the following lemma.
\begin{lemma}[Consistency of the Scalar Formulation]\label{consi_detp}
Define $\tau_{p,1}^\star$, $\tau_{p,2}^\star$, $t_{p,1}^\star$, $t_{p,2}^\star$ and $q_p^\star$ as the optimal solutions of the scalar formulation given in \eqref{ana_fm9}. Additionally, define $\tau_1^\star$, $\tau_2^\star$, $t_1^\star$, $t_2^\star$ and $q^\star$ as the optimal solutions of the deterministic optimization problem given in \eqref{scprob1_pf}. Therefore, the following convergence in probability holds
\begin{equation}
\begin{aligned}\label{sop_conv}
&\tau_{p,1}^\star \xrightarrow{p\to +\infty} \tau_1^\star,~\tau_{p,2}^\star \xrightarrow{p\to +\infty} \tau_2^\star,~t_{p,1}^\star \xrightarrow{p\to +\infty} t_1^\star, \\
&~~~~~~~~~~~~~~ t_{p,2}^\star \xrightarrow{p\to +\infty} t_2^\star,~q_p^\star \xrightarrow{p\to +\infty} q^\star.
\end{aligned}
\end{equation}
Moreover, define $\vartheta_p^\star$ and $\vartheta^\star$ as the optimal solutions of the minimization problems of \eqref{ana_fm9} and \eqref{scprob1_pf} over $\vartheta$ in the feasibility set defined in \eqref{ana_fm1}. Then, we also have the following convergence in probability 
\begin{equation}
\begin{aligned}\label{sop_conv_vt}
\vartheta_p^\star \xrightarrow{p\to +\infty} \vartheta^\star.
\end{aligned}
\end{equation}
\end{lemma}
The convergence result in \eqref{sop_conv} follows using  \cite[Theorem 2.1]{NEWEY19942111}. We can see that all the assumptions in \cite[Theorem 2.1]{NEWEY19942111} are satisfied by the formulations in \eqref{ana_fm9} and \eqref{scprob1_pf}. Moreover, the result in \eqref{sop_conv_vt} follows using \cite[Proposition 2]{dhifallah2020}. The detailed proof is omitted since it follows similar ideas as in Proposition $4$ and Proposition $5$ in \cite{dhifallah2020}. Based on \cite{sion1958}, we can further simplify the formulation in \eqref{scprob1_pf} by solving the minimization problem over the variables $q$ and $\vartheta$. Note that the optimal $\vartheta^\star$ satisfies $\vartheta^\star=0$ if $\mu_0=0$ and $\vartheta^\star=\gamma_3/\mu_0$ otherwise. Furthermore, the optimal $q$ denoted by $q_{t,\tau}^\star$ can be expressed as follows
\begin{align}\label{qstar}
q_{t,\tau}^\star=\frac{ \gamma_2 t_1 \widetilde{\mu}_1 T_1 T_{2,\lambda}(\vt,\vtau) }{\tau_1 + t_1 \widetilde{\mu}_1^2 T_1^2 T_{2,\lambda}(\vt,\vtau) }.
\end{align}
Observe that the optimal solutions, $\vartheta^\star$ and $q_{t,\tau}^\star$,  satisfy the boundedness constraints. Moreover, note that our results are valid for any bounds that satisfy the results in Lemmas \ref{lem_comw}, \ref{lem_comu} and \ref{lembtau}.
Now that we obtained the asymptotic scalar optimization problem, it remains to study the asymptotic behavior of the training and generalization errors.

\subsubsection{Asymptotic Analysis of the Training and Generalization Errors}\label{gen_con_org}
First, the generalization error is given by
\begin{align}
\mathcal{E}_{\text{test}}&=\frac{1}{4^\upsilon} \mathbb{E}\left[ \left( \varphi(\va_{\text{new}}^\top\vxi) -\widehat{\varphi}(\widehat{\vw}^\top \sigma(\mF^\top \va_{\text{new}})) \right)^2 \right],\nonumber
\end{align}
where $\va_{\text{new}}$ is an unseen data sample and $\widehat{\vw}$ is the optimal solution of the noisy formulation.
Based on the uniform Gaussian equivalence theorem (uGET), observed and proved in many earlier papers \cite{hu2020,RF_monta_2,RF_rep_1,modelling,dhifallah2020}, the asymptotic properties of the generalization error are equivalent to the asymptotic properties of $\overline{\mathcal{E}}_{\text{test}}$ defined as follows
\begin{align}
\overline{\mathcal{E}}_{\text{test}}=\frac{1}{4^\upsilon} \mathbb{E}\Big[ &\Big( \varphi(\va_{\text{new}}^\top\vxi) -\widehat{\varphi}(\mu_{0s} \widehat{\vartheta}_p + \mu_{1s} \widehat{\vw}^\top \mF^\top \va_{\text{new}}+\mu_{2s}  \widehat{\vw}^\top \vz) \Big)^2 \Big]. \nonumber
\end{align}
Here, $\widehat{\vw}$ and $\widehat{\vartheta}_p$ are the optimal solutions of our primary formulation given in \eqref{ana_fm1}. The expectation is taken over the distribution of the random vector $\va_{\text{new}}$, the random vector $\vz$ and the possibly random functions $\varphi(\cdot)$ and $\widehat{\varphi}(\cdot)$, where $\vz$ is independent of $\va_{\text{new}}$ and drawn from a standard Gaussian distribution. Moreover, the constants $\mu_{0s}$, $\mu_{1s}$ and $\mu_{2s}$ are defined as $\mu_{0s}=\mathbb{E}[\sigma(z)]$, $\mu_{1s}=\mathbb{E}[z \sigma(z)]$ and $\mu_{2s}^2=\mathbb{E}[\sigma(z)^2]-\mu_{0s}^2-\mu_{1s}^2$, where $z$ is a standard Gaussian random vector. Now, consider the following two random variables
\begin{align}
g_1=  \va_{\text{new}}^\top\vxi,~\text{and}~g_2= \mu_{0s} \widehat{\vartheta}_p+\mu_{1s} \widehat{\vw}^\top \mF^\top \va_{\text{new}}+\mu_{2s}  \widehat{\vw}^\top \vz.\nonumber
\end{align}
Given the optimal solutions $\widehat{\vw}$ and $\widehat{\vartheta}_p$, the random variables $g_1$ and $g_2$ have a bivariate Gaussian distribution with mean vector $[0,\mu_{0s} \widehat{\vartheta}_p]^\top$ and covariance matrix given by
\begin{align}
\mC_p=\begin{bmatrix}
\norm{\vxi}^2 & \mu_{1s} \vxi^\top \mF \widehat{\vw} \\
 \mu_{1s} \vxi^\top \mF \widehat{\vw} &  \mu_{1s}^2 \norm{\mF \widehat{\vw}}^2+\mu_{2s}^2 \norm{\widehat{\vw}}^2
\end{bmatrix}.\nonumber
\end{align}
Define the random variables $\widehat{q}_p^\star$, $\widehat{\beta}_p^\star$ and $\widehat{r}_p^\star$ as follows
\begin{align}\label{opt_po}
\widehat{q}_p^\star=\bar{\vv}^\top \widehat{\vw},~\widehat{\beta}_p^\star=\norm{\mF \widehat{\vw}}^2,~\text{and}~\widehat{r}_p^\star=\norm{\widehat{\vw}}^2,
\end{align}
where $\bar{\vv}=\vv/\norm{\vv}$ and the vector $\vv$ is defined as $\vv=\mF^\top {\vxi}$.
Then, the covariance matrix $\mC_p$ can be expressed as follows
\begin{align}\label{cov_mat_ps}
\mC_p=\begin{bmatrix}
1 & \mu_{1s} T_{p,1} \widehat{q}^{\star}_{p} \\
 \mu_{1s} T_{p,1} \widehat{q}^\star_{p} &  \mu_{1s}^2 \widehat{\beta}_p^\star +\mu_{2s}^2 \widehat{r}_p^\star
\end{bmatrix}.
\end{align}
Hence, to study the asymptotic properties of the generalization error, it suffices to study the asymptotic properties of $\widehat{\vartheta}^\star_p$, $\widehat{q}^\star_p$, $\widehat{\beta}^\star_p$ and $\widehat{r}^\star_p$. The following lemma summarizes the asymptotic properties of our primal formulation given in \eqref{ana_fm1}.
\begin{lemma}[Primal Consistency]\label{prm_cons}
The random variables $\widehat{\vartheta}^\star_p$, $\widehat{q}^\star_p$, $\widehat{\beta}^\star_p$ and $\widehat{r}^\star_p$ converge in probability as follows
\begin{equation}
\begin{aligned}
&\widehat{q}_p^\star \xrightarrow{p\to +\infty} q^\star,~\widehat{\vartheta}_p^\star \xrightarrow{p\to +\infty} \vartheta^\star,\widehat{\beta}^\star_p \xrightarrow{p\to +\infty} \beta^\star \\
&~~~~~~~~~~~~~\widehat{r}_p^\star \xrightarrow{p\to +\infty} r^\star=  (q^\star)^2+h^\prime(\lambda),
\end{aligned}
\end{equation}
where $q^\star$ and $\vartheta^\star$ are optimal solutions of the deterministic scalar formulation in \eqref{scprob1_pf}. Moreover, the function $h(\cdot)$ and $\beta^\star$ are defined in Theorem \ref{ther1}.
\end{lemma}
\begin{proof}
Note that the analysis in Section \ref{MAO_simp} shows that the scalar formulation given in \eqref{ana_fm9} is a simplified version of the multivariate AO formulation given in \eqref{ana_fm5}. Define the random variable $\widetilde{\vartheta}_p^\star$ as the optimal solution of the minimization of the problem \eqref{ana_fm5} over $\vartheta$ in the feasibility set defined in \eqref{ana_fm1}. Moreover, define the random variables $\widetilde{q}_p^\star$, $\widetilde{\beta}_p^\star$ and $\widetilde{r}_p^\star$ as follows
\begin{align}\label{opt_ao}
\widetilde{q}_p^\star=\bar{\vv}^\top \widetilde{\vw},~\widetilde{\beta}_p^\star=\norm{\mF \widetilde{\vw}}^2,~\text{and}~\widetilde{r}_p^\star=\norm{\widetilde{\vw}}^2,
\end{align}
where $\widetilde{\vw}$ is the optimal solution of the multivariate AO formulation given in \eqref{ana_fm5}. Based on the decomposition in \eqref{decp_ao}, note that $\widetilde{\beta}^\star_p$ satisfies the following expression
\begin{equation}
\begin{aligned}\label{betas}
\widetilde{\beta}^\star_p=\norm{\mF \widetilde{\vw}}^2&=(\widetilde{q}^\star_p)^2 \bar{\vv}^\top \mF^\top \mF \bar{\vv} + 2 \widetilde{q}^\star_p \bar{\vv}^\top \mF^\top \mF \mB^{\perp}_{\vv} \widetilde{\vr}^\star \\
&+ (\widetilde{\vr}^\star)^\top \bar{\mB}^{\perp}_{\vv} \mF^\top \mF \mB^{\perp}_{\vv} \widetilde{\vr}^\star,
\end{aligned}
\end{equation}
where $\widetilde{\vr}^\star$ is defined in \eqref{ropt} and is the optimal solution of minimizing  the function $g(\cdot)$ introduced in \eqref{fung}. Substituting the expression of $\widetilde{\vr}^\star$ given in \eqref{ropt}, performing the same analysis as in Section \ref{asy_ana_ao} and using the convergence result in \eqref{sop_conv}, it can be shown that the random quantity $\widetilde{\beta}^\star_p$ converges in probability to $\beta^\star$ defined in \eqref{bstar}. 
Additionally, observe that $\widetilde{r}_p^\star$ can be expressed as follows
\begin{align}
\widetilde{r}_p^\star=\norm{\widetilde{\vw}}^2=(\widetilde{q}^\star_p)^2+\norm{\widetilde{\vr}^\star}^2.
\end{align}
Define the function $h_p:\lambda \to -(\widetilde{q}^\star_p)^2 V_{p,3}(\vt^\star_p,\vtau^\star_p)-V_{p,4}(\vt^\star_p,\vtau^\star_p)$, where the random functions $V_{p,3}(\cdot,\cdot)$ and $V_{p,4}(\cdot,\cdot)$ are defined in \eqref{Vs} and where $\vt^\star_p=[t_{p,1}^\star,t_{p,2}^\star]^\top$ and $\vtau^\star_p=[\tau_{p,1}^\star,\tau_{p,2}^\star]^\top$. Here, $\lbrace t_{p,1}^\star,t_{p,2}^\star,\tau_{p,1}^\star,\tau_{p,2}^\star \rbrace$ are defined in Lemma \ref{consi_detp}. Given the expression of $\widetilde{\vr}^\star$ in \eqref{ropt}, we can see that $\widetilde{r}_p^\star$ can be expressed as follows
\begin{align}
\widetilde{r}_p^\star=(\widetilde{q}^\star_p)^2+h_p^\prime(\lambda),
\end{align}
where the optimal solutions are treated as constants independent of $\lambda$.
Performing the same analysis as in Section \ref{asy_ana_ao} and using the convergence result in \eqref{sop_conv}, it can be shown that the random quantity $\widetilde{r}_p^\star$ converges in probability as follows
\begin{align}\label{rncv}
\widetilde{r}_p^\star \xrightarrow{p\to +\infty} r^\star=  (q^\star)^2+h^\prime(\lambda),
\end{align}
where $q^\star$ is the optimal solution of \eqref{scprob1_pf} and the function $h(\cdot)$ is defined in Theorem \ref{ther1}. Given that  the scalar formulation given in \eqref{ana_fm9} is a simplified version of the multivariate AO formulation and based on Lemma \ref{consi_detp}, we obtain the following asymptotic properties
\begin{align}
\widetilde{q}_p^\star \xrightarrow{p\to +\infty} q^\star,~\widetilde{\vartheta}_p^\star \xrightarrow{p\to +\infty} \vartheta^\star,
\end{align}
where $q^\star$ and $\vartheta^\star$ are the optimal solutions of the deterministic scalar formulation in \eqref{scprob1_pf}. Following a similar analysis as in \cite{Thrampoulidis_2018}, we can show that the assumptions in Theorem \ref{mcgmt} are all satisfied. The main idea is to define the set $\mathcal{S}_{p,\epsilon}$ introduced in Theorem \ref{mcgmt} as 
\begin{align}
\lbrace \vw : \abs{ \norm{\mF {\vw}}^2 -\beta^\star} < \epsilon \rbrace,~\text{and}~\lbrace \vw : \abs{ \norm{{\vw}}^2 -r^\star} < \epsilon  \rbrace.
\end{align}
Then, use the strong convexity properties of the formulation in \eqref{ana_fm8} to prove that the assumptions in Theorem \ref{mcgmt} are satisfied. This means that $\widehat{\vartheta}^\star_p$, $\widehat{q}^\star_p$, $\widehat{\beta}^\star_p$ and $\widehat{r}^\star_p$ defined in \eqref{opt_po} concentrates around the same values as $\widetilde{\vartheta}^\star_p$, $\widetilde{q}^\star_p$, $\widetilde{\beta}^\star_p$ and $\widetilde{r}^\star_p$ defined in \eqref{opt_ao}. 
\end{proof}
Now, to show the convergence of the generalization error in Theorem \ref{ther1}, it suffices to show that $\overline{\mathcal{E}}_{\text{test}}$ is a continuous function in $\widehat{\vartheta}^\star_p$, $\widehat{q}^\star_p$, $\widehat{\beta}^\star_p$ and $\widehat{r}^\star_p$. Based on  Assumption \ref{itm:fun_fwf}, the functions $\varphi(\cdot)$ and $\widehat{\varphi}(\cdot)$ are square integrable over Gaussian distributions. Moreover, the optimal solutions $\widehat{\vartheta}^\star_p$, $\widehat{q}^\star_p$, $\widehat{\beta}^\star_p$ and $\widehat{r}^\star_p$ are bounded. Based on Assumption \ref{itm:fun_fwf} and the continuity under
integral sign property \cite{schilling_2005}, the continuity of $\overline{\mathcal{E}}_{\text{test}}$ follows. These properties lead to the convergence result given in \eqref{gen_conv} in Theorem \ref{ther1}. Based on the analysis in Lemma \ref{prm_cons} and Theorem \ref{mcgmt}, the optimal cost value of the noisy formulation converges in probability to the optimal cost value of the deterministic formulation in \eqref{scprob1_pf}. Combining this result with the asymptotic property stated in \eqref{rncv} shows the convergence of the training error stated in Theorem \ref{ther1}. 
\subsection{Large Number of Noise Injections}
Note that the analysis in the previous parts studies the properties of the training and generalization errors corresponding to \eqref{mform} when $p$ grows to infinity. In this part, we study the properties of the noisy formulation when $\ell$ grows to infinity slower than the dimensions $n$, $p$ and $k$. Note that in this regime and based on the analysis in the previous part, the noisy formulation converges to the deterministic formulation in \eqref{scprob1_pf}. Then, the objective is to analyze the deterministic formulation in \eqref{scprob1_pf} when $\ell$ grows to infinity. Here, we note that the bounds on the feasibility sets of the deterministic formulation in \eqref{scprob1_pf} depends on $\ell$ as follows
\begin{align}
&C_{t_1}=\sqrt{\ell} \overline{C}_{t_1},~C_{t_2}=\sqrt{\ell} \overline{C}_{t_2},~c_{\tau_1}=\sqrt{\ell} \overline{c}_{\tau_1},\nonumber\\
&c_{\tau_2}=\sqrt{\ell} \overline{c}_{\tau_2},~C_{\tau_1}=\sqrt{\ell} \overline{C}_{\tau_1},~C_{\tau_2}=\sqrt{\ell} \overline{C}_{\tau_2}.
\end{align}
We start our analysis by performing the change of variable ${\tau}_1=\tau_1/\sqrt{\ell}$, ${\tau}_2=\tau_2/\sqrt{\ell}$, ${t}_1=t_1/\sqrt{\ell}$ and ${t}_2=t_2/\sqrt{\ell}$. This means that the deterministic scalar optimization problem given in \eqref{scprob1_pf} can be expressed as follows
\begin{equation}\label{det_form_v1}
\begin{aligned}
&\max_{\substack{0 \leq {t_1} \leq \overline{C}_{t_1}\\ 0 \leq {t_2} \leq \overline{C}_{t_2}}} \min_{\substack{\overline{c}_{\tau_1} \leq {\tau_1} \leq \overline{C}_{\tau_1}\\ \overline{c}_{\tau_2} \leq {\tau_2} \leq \overline{C}_{\tau_2}}}~ \frac{\tau_1 t_1+\tau_2 t_2}{2}  - \frac{t_1^2+t_2^2}{2} +\frac{(q_{t,\tau}^\star)^2}{2 T_{2,\lambda}(\vt,\vtau)}   -\frac{\eta T_{3,\lambda}(\vt,\vtau)}{2} \\
&+\frac{t_1}{2\tau_1} \Big( \gamma_1 -  2 \widetilde{\mu}_1 T_1 q_{t,\tau}^\star \gamma_2 + \widetilde{\mu}_1^2 T_1^2 (q_{t,\tau}^\star)^2 + \mu_0^2 (\vartheta^\star)^2 -2\mu_0 \vartheta^\star\gamma_3 \Big),
\end{aligned}
\end{equation}
where $\vt=[t_1,t_2]^\top$ and $\vtau=[\tau_1,\tau_2]^\top$. Here, the functions $T_{2,\lambda}(\cdot,\cdot)$ and $q_{t,\tau}^\star$ are the same as the ones provided in Section \ref{pr_analysis}. Furthermore, the functions $T_{3,\lambda}(\cdot,\cdot)$ and $g_{\kappa,\lambda}(\cdot,\cdot)$ are given as follows
\begin{align}
T_{3,\lambda}(\vt,\vtau)&=t_1^2 \mathbb{E}\Big[\frac{ \widetilde{\mu}_1^2 \kappa + \mu_2^2}{g_{\kappa,\lambda}(\vt,\vtau)}\Big]+\frac{t_1^2+t_2^2}{\ell} \mathbb{E}\Big[\frac{ \widehat{\mu}_1^2 \kappa + \mu_3^2}{g_{\kappa,\lambda}(\vt,\vtau)}\Big],\nonumber\\
g_{\kappa,\lambda}(\vt,\vtau)&=\frac{t_1}{\tau_1} \Big( \widetilde{\mu}_1^2 \kappa + \mu_2^2 \Big)+\Big( \frac{t_1}{\tau_1 \ell}+\frac{t_2(\ell-1)}{\tau_2 \ell} \Big)  \nonumber\\
&\times \Big( \widehat{\mu}_1^2 \kappa + \mu_3^2 \Big) + \lambda.\nonumber
\end{align}
Now, we focus on analyzing the formulation in \eqref{det_form_v1} when the number of noise injections grows to infinity. The following lemma summarizes our main technical results.
\begin{lemma}[Large Number of Noise Injections]\label{lnnoise}
When $\ell$ grows to infinity, the asymptotic limit of the formulation in \eqref{det_form_v1} is obtained by updating the functions $T_{3,\lambda}(\cdot,\cdot)$ and $g_{\kappa,\lambda}(\cdot,\cdot)$ as follows
\begin{align}
&T_{3,\lambda}(\vt,\vtau)=t_1^2 \mathbb{E}\Big[\frac{ \widetilde{\mu}_1^2 \kappa + \mu_2^2}{g_{\kappa,\lambda}(\vt,\vtau)}\Big],~g_{\kappa,\lambda}(\vt,\vtau)=\frac{t_1}{\tau_1} \Big( \widetilde{\mu}_1^2 \kappa + \mu_2^2 \Big)+\frac{t_2}{\tau_2} \Big( \widehat{\mu}_1^2 \kappa + \mu_3^2 \Big) + \lambda.\nonumber
\end{align}
\end{lemma}
The convergence result in Lemma \ref{lnnoise} follows using  \cite[Theorem 2.1]{NEWEY19942111}. Specifically, we use the strong convexity property in Lemma \ref{convty_detpp}. Also, we use the pointwise convergence of the cost functions based on Assumptions \ref{itm:act_fun} and \ref{itm:ass_F} and the dominated convergence theorem. This shows that all the assumptions in \cite[Theorem 2.1]{NEWEY19942111} are satisfied by the formulation in \eqref{det_form_v1} and its asymptotic formulation mentioned in Lemma \ref{lnnoise}. Next, we refer to the asymptotic limit obtained in Lemma \ref{lnnoise} as the \textit{asymptotic deterministic formulation}.

Performing a similar analysis as in Sections \ref{for_mao}, \ref{MAO_simp} and \ref{asy_ana_ao}, it can be checked that the asymptotic deterministic formulation obtained in Lemma \ref{lnnoise} is the asymptotic limit of the following formulation
\begin{align}\label{form_asy_pf}
\min_{\vw\in\mathbb{R}^k}& \frac{1}{2n} \sum_{i=1}^{n}  \Big(y_i-  \vw^\top \widehat{\sigma}(\mF^\top \va_i)  \Big)^2+ \tfrac{1}{2} \norm{\mR^{\frac{1}{2}} \vw}^2 + \tfrac{\lambda}{2} \norm{\vw}^2.
\end{align}
Here, the regularization matrix $\mR$ is defined as follows
\begin{align}
\mR=\widehat{\mu}_1^2 \mF^\top \mF + \mu_3^2 \mI_k,
\end{align}
and the new activation function $\widehat{\sigma}(\cdot)$ satisfies the following properties
\begin{equation}
\begin{aligned}
&\mathbb{E}[\widehat{\sigma}(z)]=\mathbb{E}[\sigma(x_1)],~\mathbb{E}[z\widehat{\sigma}(z)]=\mathbb{E}[z\sigma(x_1)] \\
&~~~~~~~~~~~\mathbb{E}[\widehat{\sigma}(z)^2]=\mathbb{E}[\sigma(x_1)\sigma(x_2)],
\end{aligned}
\end{equation}
where $x_1=z+\Delta v_1$, $x_2=z+\Delta v_2$ and $z$, $v_1$ and $v_2$ are independent standard Gaussian random variables. Now, note that the norm of any vector $\vx\in\mathbb{R}^p$ can be expressed as follows
\begin{align}\label{trick_asy}
\frac{1}{2}\norm{\vx}^2&=\max_{\vu\in\mathbb{R}^p} -\frac{\norm{\vu}^2}{2}+\vu^\top \vx =\max_{t \geq 0} -\frac{t^2}{2}+ t \norm{\vx}\nonumber\\
&=\max_{t \geq 0}\inf_{\tau>0} -\frac{t^2}{2}+ \frac{t \tau}{2}+ \frac{t}{2 \tau} \norm{\vx}^2.
\end{align}
We can see that the optimal solution of the max-min problem in \eqref{trick_asy}, denoted by $t^\star$ and $\tau^\star$, satisfies $t^\star=\tau^\star$. This trick can be used in the CGMT framework to show that the asymptotic limit of the formulation in \eqref{form_asy_pf} can also be expressed as follows
\begin{equation}
\begin{aligned}
\label{scprob1_asy_pf}
&\hspace{-2mm}\max_{\substack{0 \leq {t_1} \leq \overline{C}_{t_1} }} \min_{\substack{\overline{c}_{\tau_1} \leq {\tau_1} \leq \overline{C}_{\tau_1}}}~ \frac{\tau_1 t_1}{2}  - \frac{t_1^2}{2} +\frac{(q_{t,\tau,\infty}^\star)^2}{2 T_{2,\lambda,\infty}(t_1,\tau_1)}   -\frac{\eta T_{3,\lambda,\infty}(t_1,\tau_1)}{2} \\
&+\frac{t_1}{2\tau_1} \Big( \gamma_1 -  2 \widetilde{\mu}_1 T_1 q_{t,\tau,\infty}^\star \gamma_2 + \widetilde{\mu}_1^2 T_1^2 (q_{t,\tau,\infty}^\star)^2 + \mu_0^2 (\vartheta^\star)^2 -2\mu_0 \vartheta^\star\gamma_3 \Big),    
\end{aligned}
\end{equation}
where the constant $\vartheta^\star$ satisfies $\vartheta^\star=0$ if $\mu_0=0$ and $\vartheta^\star=\gamma_3/\mu_0$ otherwise. Moreover, $q_{t,\tau,\infty}^\star$ is defined as follows
\begin{align}\label{qstar_asy_pf}
q_{t,\tau,\infty}^\star=\frac{ \gamma_2 t_1 \widetilde{\mu}_1 T_1 T_{2,\lambda,\infty}(t_1,\tau_1) }{\tau_1 + t_1 \widetilde{\mu}_1^2 T_1^2 T_{2,\lambda,\infty}(t_1,\tau_1) }.
\end{align}
Here, the functions $T_{2,\lambda,\infty}(\cdot,\cdot)$, $T_{3,\lambda,\infty}(\cdot,\cdot)$ and $g_{\kappa,\lambda,\infty}(\cdot,\cdot)$ can be expressed as follows
\begin{align}
&T_{2,\lambda,\infty}(t_1,\tau_1)= \frac{\delta}{T_1^2} \mathbb{E}\Big[ \frac{\kappa}{g_{\kappa,\lambda,\infty}(t_1,\tau_1)} \Big] /\Big(1-\frac{\widetilde{\mu}_1^2 t_1 \delta}{ \tau_1} \mathbb{E}\Big[ \frac{\kappa}{g_{\kappa,\lambda,\infty}(t_1,\tau_1)} \Big] \Big)\\
&T_{3,\lambda,\infty}(t_1,\tau_1)=t_1^2 \mathbb{E}\Big[{ (\widetilde{\mu}_1^2 \kappa + \mu_2^2})/{g_{\kappa,\lambda,\infty}(t_1,\tau_1)}\Big] \nonumber\\
&g_{\kappa,\lambda,\infty}(t_1,\tau_1)=\frac{t_1}{\tau_1} \Big( \widetilde{\mu}_1^2 \kappa + \mu_2^2 \Big)+ \Big( \widehat{\mu}_1^2 \kappa + \mu_3^2 \Big) + \lambda. \nonumber
\end{align}
The property in \eqref{trick_asy} can also be used to show that the optimal solution $t_2^\star$ and $\tau_2^\star$ of the asymptotic deterministic formulation obtained in Lemma \ref{lnnoise} satisfy $t_2^\star=\tau_2^\star$. This then leads to the formulation in \eqref{scprob1_asy_pf}.

Now, define the asymptotic training and generalization errors stated in Theorem \ref{ther1} as $\mathcal{E}_{\text{train},\infty}$ and $\mathcal{E}_{\text{test},\infty}$, respectively. Then, the asymptotic training error converges as follows
\begin{align}\label{tr_conv_asy_pf}
\mathcal{E}_{\text{train},\infty} \xrightarrow{~\ell \to +\infty~} C^\star(\Delta,\lambda)-\frac{\lambda}{2} \left((q^\star)^2+h_{\infty}^\prime(\lambda) \right), 
\end{align}
where $C^\star(\Delta,\lambda)$ is the optimal cost of the deterministic problem in \eqref{scprob1_asy_pf}.
Here, the function $h_{\infty}(\cdot)$ is defined as follows
\begin{align}
h_{\infty}(\lambda)=-(q^\star)^2 \Big( \lambda - \frac{1}{T_{2,\lambda,\infty}(t_1^\star,\tau_1^\star)} \Big)-\eta T_{3,\lambda,\infty}(t_1^\star,\tau_1^\star).\nonumber
\end{align}
Moreover, the asymptotic generalization error converges as follows
\begin{align}\label{gen_conv_asy_pf}
\mathcal{E}_{\text{train},\infty} \xrightarrow{~\ell \to+\infty~} \frac{1}{4^\upsilon} \mathbb{E}\left[ \left( \varphi(g_1) -\widehat{\varphi}(g_2) \right)^2 \right],
\end{align}
where $g_1$ and $g_2$ have a bivariate Gaussian distribution with  mean vector $[0,\mu_{0s} \vartheta^\star]$ and covariance matrix $\mC$, defined as follows
\begin{align}
\mC=\begin{bmatrix}
1 & \mu_{1s} \rho T_1 q^{\star} \\
 \mu_{1s} \rho T_1 q^\star &  \mu_{1s}^2 \beta_{\infty}^\star+\mu_{2s}^2 \left((q^\star)^2+h_{\infty}^\prime(\lambda) \right)
\end{bmatrix}.\nonumber
\end{align}
The constant $\vartheta^\star$ satisfies $\vartheta^\star=0$ if $\mu_0=0$ and $\vartheta^\star=\gamma_3/\mu_0$ otherwise. Here, the constants $\mu_{0s}$, $\mu_{1s}$ and $\mu_{2s}$ are defined as $\mu_{0s}=\mathbb{E}[\sigma(z)]$, $\mu_{1s}=\mathbb{E}[z \sigma(z)]$ and $\mu_{2s}^2=\mathbb{E}[\sigma(z)^2]-\mu_{0s}^2-\mu_{1s}^2$, where $z$ is a standard Gaussian random variable. Additionally, the constant $\beta_{\infty}^\star$ can be computed via the following expression
\begin{align}\label{bstar_pf}
\beta_{\infty}^\star&= \frac{1}{V_1+V_3} \Big( V_1 T_1^2 - V_2-V_4 -\lambda + \frac{1}{T_{2,\lambda,\infty}(t_1^\star,\tau_1^\star)} \Big) (q^\star)^2  \nonumber\\
&+ \frac{\eta T_{3,\lambda,\infty}(t_1^\star,\tau_1^\star)}{V_1+V_3} - \frac{V_2+V_4+\lambda}{V_1+V_3} h_{\infty}^\prime(\lambda),
\end{align}
where the constants $V_1$, $V_2$, $V_3$ and $V_4$ are defined as follows
\begin{align}
&V_1=\frac{t_1^\star \widetilde{\mu}_1^2}{\tau_1^\star},~V_3=\widehat{\mu}_1^2,~V_2=\frac{t_1^\star \mu_2^2}{\tau_1^\star},~V_4=\mu_3^2.\nonumber
\end{align}
Here, $q^\star=q^\star_{t^\star,\tau^\star,\infty}$ satisfies the expression in \eqref{qstar_asy_pf}. Moreover, $t_1^\star$ and $\tau_1^\star$ denote the optimal solution of the problem defined in \eqref{scprob1_asy_pf}. Also, we treat $q^\star$, $t_1^\star$ and $\tau_1^\star$ as constants independent of $\lambda$ when we compute the derivative of the function $h_\infty(\cdot)$. The results in \eqref{tr_conv_asy_pf} and \eqref{gen_conv_asy_pf} can be proved using a similar analysis as in Section \ref{gen_con_org}. Performing a similar analysis as in Sections \ref{for_mao}, \ref{MAO_simp}, \ref{asy_ana_ao} and \ref{gen_con_org}, it can be checked that the training and generalization errors corresponding to the formulation in \eqref{form_asy_pf} converge in probability to the limiting functions obtained in \eqref{tr_conv_asy_pf} and \eqref{gen_conv_asy_pf}, respectively.

Note that the analysis in this Section is valid for any bounds that satisfy the theoretical results in Lemmas \ref{lem_comw}, \ref{lem_comu} and \ref{lembtau}. Moreover, observe that the cost functions of both deterministic problems in \eqref{scprob1_pf} and \eqref{scprob1_asy_pf} diverge when $t_1$, $t_2$, $\tau_1$ or $\tau_2$ grows to infinity or when $\tau_1$ or $\tau_2$ goes to $0$. This means that the solution of the unconstrained version of the formulations in \eqref{scprob1_pf} and \eqref{scprob1_asy_pf} should satisfy the feasibility constraints in \eqref{scprob1_pf} and \eqref{scprob1_asy_pf}. This means that the optimization problems in \eqref{scprob1_pf} and \eqref{scprob1_asy_pf} can be equivalently formulated as in \eqref{scprob1} and \eqref{det_pr_lasy}. This completes the proof of Theorem \ref{ther1}, Theorem \ref{thm2} and Lemma \ref{lem1}.

\section{Conclusion}\label{concd}
In this paper, we precisely analyzed a random perturbation method used to regularize machine learning problems. Specifically, we provided an accurate characterization of the training and generalization errors corresponding to the noisy feature formulation. Our predictions are based on a correlated Gaussian equivalence conjecture and an extended version of the CGMT, referred to as the multivariate CGMT. Moreover, our analysis shows that Gaussian noise injection in the input data has the same effects of a weighted ridge regularization when the number of noise samples grows to infinity. Additionally, it provides the explicit dependence of the introduced regularization on the feature matrix, the activation function and the noise variance. Simulation results validate our predictions and show that inserting noise during training moves the interpolation threshold and can mitigate the double descent phenomenon in the generalization error.

\section{Appendix: Additional Technical Details}\label{app_add}
In this part, we provide additional technical details to prove the results stated in Theorem \ref{ther1}, Theorem \ref{thm2} and Lemma \ref{lem1}. Specifically, we provide a rigorous proof of the theoretical results stated in Lemma \ref{lem_eqv1} and Lemma \ref{lembtau}.

\subsection{Proof of Lemma \ref{lem_eqv1}: High--dimensional Equivalence I}\label{lem_eqv1_pf}
The optimization problems given in \eqref{ana_fm6_nxt} and \eqref{ana_fm6_nxt2} share the same feasibility set $\mathcal{D}$ which we define as follows
\begin{align}
\mathcal{D}=\lbrace (\vw,t_1,t_2): \norm{\vw} \leq C_w, 0\leq t_1 \leq C_{t_1}, 0\leq t_2 \leq C_{t_2} \rbrace.
\end{align}
Define $\widehat{f}_{p,1}$ as the cost function of the optimization problem given in \eqref{ana_fm6_nxt} and define $\widehat{f}_{p,2}$ as the cost function of the optimization problem given in \eqref{ana_fm6_nxt2}. Note that the following inequality $\abs{\sqrt{x}-\sqrt{y}}\leq \sqrt{\abs{x-y}}$ is true for any $x\geq 0$ and $y \geq 0$. Therefore, we have the following inequality
\begin{align}\label{unif_f1f2}
&\sup_{(\vw,t_1,t_2)\in \mathcal{D} } \abs{\widehat{f}_{p,2}(\vw,t_1,t_2)-\widehat{f}_{p,1}(\vw,t_1,t_2)} \leq  \sup_{(\vw,t_1,t_2)\in \mathcal{D} } \Big\lbrace  \sqrt{\frac{2t_1^2}{\ell^2 n} \abs{Z_{p,1}} }+\sqrt{\frac{2t_2^2}{\ell^2 n} \abs{Z_{p,2}} } \Big\rbrace.
\end{align}
where we perform the change of variable $t_1=t_1/\sqrt{n}$ and $t_2=t_2/\sqrt{n}$. Here, $Z_{p,1}$ is defined as follows
\begin{align}
Z_{p,1}&=\sqrt{\ell} \norm{\mSigma^{\frac{1}{2}} \vw} \norm{\mGm^{\frac{1}{2}} \vw} \vh_1^\top \widehat{\vh}_2 + \sqrt{\ell} \norm{\mSigma^{\frac{1}{2}} \vw} \vh_1^\top ( - \mV_1^\top \widehat{\vy} + \mu_0 \vartheta \mV_1^\top \vec{1}_{\ell n} + \widetilde{\mu}_1 T_{p,1} q  \mV_1^\top \widehat{\vs} )\nonumber\\
&+ \norm{\mGm^{\frac{1}{2}} \vw} \widehat{\vh}_2^\top ( - \mV_1^\top \widehat{\vy} + \mu_0 \vartheta \mV_1^\top \vec{1}_{\ell n} + \widetilde{\mu}_1 T_{p,1} q  \mV_1^\top \widehat{\vs} ),
\end{align}
and $Z_{p,2}$ is defined as follows
\begin{align}
Z_{p,2}&=\norm{\mGm^{\frac{1}{2}} \vw} \widetilde{\vh}_2^\top ( - \mV_2^\top \widehat{\vy} + \mu_0 \vartheta \mV_2^\top \vec{1}_{\ell n} + \widetilde{\mu}_1 T_{p,1} q   \mV_2^\top \widehat{\vs}  ).
\end{align}
Given that the set $\mathcal{D}$ is bounded and based on Assumptions \ref{itm:act_fun} and \ref{itm:ass_F}, $Z_{p,1}$ and $Z_{p,2}$ can be bounded by a constant independent of the optimization variables. Combining this with the weak law of large numbers, one can see that the right hand side of \eqref{unif_f1f2} converges in probability to zero. Then, we obtain the following convergence in probability
\begin{align}\label{unf_con_tf1tf2}
\sup_{(\vw,t_1,t_2)\in \mathcal{D}}\abs{\widehat{f}_{p,2}(\vw,t_1,t_2)-\widehat{f}_{p,1}(\vw,t_1,t_2)} \xrightarrow{p\to+\infty} 0.
\end{align} 
Moreover, the following two properties are true for bounded functions
\begin{align}\label{sup_ineq}
\begin{cases}
\abs{\sup_{\vx} f(\vx) - \sup_{\vx} g(\vx)}\leq \sup_{\vx} \abs{f(\vx)-g(\vx)}\\
\abs{\inf_{\vx} f(\vx) - \inf_{\vx} g(\vx)}\leq \sup_{\vx} \abs{f(\vx)-g(\vx)}.
\end{cases}
\end{align}
Given that the functions $\widehat{f}_{p,1}$ and $\widehat{f}_{p,2}$ are bounded in the set $\mathcal{D}$ and the result in \eqref{unf_con_tf1tf2}, we get the following convergence in probability
\begin{align}\label{conv_proo1o2}
\abs{ \widehat{O}^{\star}_{p,1}-\widehat{O}^{\star}_{p,2} } \overset{p \to +\infty}{\longrightarrow} 0,
\end{align}
where $\widehat{O}^{\star}_{p,1}$ and $\widehat{O}^{\star}_{p,2}$ are the optimal objective values of the optimization problems given in \eqref{ana_fm6_nxt} and \eqref{ana_fm6_nxt2}, respectively. Now, define $\widehat{\mathcal{S}}^{\star}_{p,1}$ and $\widehat{\mathcal{S}}^{\star}_{p,2}$ as the set of optimal solutions of the minimization problems in \eqref{ana_fm6_nxt} and \eqref{ana_fm6_nxt2}, respectively. Next, the objective is to show that
\begin{equation}
\mathbb{D}( \widehat{\mathcal{S}}^{\star}_{p,1},\widehat{\mathcal{S}}^{\star}_{p,2} )  \overset{p}{\longrightarrow} 0.
\end{equation}
Moreover, define the functions $\widetilde{f}_{p,1}$ and $\widetilde{f}_{p,2}$ as follows
\begin{align}\label{fn1t}
\begin{cases}
\widetilde{f}_{p,1}(\vw)=\max\limits_{\substack{0 \leq t_1 \leq C_{t_1}\\0 \leq t \leq C_{t_2}}} \widehat{f}_{p,1}(\vw,t_1,t_2)\\
\widetilde{f}_{p,2}(\vw)=\max\limits_{\substack{0 \leq t_1 \leq C_{t_1}\\0 \leq t \leq C_{t_2}}} \widehat{f}_{p,2}(\vw,t_1,t_2).
\end{cases}
\end{align}
Note that the set $\widehat{\mathcal{S}}^{\star}_{p,1}$ is the set of minimizing $\vw$ of the first function in \eqref{fn1t}. Based on Lemma \ref{convty_vn2}, the function $\widetilde{f}_{p,2}$ is strongly convex in the feasibility set where $\lambda$ is a strong convexity parameter. This means that it has a unique minimizer denoted by $\vw_{p,2}^\star$. Now, assume that $\vw_{p,1}^\star$ is a minimizer of the function $\widetilde{f}_{p,1}$. Moreover, assume that there exists $\gamma>0$  independent of $p$ such that the following convergence holds true
\begin{align}
\mathbb{P} \Big( \sup_{\vw^\star\in \widehat{\mathcal{S}}^{\star}_{p,1}} \norm{\vw^\star-\vw_{p,2}^\star}_2 \geq \gamma \Big) \overset{p\to\infty}{\longrightarrow} 1.
\end{align}
Given the strong convexity of the function $\widetilde{f}_{p,2}$, we have the following inequality
\begin{align}
\widetilde{f}_{p,2}(\beta \vw_1 +(1-\beta) \vw_2) &\leq \beta \widetilde{f}_{p,2}(\vw_1)+(1-\beta) \widetilde{f}_{p,2}(\vw_2)\nonumber\\
&-\frac{\lambda}{2} \beta (1-\beta) \norm{\vw_1-\vw_2}^2_2,
\end{align}
where this is valid for any $\beta\in[0,1]$ and feasible $\vw_1$ and $\vw_2$.
Take $\vw_1=\vw_{p,1}^\star$, $\vw_2=\vw_{p,2}^\star$ and $\beta=1/2$. Based on the fact that $\vw_{p,2}^\star$ is a minimizer of the function $\widetilde{f}_{p,2}$, there exists $\gamma>0$ independent of $p$ such that 
\begin{align}\label{cont_hyp}
\mathbb{P} \Big( \sup_{\vw^\star\in \widehat{\mathcal{S}}^{\star}_{p,1}}\abs{\widetilde{f}_{p,2}(\vw_{p,2}^\star)-\widetilde{f}_{p,2}(\vw^\star)} \geq \frac{\lambda \gamma^2}{4}  \Big) \overset{p\to\infty}{\longrightarrow} 1.
\end{align}
Next, we use the convergence in probability established in \eqref{unf_con_tf1tf2} and \eqref{conv_proo1o2} to show that the result in \eqref{cont_hyp} produces a contradiction. 
To this end, note that the following inequality is always valid
\begin{align}
\abs{\widetilde{f}_{p,2}(\vw_{p,2}^\star)-\widetilde{f}_{p,2}(\vw_{n,1}^\star)} &\leq \abs{\widetilde{f}_{p,2}(\vw_{p,2}^\star)-\widetilde{f}_{p,1}(\vw_{p,1}^\star)}+\abs{\widetilde{f}_{p,1}(\vw_{p,1}^\star)-\widetilde{f}_{p,2}(\vw_{p,1}^\star)},
\end{align}
which means that the following inequality is always true
\begin{align}\label{conv_tf1tf2_2}
\abs{\widetilde{f}_{p,2}(\vw_{p,2}^\star)-\widetilde{f}_{p,2}(\vw_{p,1}^\star)} &\leq \abs{O^{\star}_{p,2}-O^{\star}_{p,1}}+\sup_{\norm{\vw}\leq C_w}\abs{\widetilde{f}_{p,1}(\vw)-\widetilde{f}_{p,2}(\vw)}.
\end{align}
Observe that the inequality derived in \eqref{conv_tf1tf2_2} implies that the following inequality holds true
\begin{align}\label{conv_tf1tf2_3}
\sup_{\vw^\star\in \widehat{\mathcal{S}}^{\star}_{p,1}}&\abs{\widetilde{f}_{p,2}(\vw_{p,2}^\star)-\widetilde{f}_{p,2}(\vw^\star)} \leq \abs{O^{\star}_{p,2}-O^{\star}_{p,1}}+\sup_{\norm{\vw}\leq C_w}\abs{\widetilde{f}_{p,1}(\vw)-\widetilde{f}_{p,2}(\vw)}.
\end{align}
Now, based on \eqref{unf_con_tf1tf2}, \eqref{sup_ineq} and \eqref{conv_proo1o2}, the right hand side of \eqref{conv_tf1tf2_2}, converges in probability to zero. This means that the following convergence in probability holds 
\begin{align}
\sup_{\vw^\star\in \widehat{\mathcal{S}}^{\star}_{p,1}}\abs{\widetilde{f}_{p,2}(\vw_{p,2}^\star)-\widetilde{f}_{p,2}(\vw^\star)} \overset{p \to + \infty}{\longrightarrow} 0.
\end{align}
This contradicts with the result in \eqref{cont_hyp}. This means that for any $\epsilon_1>0$ and $\epsilon_2>0$, there exists $p_0\in\mathbb{N}$ such that for any $p\geq p_0$, we have that
\begin{align}
\mathbb{P} \Big( \sup_{\vw^\star\in \widehat{\mathcal{S}}^{\star}_{p,1}} \norm{\vw^\star-\vw_{p,2}^\star}_2 < \epsilon_1 \Big) \geq 1-\epsilon_2.
\end{align}
This means that the following convergence in probability is true
\begin{equation}
\mathbb{D}( \widehat{\mathcal{S}}^{\star}_{p,1},\widehat{\mathcal{S}}^{\star}_{p,2} )  \overset{p \to + \infty}{\longrightarrow} 0,
\end{equation}
where $\mathbb{D}( \mathcal{A},\mathcal{B} )$ denotes the deviation between the sets $\mathcal{A}$ and $\mathcal{B}$ and is defined as $\mathbb{D}( \mathcal{A},\mathcal{B} )=\sup_{\vx_1\in\mathcal{A}} \inf_{\vx_2 \in\mathcal{B}} \norm{\vx_1-\vx_2}_2$.
This completes the proof of Lemma \ref{lem_eqv1}.

\subsection{Proof of Lemma \ref{lembtau}: Additional Compactness}\label{lembtau_pf}
We start our prove by analyzing the feasibility sets of the primal formulation in \eqref{ana_fm2}. Note that the optimal solution of the formulation given in \eqref{ana_fm2} can be expressed in closed form as follows
\begin{align}
\widehat{\vw}_p= \left[ \frac{1}{n \ell} \mK^\top \mK + \lambda \mI_k \right]^{-1} \left( \frac{\mK^\top \bar{\vy}}{n \ell} \right),
\end{align}
for a sufficiently large $C_w$. Here, the matrix $\mK \in \mathbb{R}^{n \ell \times k}$ is defined as follows
\begin{align}
\mK=\widetilde{\mu}_1 \bar{\vs} \vxi^\top \mF + \bar{\mG}  \mSigma^{\frac{1}{2}}+\mT \mGm^{\frac{1}{2}}.
\end{align}
The matrices $\mSigma$ and $\mGm$ are defined in \eqref{msG}. Moreover, $\bar{\vy}$, $\bar{\vs}$ and $\bar{\mG}$ are formed by performing $\ell$ times concatenation of $\widetilde{\vy}=\vy-\mu_0 \vartheta \vec{1}_n$, $\vs$ and $\mG$. Here, $\vy=\varphi(\vs)$ and  $\vs$, $\mG$ and $\mT$ have independent standard Gaussian components. Now, based on the results in \cite{eigen_cons} and Assumptions \ref{itm:act_fun} and \ref{itm:ass_F}, there exists a positive constant $C_1>0$ such that
\begin{align}
\norm{\mK}/\sqrt{n}\leq C_1,
\end{align}
with probability going to $1$ as $p$ grows to $+\infty$. Therefore, there exists a positive constant $C_2>0$ such that
\begin{align}
\sigma_{\text{min}}\Big( \big[ \frac{1}{n \ell} \mK^\top \mK + \lambda \mI_k \big]^{-1} \Big) \geq \frac{1}{C_2+\lambda},
\end{align}
where $\sigma_{\text{min}}(\cdot)$ denotes the minimum eigenvalue. Now, observe that
\begin{align}
\norm{ \mK^\top \bar{\vy} }^2=\widetilde{\mu}_1^2 (\bar{\vs}^\top \bar{\vy})^2 \vxi^\top \mF \mF^\top \vxi+\bar{\vy}^\top \mB \mB^\top \bar{\vy}+2 \widetilde{\mu}_1 (\bar{\vs}^\top \bar{\vy}) \vxi^\top \mF \mB^\top \bar{\vy},
\end{align}
where $\mB=\bar{\mG}  \mSigma^{\frac{1}{2}}+\mT \mGm^{\frac{1}{2}}$. Given that the random quantities $\vs$, $\mG$ and $\mT$ have independent standard Gaussian components, we have the following 
\begin{align}
\frac{1}{ \ell n}\vxi^\top \mF \mB^\top \bar{\vy} \overset{p \to + \infty}{\longrightarrow} 0.
\end{align}
Moreover, using the weak law of large numbers and Assumptions \ref{itm:fun_fwf} and \ref{itm:ass_F}, we obtain the following asymptotic results
\begin{align}
\frac{\bar{\vs}^\top \bar{\vy}}{\ell n} \overset{p \to + \infty}{\longrightarrow} \mathbb{E}[z \varphi(z) ],~ \vxi^\top \mF \mF^\top \vxi \overset{p \to + \infty}{\longrightarrow} \delta \mathbb{E}[\kappa].
\end{align}
Combining this with Assumptions \ref{itm:fun_fwf}, \ref{itm:act_fun} and \ref{itm:ass_F}, we obtain the following inequality
\begin{align}
\frac{1}{(\ell n)^2}\norm{ \mK^\top \bar{\vy} }^2 \geq \frac{1}{2} \widetilde{\mu}_1^2 \delta \mathbb{E}[z \varphi(z) ]^2 \mathbb{E}[\kappa],
\end{align}
valid with probability going to $1$ as $p$ grows to infinity. This shows that there exists a positive constant $c_w>0$ such that
\begin{align}\label{lboptw}
\norm{\widehat{\vw}_p} \geq c_w,
\end{align}
with probability going to $1$ as $p$ grows to infinity. Then, we can apply the multivariate CGMT framework with the additional constraint in \eqref{lboptw}. Based on this result and Assumption \ref{itm:act_fun}, there exists positive constants $c_{\tau_1}>0$, $C_{\tau_1}>0$, $c_{\tau_2}>0$ and $C_{\tau_2}>0$, such that the following convergence in probability holds
\begin{align}
\mathbb{P}(c_{\tau_1} \leq \widehat{\tau}_1 \leq C_{\tau_1}) \xrightarrow{n \to \infty} 1,~\mathbb{P}(c_{\tau_2} \leq \widehat{\tau}_2 \leq C_{\tau_2}) \xrightarrow{n \to \infty} 1,
\end{align}
where $\widehat{\tau}_1$and $\widehat{\tau}_2$ are the optimal solutions of the formulation in \eqref{ana_fm7}. This completes the proof of Lemma \ref{lembtau}.

\balance
\bibliographystyle{IEEEtran}
\bibliography{refs}

\end{document}